\documentclass{article}

\usepackage{microtype}
\usepackage{wrapfig}
\usepackage{graphicx}
\usepackage{subcaption}
\usepackage{booktabs} 
\usepackage{nicefrac}
\usepackage{dsfont}
\usepackage{svg}
\usepackage{hyperref}

\usepackage{listings}
\usepackage{xcolor}

\definecolor{codegreen}{rgb}{0,0.6,0}
\definecolor{codegray}{rgb}{0.5,0.5,0.5}
\definecolor{codepurple}{rgb}{0.58,0,0.82}
\definecolor{backcolour}{gray}{0.9}

\lstdefinestyle{mystyle}{
    backgroundcolor=\color{backcolour},   
    commentstyle=\color{codegreen},
    keywordstyle=\color{magenta},
    numberstyle=\tiny\color{codegray},
    stringstyle=\color{codepurple},
    basicstyle=\ttfamily\footnotesize,
    breakatwhitespace=false,         
    breaklines=true,                 
    captionpos=b,                    
    keepspaces=true,                 
    numbers=left,                    
    numbersep=5pt,                  
    showspaces=false,                
    showstringspaces=false,
    showtabs=false,                  
    tabsize=2
}
\usepackage{hyperref}

\usepackage[accepted]{icml2026}

\usepackage{amsmath}
\usepackage{amssymb}
\usepackage{mathtools}
\usepackage{amsthm}

\usepackage{thmtools}
\usepackage{thm-restate}

\usepackage[capitalize,noabbrev]{cleveref}

\usepackage{amsmath,amsfonts,bm}

\def\eqref#1{equation~\ref{#1}}

\def\1{\bm{1}}

\def\vone{{\bm{1}}}

\def\vtheta{{\bm{\theta}}}

\def\vq{{\bm{q}}}

\def\vu{{\bm{u}}}

\def\vx{{\bm{x}}}

\def\vz{{\bm{z}}}

\def\mI{{\bm{I}}}
\def\mJ{{\bm{J}}}

\def\mM{{\bm{M}}}

\DeclareMathAlphabet{\mathsfit}{\encodingdefault}{\sfdefault}{m}{sl}
\SetMathAlphabet{\mathsfit}{bold}{\encodingdefault}{\sfdefault}{bx}{n}

\theoremstyle{plain}
\newtheorem{theorem}{Theorem}[section]

\theoremstyle{definition}
\newtheorem{definition}[theorem]{Definition}

\theoremstyle{remark}

\newcommand{\normtwo}[1]{\lVert #1 \rVert_2}

\usepackage[textsize=tiny]{todonotes}

\icmltitlerunning{Accelerating Q-learning through Efficient Value-Sharing across Actions}

\begin{document}

\twocolumn[
    \icmltitle{Accelerating Q-learning through Efficient Value-Sharing across Actions}



  \icmlsetsymbol{equal}{*}

  \begin{icmlauthorlist}
    \icmlauthor{Prabhat Nagarajan}{uofa,amii}
    \icmlauthor{Brett Daley}{brett}
    \icmlauthor{Martha White}{uofa,amii,cifar}
    \icmlauthor{Marlos C.\ Machado}{uofa,amii,cifar}
  \end{icmlauthorlist}

  \icmlaffiliation{uofa}{Department of Computing Science, University of Alberta, Edmonton, AB, Canada.}
  \icmlaffiliation{amii}{Alberta Machine Intelligence Institute.}
  \icmlaffiliation{brett}{Meta (independent work, not on Meta's behalf).}
  \icmlaffiliation{cifar}{Canada CIFAR AI Chair}

  \icmlcorrespondingauthor{Prabhat Nagarajan}{nagarajan@ualberta.ca}

  \icmlkeywords{mean-expansion layer, dueling, Q-learning}

  \vskip 0.3in
]



\printAffiliationsAndNotice{}  

\begin{abstract}
  
Action-values are foundational to many control algorithms such as Q-learning.
Therefore learning action-values efficiently is central to reinforcement learning (RL). 
However, learning them can be slow, requiring many updates to move values from their initialization, typically near zero, to their true values, which may be far from zero. 
Moreover, action-value learning algorithms typically update each state–action pair independently, without learning shared value structure across actions within a state.
In this paper, we address these inefficiencies by introducing the \textit{mean-expansion layer}, which accelerates action-value learning by sharing values across actions within a state and by changing the problem from directly learning potentially large action-values to learning a lower-norm representation of them.
In deep RL, this layer can be applied as a parameter-free addition to Q-network architectures without altering the underlying algorithm.
Applied to deep Q-networks and implicit quantile networks, it improves aggregate performance across 57 Atari games while increasing action gaps and dramatically reducing value overestimation.

\end{abstract}

\section{Introduction} \label{sec:intro}

Many reinforcement learning (RL) algorithms learn action-value functions, or Q-functions.
Q-functions represent the expected discounted return for state-action pairs.
Agents can make decisions by referencing their learned Q-function and selecting high-value actions.
Consequently, learning action-value functions efficiently is a central focus in RL~\citep{dqn}.
We draw attention to two observations.

The first observation is that the true action-values being learned are often of high (Euclidean) norm, due to each action-value representing an entire return.
Large norms can be problematic when we consider that almost universally, action-value estimates are initialized to be close to zero.
To add to this, the update rules of action-value learning algorithms like Q-learning~\citep{watkins1989learning,q_learning} make small, incremental changes to the predicted values.
Consequently, when the true action-value functions have high norm, many updates are needed to change the small initial values to the large values over the course of training~\citep{rdq}.
We can reason, then, that lower-norm outputs can be found faster.

The second motivating observation is that in many practical settings, action-values in a state are not mean-zero, as values are often correlated within a state.
If one action in a state has a positive or negative action-value, other actions are also likely to as well, and may even share similar values.
The majority of action-value learning methods, however, update each action-value in a state individually and independently of other actions.
Action-value learning could be accelerated, for example, by learning a common state-dependent baseline value that is shared across all actions in a state.
Doing so would transform the challenging problem of independently learning potentially large individual action-values to an easier one of learning smaller residuals.

These two ideals---value sharing through a state-dependent baseline and learning lower-norm solutions---are very related.
Learning a single scalar baseline for each state that is shared across actions requires the baseline be stored only once, rather than stored repeatedly in each action-value.
These ideals motivate the goal of our approach, which is \textbf{to represent a vector of action-values with a lower-norm vector by efficiently sharing value across actions}.

In deep RL, there are existing methods that share values to accelerate learning, namely dueling network approaches~\citep{dueling,va_learning,rdq}.
At each state, dueling network approaches produce a state-dependent baseline value along with action-specific residuals to construct Q-values.
In practice, however, learning this baseline requires adding extra parameters to a Q-network through a separate multi-layer output stream.

In this paper, we derive and analyze a value-sharing method which we call the \textit{mean-expansion layer} that does not require an explicit baseline nor extra model parameters.
This layer allows us to to share values across actions in a state by constructing an \textit{implicit baseline} which is shared by all actions to construct action-values.
This layer is implementable as a parameter-free modification to a Q-network.
It introduces no changes to the underlying algorithm and can be easily applied to both tabular and deep Q-learning. 
Simple PyTorch~\citep{pytorch} code to implement this layer is shown in Appendix~\ref{appendix:mel_pseudocode}.

Our empirical findings demonstrate several benefits of the mean-expansion layer.
We first show that it can accelerate learning in a controlled gridworld setting.
In deep RL, we find that it accelerates learning and improves aggregate performance when applied to deep Q-networks~\citep[DQN;][]{dqn} and implicit quantile networks~\cite{iqn} in 57 Atari 2600 games~\citep{ale}.
Lastly, we find that it substantially reduces overestimation and increases the action gap in DQN --- both desirable properties in deep RL.

\section{Preliminaries} \label{sec:background}
We formulate the RL problem as a finite Markov decision process (MDP).
A finite MDP is a tuple $(\mathcal{S}, \mathcal{A}, P, \mu, R, \gamma,)$, where $\mathcal{S}$ is a finite set of environment states and $\mathcal{A}$ is the finite set of actions available to the agent.
In this paper we use $n$ to denote $|\mathcal{A}|$, the cardinality of the action space.
The state transition probability $P(s'|s,a)$ is the probability that the agent transitions to state $s'$ after taking action $a$ in state $s$.
The distribution $\mu \in \Delta(\mathcal{S})$ is a start-state distribution, where $\mu(s)$ denotes the probability of an episode beginning in state~$s$.
The reward function $R$ maps state transitions $(s,a,s')$ to a bounded scalar reward $r = R(s,a,s')$, which the agent receives at the next timestep as it transitions to $s'$. 

The agent's objective is to maximize its expected discounted return $\mathbb{E}_{P, \pi, \mu}\left[\sum_{t=0}^\infty \gamma^t R_{t+1} \right]$.
The quantity $\gamma \in [0, 1)$ is a discount rate that exponentially discounts future rewards.
A policy $\pi(\cdot|s) \in \Delta(\mathcal{A})$ is a decision rule defined for all actions that specifies the probability $\pi(a|s)$ of taking action $a$ in state $s$.
The expected discounted return for following policy $\pi$ after taking action $a$ in state $s$ is known as the action-value function for policy $\pi$, formalized as
\[
q_{\pi}(s,a) \coloneqq \mathbb{E}_{P, \pi}\left[\sum_{t=0}^\infty \gamma^t R_{t+1} \big| S_t = s, A_t = a\right].
\]
To learn to maximize the expected discounted return, algorithms like Q-learning~\citep{q_learning} learn an estimated action-value function $Q$ that approximates the action-value function of the optimal policy $q^* = \max_{\pi} q_{\pi}(s,a), \forall s \in \mathcal{S}, a \in \mathcal{A}$.

Deep Q-networks (DQN)~\citep{dqn} adapt Q-learning~\citep{q_learning} to leverage neural networks.
The algorithm learns a Q-network, which is a Q-function parameterized with a set of neural network parameters $\bm{\theta}$. 
Given a state $s$, the Q-network outputs a vector of action-values:
$\mathbf{q}(s; \bm{\theta}) = Q(s, \cdot; \bm{\theta}) = [Q(s,a_1; \bm{\theta}), ..., Q(s,a_n; \bm{\theta})]^\top$.
The agent's most recent $M$ experience transitions $(s, a, r, s')$ are stored in a replay buffer $\mathcal{D}$~\citep{lin_replay}.
This buffer serves as a training dataset for supervised regression of target values, where the targets $y(a, r, s')$ are constructed from the Q-learning update, $y(a,r,s') = r + \gamma \max_{a'}Q(s', a'; \bm{\theta}^{-})$.
The parameters $\bm{\theta}^-$ are the parameters of the \textit{target network}, a time-delayed copy of the Q-network, used to produce stable targets for a fixed interval.
The target network parameters are periodically copied from the Q-network parameters: $\bm{\theta}^- \gets \bm{\theta}$.
As the agent interacts with the environment, the algorithm samples transitions from $\mathcal{D}$ uniformly at random and minimizes the loss  ${\mathbb{E}}_{(s,a,r,s') \sim \mathcal{U}(\mathcal{D})}[(y(a, r,s') - Q(s,a;\mathbf{\bm{\theta}}))^2]$.

\section{The Mean-Expansion Layer} \label{sec:mst}

In this section, we aim to represent a vector $\vq$ of action-values with a lower-norm vector in a manner that shares values efficiently across actions.
To do so, we derive the \textit{mean-expansion layer} (ME layer), a simple, invertible, linear transformation.
We are motivated by the \textit{baseline-residual decomposition}, which represents a vector with a scalar baseline and a vector of residuals.
We begin this section by analyzing the norm-reducing properties of baseline-residual decompositions.
To strive for simplicity in neural network implementations and to further reduce norm, we seek a representation of $\vq$ in which we do not have to explicitly carry an extra baseline parameter.

Let $\vq \in \mathbb{R}^n$.
We can represent each entry of $\vq$ as $q_i = (q_i - b) + b$, where $b \in \mathbb{R}$ is a scalar \textit{baseline} and each $z_i = q_i - b$ is a \textit{residual}. The vector $\vz = [z_1,...,z_n]^\top$ is a \textit{residual vector}.
\begin{definition}[\textit{Baseline-residual vector}]
    The vector $\mathbf{u}(\vq, b) = [z_1, ..., z_n, b]^\top$ is the \textit{baseline-residual vector} of $\vq$ with baseline $b$.
    \end{definition}

This decomposition of $\vq$ explicitly separates the shared component $b$ from $\vz$.
That is, changing $b$ changes all entries $q_i$, $i \in \{1,...,n\}$.
By contrast, changing $z_i$ affects only $q_i$.
The vector $\vu(\vq,b)$ can be used to construct $\vq$:
\begin{equation} \label{baseline_residual_decomposition}
        \vq
    = \left(\vq - b\vone\right)
    + b\vone = \vz + b\vone,
\end{equation}
where $\vone$ denotes the all-ones vector.
\subsection{Norm-reducing baselines}

Representing $\vq$ through a baseline-residual vector $\mathbf{u}(\vq, b)$ has the advantage of being efficient, in that its Euclidean norm is often less than that of $\vq$: $\|\mathbf{u}(\vq, b)\|_2^2 < \|\vq\|_2^2$.
In fact, if the entries of $\vq$ have non-zero mean, then it can be represented with a lower-norm baseline-residual decomposition.
Moreover, we can compute the baseline that results in the minimum norm baseline-residual vector.

\begin{restatable}{prop}{OptimalBaseline}{\text{(Norm-minimizing baseline)}.}
\label{prop:optimal_baseline}
    Given a vector $\vq \in \mathbb{R}^n$, the norm-minimizing baseline is $b^{*} = \nicefrac{\sum_{i=1}^{n} q_i}{(n + 1)}$. That is,
    \begin{align} \label{equation:optimal_baseline}
        b^{*} = \underset{b}{\textup{argmin }} \|\mathbf{u}(\vq,b)\|_2^2 = \frac{\sum_{i=1}^{n} q_i}{n + 1}.
    \end{align}
\end{restatable}

\begin{proof}
    See Appendix~\ref{appendix:theory}.
\end{proof}

More generally, we consider any baseline-residual decomposition efficient if its vector has lower norm than $\vq$.
\begin{definition}[\textit{Norm-reducing baseline}]
    Any $b \in \mathbb{R}$ such that $\|\mathbf{u}(\vq, b)\|^2 < \|\vq\|^2$ is called a \textit{strictly norm-reducing baseline}.
\end{definition}

In Proposition~\ref{prop:good_baselines} in Appendix~\ref{appendix:theory}, we show that in general, if the mean value of the entries of $\vq$,
\begin{equation}
    \mu_{\vq} = \frac{1}{n}\sum_{i=1}^{n} q_i,
\end{equation}
is non-zero, then any baseline non-inclusively between 0 and $2b^{*}$ is a strictly norm-reducing baseline.

In this paper, we only consider baselines between $0$ and $\mu_{\vq}$, which are all strictly norm-reducing baselines,  other than $b=0$.
We go one step further in this section, eliminating the baseline from our representation, resulting in even further norm reduction than a baseline-residual vector.

\subsection{Deriving the Mean-Expansion Layer}

To avoid learning a baseline explicitly, our approach is to store only a residual vector and reconstruct a desired baseline implicitly.
To decide our objective, consider how the mean of the residual vector relates to the baseline.
Let $\vz = \vq - b\vone$ be the residual vector of $\vq$ with baseline~$b$. 
The absolute value of the mean of this residual vector $g_{\vz}(b) = | \frac{1}{n}\sum_{i}(q_i - b) | = |\mu_\vq - b| = |\mu_{\vz}|$ is a decreasing function from $b=0$ to $b = \mu_{\vq}$, regardless of the sign of $\mu_\vq$ (Proposition~\ref{thm:residual_mean}, Appendix~\ref{appendix:theory}).
That is, the mean $\mu_{\vz}$ becomes smaller from $b=0$ to $b=\mu_{\vq}$.
Moreover, the Euclidean norm of the residual vector $\vz$ is also decreasing from $b=0$ to $b=\mu_{\vq}$ (Proposition~\ref{thm:residual_norm}, Appendix~\ref{appendix:theory}).
Consequently, to encourage a lower norm residual vector, we minimize
\begin{equation} \label{vst_optimization}
   \min_{\vz} \left( \frac{1}{n}\normtwo{\vz - \vq}^2 + k \left(\frac{1}{n}\sum_{i=1}^n z_i\right)^2 \right),
\end{equation}
where $k \ge 0$ is a penalty coefficient on the mean value of the vector's entries.
The first term encourages $\vz$ to fit $\vq$ exactly, or equivalently, $\vz = \vq - b\vone$, with $b=0$.
The second term penalizes the mean of the entries of $\vz$ for having large magnitude, encouraging $\vz = \vq - b\vone$, for some $b \in [0, \mu_\vq)$, as we will soon show.

Solutions that minimize this objective trade off directly representing $\vq - 0$ and representing some residual vector $\vq - b\vone$.
The choice of $k$ dictates the degree of this tradeoff, or how much the mean should be penalized for being large.
The fraction $\nicefrac{1}{n}$ on the first term reflects the fact that $\normtwo{\vz - \vq}^2$ is a sum over $n$ squared deviations, so its loss grows with $n$.
To account for this sum, we normalize the sum of squared deviations by $n$.
By contrast, regardless of the value of $n$, the second term squares a single scalar value (i.e., the mean), and does not require normalization as $n$ grows.

We can characterize the solution to this objective as a solution to a linear system of equations.
Let $\mJ$ be the all-ones matrix, i.e., the $n \times n$ matrix of ones.

\begin{restatable}{prop}{RegularizationInterpretation}
\label{thm:regularization}
Given vector $\vq \in \mathbb{R}^n$, the unique minimizer of Equation~\ref{vst_optimization} is the solution to the system of equations:
\begin{equation} \label{me_transformation}
    \vq = \left(\mI + \frac{k}{n}\mJ \right)\vz.
\end{equation}
\end{restatable}

That is, to solve the objective in Equation~\ref{vst_optimization}, we can solve the system in Equation~\ref{me_transformation}.
For $k\ge 0$, we call the invertible matrix $\mM_k = \mI+\frac{k}{n}\mJ$ the \textit{mean-expansion layer}, where $k$ is a mean-scaling coefficient.

Returning to the language of baselines and residuals, the vector $\vz$ is still a residual vector, but the baseline vector $b\vone = \frac{k}{n}\mJ\vz$ is inferred from $\vz$.
In other words, given some target vector of values $\vq$, for some desired baseline $b \in [0, \mu_{\bm{q}})$, $\vq$ can be entirely represented by the residuals $\vz$ that represent the solution to the system described in Equation~\ref{me_transformation}.

\subsection{Geometric Properties} 
\label{subsection:geometry}

\begin{figure}[t]
    \centering
    \includegraphics[width=0.5\textwidth]{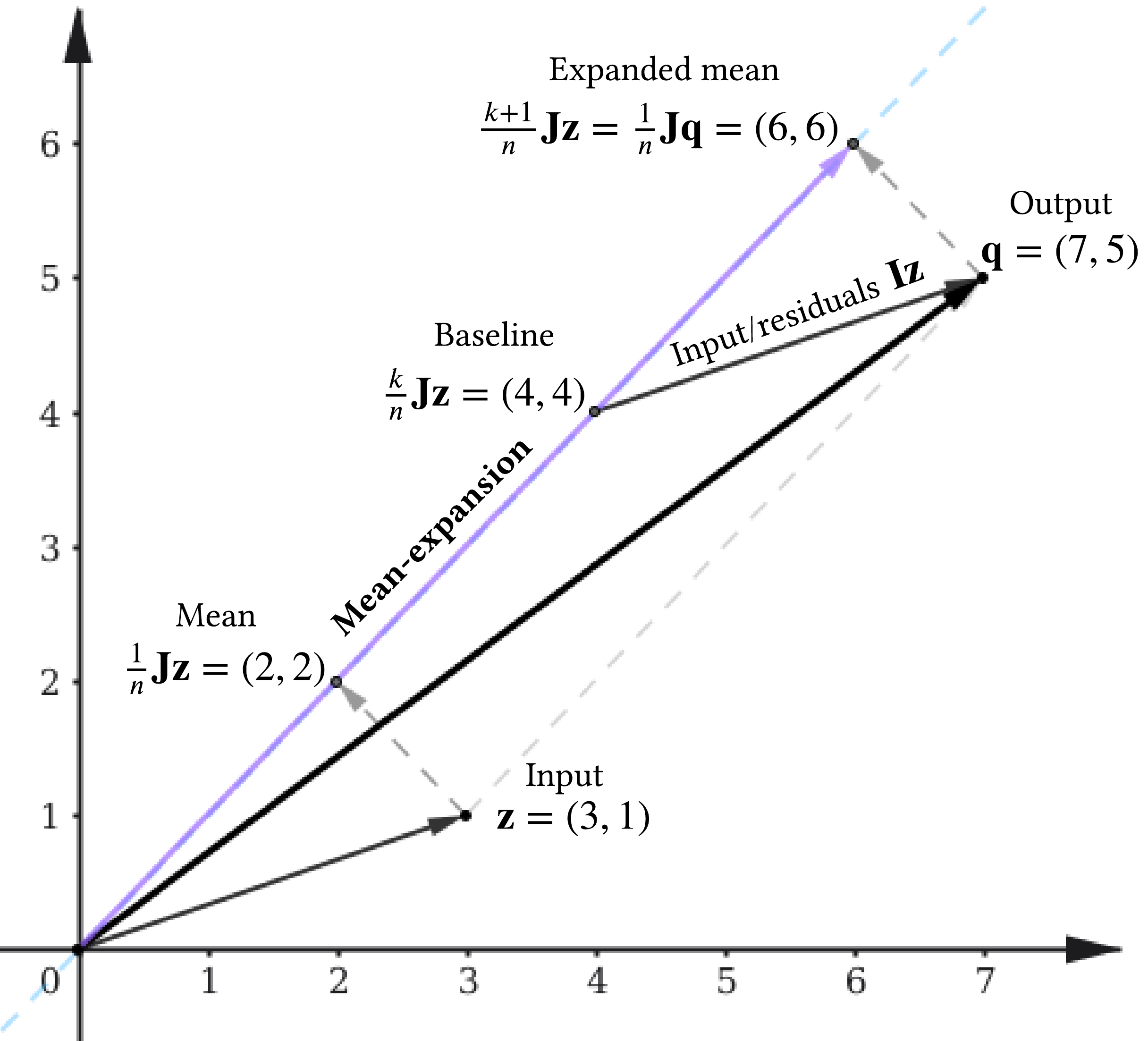}
    \caption{\textbf{The mean-expansion layer.} 
    The input vector $\vz = (3,1)$, is projected onto the all-ones vector to produce the mean component $(2,2)$.
    This mean component is scaled by $k$, where $k=2$, to produce the implicit baseline vector $\frac{k}{n}\mJ\vz = (4,4)$.
    The input vector, which also serves as the residual vector, $\vz$, is added to this implicit baseline to produce the output $\vq = (\mI + \frac{k}{n}\mJ)\vz = (7,5)$.
    The mean vector of $\vq$ is $(6,6)$, scaling the mean component of $\vz$ by $k+1$.
    The ME layer allows us to store low-norm vector $\vz$ instead of $\vq$ without explicitly storing a baseline.
    }
    \label{fig:geometric}
\end{figure}

As per its name, the mean-expansion layer scales the mean component of a vector.
To understand this, first note that the matrix $\frac{1}{n}\mJ$ is the projection matrix onto the all-ones direction $\vone$.
Applied to a vector, it produces a constant vector of means $\frac{1}{n}\mJ \vq = \mu_\vq \vone = [\mu_\vq,...,\mu_\vq]^{\top}$.
We call $\mu_\vq \vone$ the \textit{mean component} of $\vq$, i.e., the projection of $\vq$ onto $\vone$.
Geometrically, the mean-expansion layer scales the mean component by a factor of $k+1$.

Equation~\ref{me_transformation} can be rewritten as (see Appendix~\ref{appendix:theory:properties}):

\begin{equation} \label{equation:geometry}
    \vq = \underbrace{\left((k+1)\frac{1}{n}\mJ\vz\right)}_{\text{mean scaling}} + \underbrace{\left(\vz - \frac{1}{n}\mJ\vz\right)}_{\text{mean orthogonal}}.
\end{equation}
The mean component of $\vz$ is scaled by $k+1$ and added to the component of $\vz$ orthogonal to $\vone$ to produce $\vq$.
Equivalently, the residual vector $\vz$ is $\vq$ with its mean component scaled down.
Small changes to $\vz$ along the all-ones dimension result in large changes to $\vq$ along the all-ones dimension.
Figure~\ref{fig:geometric} provides a visual depiction of the transformation.

\subsection{Selecting $k$}

We have established how the mean-expansion layer avoids the explicit maintenance of a baseline parameter $b$.
All that remains, then, is to understand how to select $k$, based on the algorithm designer's desired implicit baseline $b$.
The relationship between $b$ and $k$ can be formalized as follows.

\begin{restatable}{prop}{BaselineRelationship}
\label{prop:baseline_relationship}
Let $\mu_\vz = \frac{1}{n}\sum_{i=1}^n z_i$.
Let $b = k\mu_{\vz}$, where $b\vone = \frac{k}{n}\mJ\vz$.
Then if $k > 0$,
\begin{align} \label{baseline_q_relationship}
    b = k \mu_{\vz} = \frac{\sum_{i=1}^n q_i}{n + \frac{n}{k}}.
\end{align}
\end{restatable}

\begin{proof}
    See Appendix~\ref{appendix:theory}.
\end{proof}

Applying Equation~\ref{baseline_q_relationship}, we see that as $k \rightarrow 0^{+}$, i.e., as $k$ approaches 0 from the right, the baseline $b \rightarrow 0$.
As $k \rightarrow \infty$, $b \rightarrow \mu_\vq$.
In this paper, we use $k=n$ unless stated otherwise, which corresponds to the norm-minimizing baseline $b^{*}$.

It is easy to see that for $k > 0$, when $\vq$ has non-zero mean, $\vz$ has a lower norm than both $\vq$ \textit{and} its corresponding baseline-residual vector.
That is, $\|\vz\|_{2}^{2} < \|\vu(\vq, k\mu_\vz)\|_{2}^2 < \|\vq\|_2^2$.
This follows from the fact that when $k > 0$, the implied baseline is in the range $(0, \mu_\vq)$, which are all strictly norm-reducing baselines.
By not explicitly storing the baseline, $\vz$ has even lower norm than its baseline-residual counterpart.

\subsection{Optimization and Numerical Stability}

When using the mean-expansion layer, $k$ cannot be too large.
While no mathematical issues arise with an arbitrarily large $k$, computational or numerical issues do.
Computing $\vz$ requires solving the system in Equation~\ref{me_transformation}.
For $k \ge 0$, the condition number of $\mM_k$ is $k + 1$, arising from scaling the all-ones direction by $k+1$. 
Hence, solving this system becomes numerically unstable for large values of $k$.
For extremely large $n$, $k=1$ is a safe choice, keeping the condition number at $2$.
The value $k = 1$ corresponds to the baseline $\nicefrac{\mu_\vq}{2}$, which is a strictly norm-reducing baseline when the mean is non-zero.

\section{The Mean-Expansion Layer for Q-networks} \label{sec:rl_applications}

Implementing the mean-expansion layer in a neural network is straightforward.
Just as a softmax layer transforms a vector of logits into a probability vector, the ME layer takes as input an $n$-dimensional vector of residuals and produces an output that adds those residuals to a baseline which is constructed by scaling the mean of those very residuals.
The layer increases the magnitude of the mean value of the input vector entries while preserving their relative differences.
PyTorch~\citep{pytorch} code for the ME layer is provided in Appendix~\ref{appendix:mel_pseudocode}.

The mean-expansion layer is added immediately before the vector-valued output of a neural network. 
In the previous section, we described how solving the system in Equation~\ref{me_transformation} allows us to produce the lower magnitude residual-only representation of some vector $\vq$.
When producing a vector-valued output $\vq$ for a state, as is done in Q-learning, we cannot solve Equation~\ref{me_transformation}  without prior knowledge of $\vq$.
Since $\vq$ is being learned, we instead implicitly change the problem to one of learning $\vz$.
The network's penultimate layer first outputs $\vz$ and then we apply the final layer, an ME layer, to construct $\vq = \mM_k\vz$.
Losses on $\vq$ are then backpropagated through $\mM_k$
to $\vz$.
We transition from gradient descent on $\vq$ in a standard Euclidean space to gradient descent on $\vq$ in a modified Euclidean space where distances along the all-ones direction are shortened.

Taking as input the vector $\vz$, the layer is implemented in a manner similar to Equation~\ref{equation:geometry} to produce output $\vq$:
\begin{align*}
    \mu_\vz &\gets \frac{1}{n} \sum_{i=1}^n z_i, \\
    \vq &\gets  \vz - \mu_\vz \vone + (k+1)\mu_\vz \vone.
\end{align*}

\subsection{Implicit-Baseline Deep Q-networks}

Concretely, in DQN, the standard output layer becomes the penultimate layer and we add an ME layer as an output layer. 
Given a state $s$, a standard Q-network outputs a vector $\vq(s;\bm{\theta})$ containing the $n$ values corresponding to each action-value $Q(s,a;\bm{\theta})$.
Adding a mean-expansion layer causes the network to first output a residual vector~$\vz$, whose entries are aggregated according to the mean-expansion layer to produce the action-value outputs.
We call this method Implicit-Baseline DQN, or $\text{IB-DQN}(k)$, where $k \ge 0$ is the hyperparameter to set the desired baseline.

IB-DQN has several benefits.
First, it generalizes DQN as special case when $k=0$.
Moreover, the invertibility of the transformation ensures that with a sufficiently deep network, the representation capacity does not change relative to a standard Q-network.
As a mere layer addition, IB-DQN does not modify the DQN algorithm itself and introduces no additional learnable parameters to the model.
A major benefit of IB-DQN is that $k$ is not an opaque hyperparameter.
We can select $k$ to produce some specific implicit baseline as a function of $\vq$, which may make the problem of hyperparameter selection easier.

\subsection{Mean-Expansion for Tabular Q-learning}
The mean-expansion layer can also be applied to tabular Q-learning.
In lieu of a lookup table of Q-values, we maintain a lookup table of residuals $Z(s,a)$ and construct the Q-values using the mean-expansion layer $M_k$: $Q(s, \cdot) = M_k Z(s,\cdot)$. This is equivalent to $Q(s,a) = Z(s,a) + \frac{k}{n}\sum_{a'}Z(s,a')$.

A gradient descent step with step size $\alpha_t$ produces the semi-gradient (i.e., the target is treated as a constant) update rule for all $a \in \mathcal{A}$ (see Appendix~\ref{appendix:tabular_update} for gradient computation):
\begin{align} \label{action_update}
    Z(s_t,a_t) &\gets Z(s_t,a_t) + \alpha_t \delta_t \left(1 + \frac{k}{n}\right).
\end{align}
\begin{align} \label{baseline_update}
     Z(s_t,a) &\gets Z(s_t,a) + \alpha_t \delta_t \left(\frac{k}{n}\right), \quad a \neq a_{t}.
\end{align}
We call this \textit{Implicit-Baseline Q-learning} (IBQ), as we combine the mean-expansion layer, which generates an implicit baseline, with Q-learning.
Each action's residual $Z(s_t, a)$ is incremented by $\frac{k}{n} \alpha_t \delta_t$.
The chosen action's residual is incremented by an additional $\alpha_t \delta_t$.
The action-value of the chosen action is incremented in greater proportion to the Q-learning error than the other actions, but the entire update emphasizes updating the baseline substantially more when $k$ is large.
Observe that we recover Q-learning~\citep{q_learning} as a special case when $k=0$.
More generally, this derivation applies to any TD-learning based action-value learning algorithm like Sarsa~\citep{sarsa} and Expected Sarsa~\citep{expected_sarsa}.

The update in Equations~\ref{action_update} and ~\ref{baseline_update} highlight how the ME layer induces the value-sharing in Q-learning.
The construction of $Q(s,a)$ shows that all the $Z(s,a)$ in a state each carry some component of the shared baseline used to construct \textit{any} action's value in that state.
Consequently, the credit for the TD error is distributed across all actions in a state, as indicated by the update equations.
This manner of credit distribution should accelerate learning of the shared baseline in a state relative to the residuals, potentially accelerating overall learning.

\section{Related Work} \label{sec:related}
The most related class of approaches are dueling network approaches~\citep{dueling,va_learning,rdq}, the closest being Dueling DQN.
Dueling DQN, like our method, is also an architectural modification to a Q-network that is otherwise trained through standard DQN, without modification to the loss function.
Dueling DQN, despite the use of the terms ``value'' and ``advantage'', in fact implements a baseline-residual decomposition of the action-value function.
Dueling DQN outputs a state-specific baseline $B(s; \vtheta)$, which serves as the mean action-value.
It also outputs action-specific residuals $Z(s,a; \vtheta)$ to construct action-values as $Q(s,a; \vtheta) = B(s; \vtheta) + Z(s,a; \vtheta)$.
To learn this baseline, however, Dueling DQN adds an additional two-layer output stream to a Q-network, introducing a substantial number of extra parameters.
In Appendix~\ref{appendix:dueling}, we show how Dueling DQN is a baseline-residual decomposition, specifically a mean-residual decomposition.

Another related method is called Regularized Dueling Q-learning (RDQ)~\citep{rdq}.
RDQ leverages the same architecture as Dueling DQN.
However, $B(s)$ does not serve as a predefined baseline in RDQ, unlike in Dueling DQN where its value is the mean action-value prediction.
Instead, RDQ augments the DQN loss with an additional L2-penalty,
$\beta (\frac{1}{2}B(s; \vtheta)^2 + \frac{1}{2}\sum_{a}Z(s,a; \vtheta)^2)$, where $\beta$ is a regularization coefficient.

Both RDQ and Dueling DQN learn an extra functional parameter to represent a baseline and both require an additional stream of learnable parameters.
Dueling DQN, like our method, is a mere architectural modification to DQN with a predefined baseline.
RDQ, by contrast, neither uses a predefined baseline, nor is it a mere architectural change, as it augments the DQN loss function.

Advantage updating~\citep{advantage_updating} serves as a conceptual predecessor for the methods mentioned here.
It is a classical algorithm that decomposes a value function into state-values and action-advantages, where the state-value is shared by actions in a state.
Its successor, Advantage Learning~\citep{advantage_learning} avoids learning an explicit state-value function.
These methods emphasize stable learning through gradient descent with function approximation rather than accelerating learning through value-sharing.

Baseline-residual decompositions are broadly related to centering and normalization, which have been explored in RL~\citep{reward_centering}.
Pop-Art~\citep{van2016learning} is a method which adaptively normalizes targets to be robust to different value scales.
\citet{reward_shifting} explore how shifting rewards can improve curiosity-driven exploration.
More generally, baselines are commonly used in RL for variance reduction and stability, particularly in policy gradient methods~\citep{reinforce,baseline_pg}.

\section{Experiments} \label{sec:experiments}

In this section, we evaluate our implicit-baseline Q-learning methods in a gridworld setting as well as in 57 Atari games.
In particular, we examine its sample efficiency, performance, sensitivity to $k$, overestimation, and action gaps.

\subsection{Gridworld Experiment}

\begin{figure}[tbp!]
    \centering
    \includegraphics[width=0.45 \textwidth]{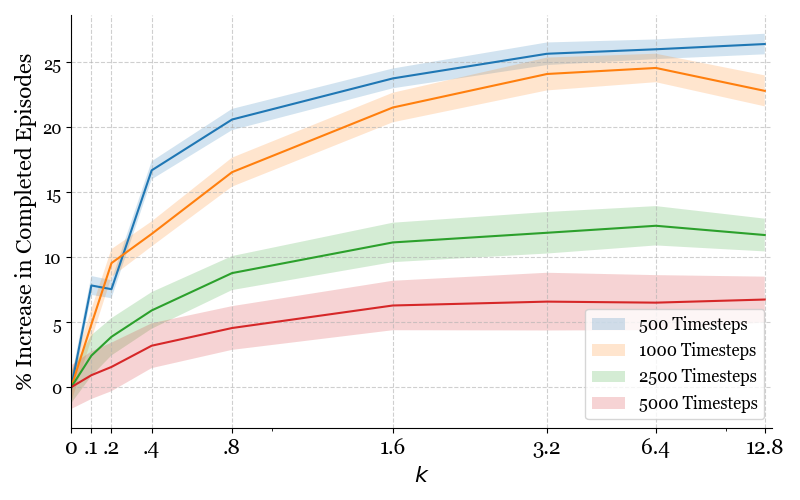}
    \caption{\textbf{Gridworld results.}    
    We compare IBQ with different values of $k$, including $k=0$, which is Q-learning.
    We report the percentage increase in episode completions across four sample complexity regimes.
    Shaded regions corresponds to a 95\% confidence interval.
    For most k, IBQ(k) can complete over 20\% more episodes than Q-learning within 1k timesteps.
    Both algorithms quickly master the task and their gap decreases with more timesteps.}
    \vspace{-12pt}
    \label{results:gridworld}
\end{figure}

\begin{figure*}[t]
    \centering
 \begin{subfigure}[b]{0.35\textwidth}
        \centering
        \includegraphics[width=\linewidth]
        {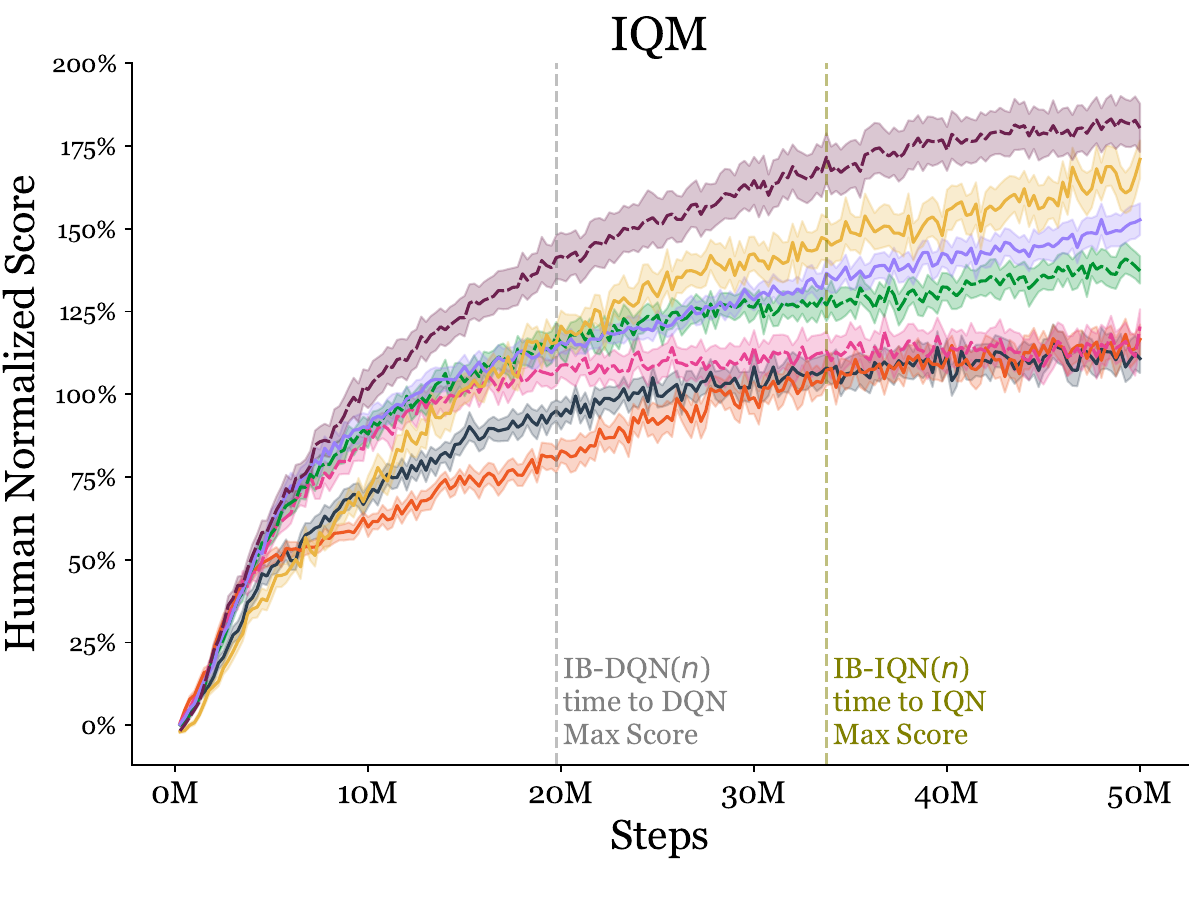}
        {\raggedleft \includegraphics[width=0.85\linewidth]{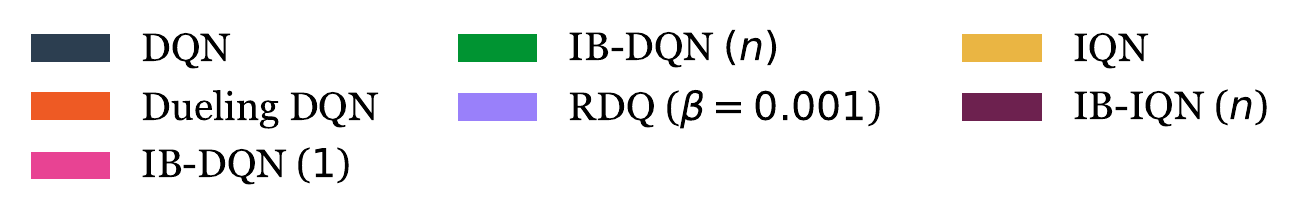} \par}
        \vspace{0.33cm} 
        \label{fig:iqm}
    \end{subfigure}
    \hfill
    \begin{subfigure}[b]{0.64\textwidth}
        \centering
        \includegraphics[width=\linewidth]{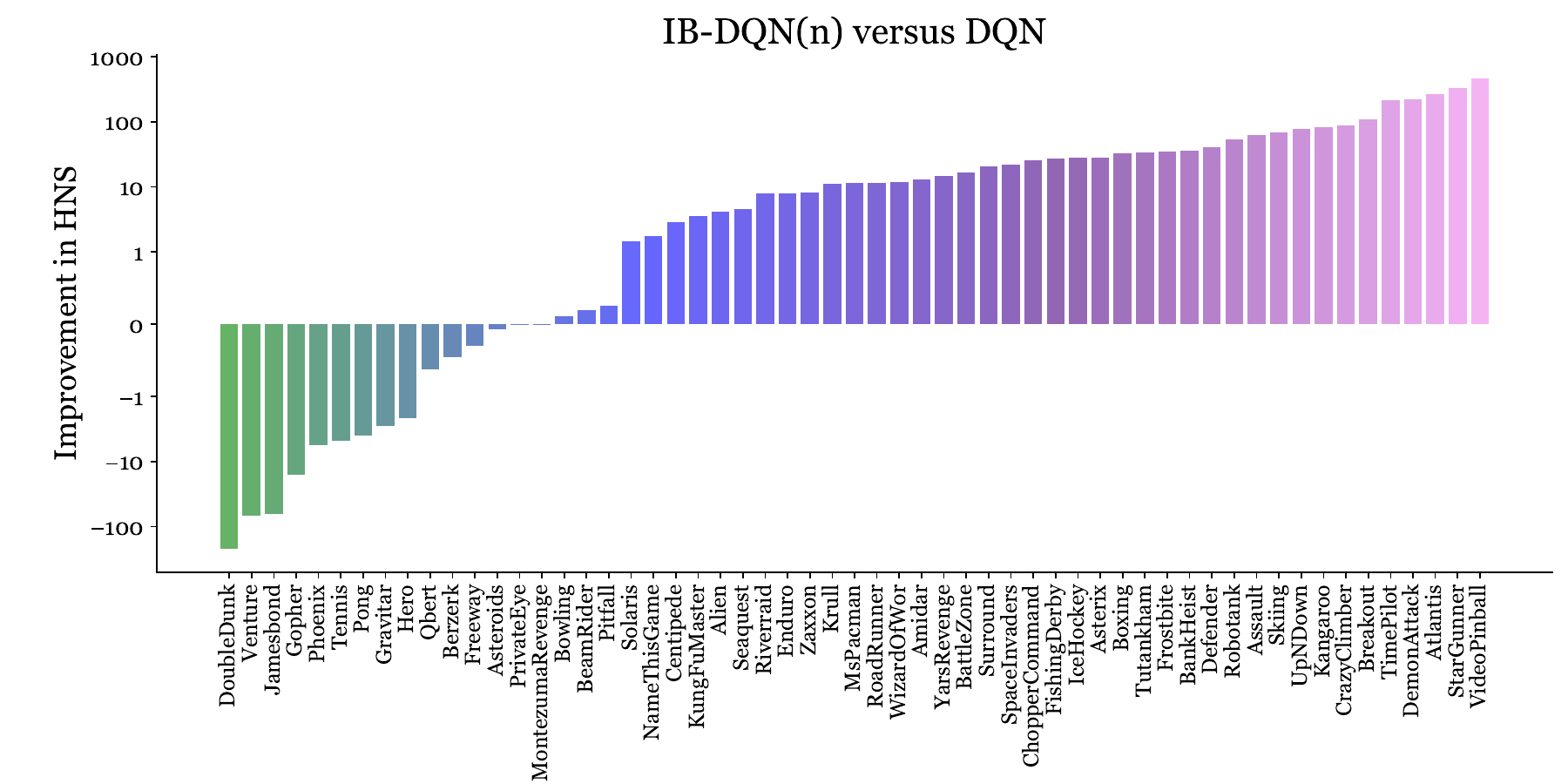}
        \label{fig:median}
    \end{subfigure}
  \caption{ (\textit{left}) Different algorithms and the interquartile mean of their human-normalized score across 57 games. All algorithms were run for five seeds per game.
  The shaded region depicts the 95\% stratified bootstrap confidence interval~\citep{statistical_precipice}. Dashed lines indicate the use of the mean-expansion layer. (\textit{right}) The increase in human-normalized score, measured as the average area-under-the-curve, when switching from DQN to IB-DQN($k=n$). Note that the y-axis is log-scale.}
  \label{fig:hns_comparison}    
\end{figure*}

We first examine IBQ, the tabular algorithm which combines the mean-expansion layer with tabular Q-learning (Equations~\ref{action_update} and ~\ref{baseline_update}).
To highlight its sample efficiency, we examine it in a pedagogical 5x5 stochastic gridworld task.
Episodes begin in the bottom left corner and terminate upon reaching the goal state in the top right corner.
The discount rate is 0.95.
The agent receives a reward of 5 for reaching the goal and 0 otherwise.
The agent's actions are the four cardinal directions.
The agent transitions according to its chosen action with probability $\nicefrac{3}{4}$ and according to a different randomly selected action with probability $\nicefrac{1}{4}$.

We evaluate IBQ at different $k$.
For each $k$, we report the best results across a sweep over 61 step sizes chosen from a logarithmic search over $(0,1]$, running each combination for 128 seeds.
All agents employ an $\epsilon$-greedy policy with $\epsilon=0.1$ for 5k timesteps.
Figure~\ref{results:gridworld} shows, for different values of $k$, the percentage increase in the number of completed episodes over Q-learning, across four different sample complexity regimes.
The shaded regions correspond to a 95\% confidence interval.

IBQ consistently learns faster than Q-learning.
By distributing credit to all four actions rather than one, value propagates more quickly.
The gap between Q-learning and IBQ decreases, however, with more experience, as both algorithms equally master the task.
We also observed that when $k$ is large, smaller step sizes are more suitable, given the larger scaling induced by $k$.
Conversely, when $k$ is small, larger step sizes are more suitable.

\subsection{Deep RL Empirical Methodology}

\paragraph{Setting.} We evaluate all of our agents in the Arcade Learning Environment (ALE)~\citep{ale} on the 57 standard Atari environments used in the literature~\citep{double_dqn}.
We follow the evaluation practices proposed by \citet{revisitingale}.
Specifically, we train and evaluate agents with sticky actions with probability 0.25, game-over termination signals, and the full action set.

\paragraph{Code and baselines.} Our algorithms and baselines\footnote{\fontsize{9pt}{11pt}\href{https://github.com/prabhatnagarajan/me_layer}{github.com/prabhatnagarajan/me\_layer}.} are implemented based on the PFRL
library~\citep{pfrl} and its reproduction of DQN.
Our DQN implementation follows modern best practices, including the use of the Adam optimizer~\citep{adam} and the squared-error loss, using the optimizer settings of ~\citet{rainbow}, as has become standard practice~\citep{revisiting_rainbow}.
IB-DQN($k$) refers to DQN which adds a final mean-expansion layer with mean-scaling coefficient $k$.
For IB-DQN, we do not modify the underlying DQN algorithm or its hyperparameters.
We compare against the most related value-sharing method, Dueling DQN.
We also scale up and evaluate RDQ, which was originally evaluated on MinAtar~\citep{minatar}, to the full ALE.
For RDQ, we set $\beta=0.001$, matching the original paper after testing several values on a subset of games.
For implicit quantile networks (IQN)~\citep{iqn}, we largely match the training settings of the original paper.
IB-IQN refers to IQN where an ME layer is added at the end of the network.
Appendix~\ref{appendix:experiments} contains more training and evaluation details for reproducibility.

\subsection{Atari-57 Performance}
Figure~\ref{fig:hns_comparison} (\textit{left}) depicts the performance in terms of the interquartile mean (IQM)~\citep{statistical_precipice} of the human-normalized score~\citep{dqn} of DQN, Dueling DQN, IB-DQN, RDQ, IQN, and IB-IQN.
All algorithms were run for five seeds per game.
Figure~\ref{fig:hns_comparison} (\textit{right}) shows the change in human-normalized score when switching from DQN to IB-DQN.
Figure~\ref{Atari57:score:page_2} in Appendix~\ref{appendix:results} shows the per-game learning curves.
Additionally, Table~\ref{table_results} in Appendix~\ref{appendix:results} reports the mean score over the last three evaluations for each agent in each environment, across all seeds.

IB-DQN($n$) clearly performs better than both DQN and Dueling DQN, with much better sample efficiency.
The first vertical dotted line shows the first timestep at which the IQM of IB-DQN($n$) exceeds DQN's highest score across its curve.
IB-DQN$(n)$ is able to achieve DQN's best performance at just under 20M timesteps, i.e., within 40\% of the total training time, exhibiting improved sample efficiency.
IB-IQN($n$), i.e., IQN with the ME layer, also exhibits  faster learning and performance over IQN.
It achieves IQN's highest score across its curve at 33.75M timesteps, at just over $\nicefrac{2}{3}$ of its total training time.
To test the impact of the ME layer under a different optimizer,  we evaluate the ME layer on DQN with RMSprop~\citep{rmsprop} and the Huber loss in Figure~\ref{fig:hns_rmsprop} of Appendix~\ref{appendix:rmsprop}.
The ME layer accelerates learning and improves performance in this setting too, reinforcing our core findings.

When $k=1$, which represents the conservative ME layer with a low condition number, we see faster initial learning and a modest improvement through training.
By scaling the all-ones component of a vector, the gradients along the all-ones component are also scaled.
This gradient scaling accelerates learning of a common baseline value for all entries of the vector.
When $k=1$, the mean action-value in a state is updated less aggressively than larger values of $k$, yet this appears sufficient to accelerate initial learning.
Overall, these results demonstrate that the ME layer can be an effective addition to a standard deep Q-network.

Scaling RDQ demonstrates it to be an effective algorithm.
Unlike Dueling DQN, RDQ is able to effectively leverage its additional parameters through explicit regularization.
RDQ outperforms IB-DQN($n$), though the latter is competitive.
This is to be expected; RDQ has an extra stream of parameters and an augmented loss.
The primary appeal of IB-DQN($n$) is its simplicity as a simple parameter-free layer that can be added to Q-networks that boosts performance through an implicit baseline.

Dueling DQN does not outperform DQN, contrary to the original results of ~\citet{dueling}.
It has been shown that modern implementations of DQN that use Adam and the squared-error loss perform on par with distributional variants~\cite{optimistic}, which have been known to outperform Dueling Double DQN with prioritized replay~\citep{c51}.
Though it is difficult to draw definitive conclusions without extensive tuning, it may be the case the Dueling DQN's benefits are better realized in older implementations of DQN with did not use Adam and the squared-error loss.

\begin{figure*}[t]
    \centering
    \begin{subfigure}[c]{0.495\textwidth}
        \centering
        \includegraphics[width=\linewidth]{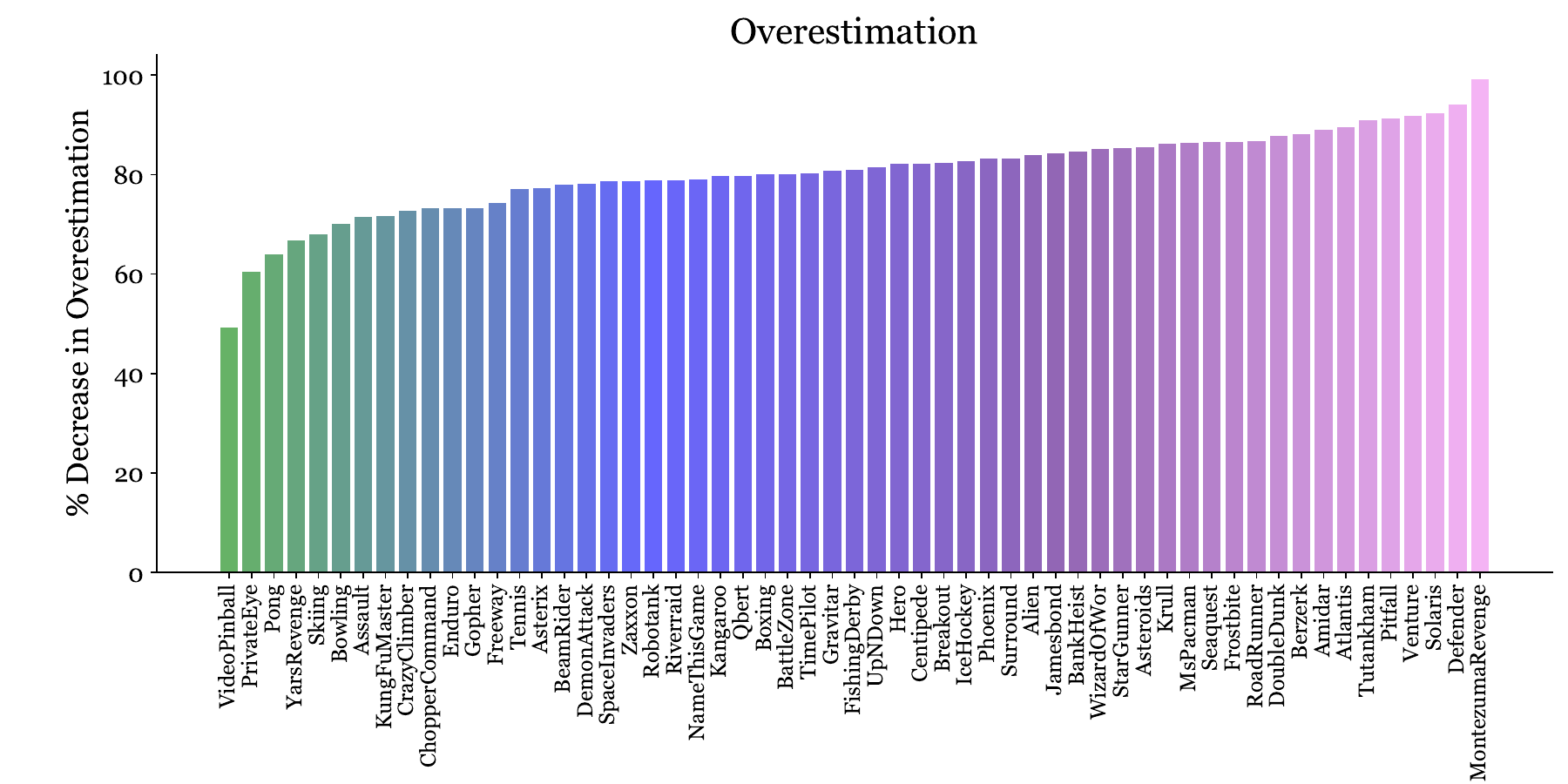}
    \end{subfigure}\hfill 
    \begin{subfigure}[c]{0.495\textwidth}
        \centering
        \includegraphics[width=\linewidth]{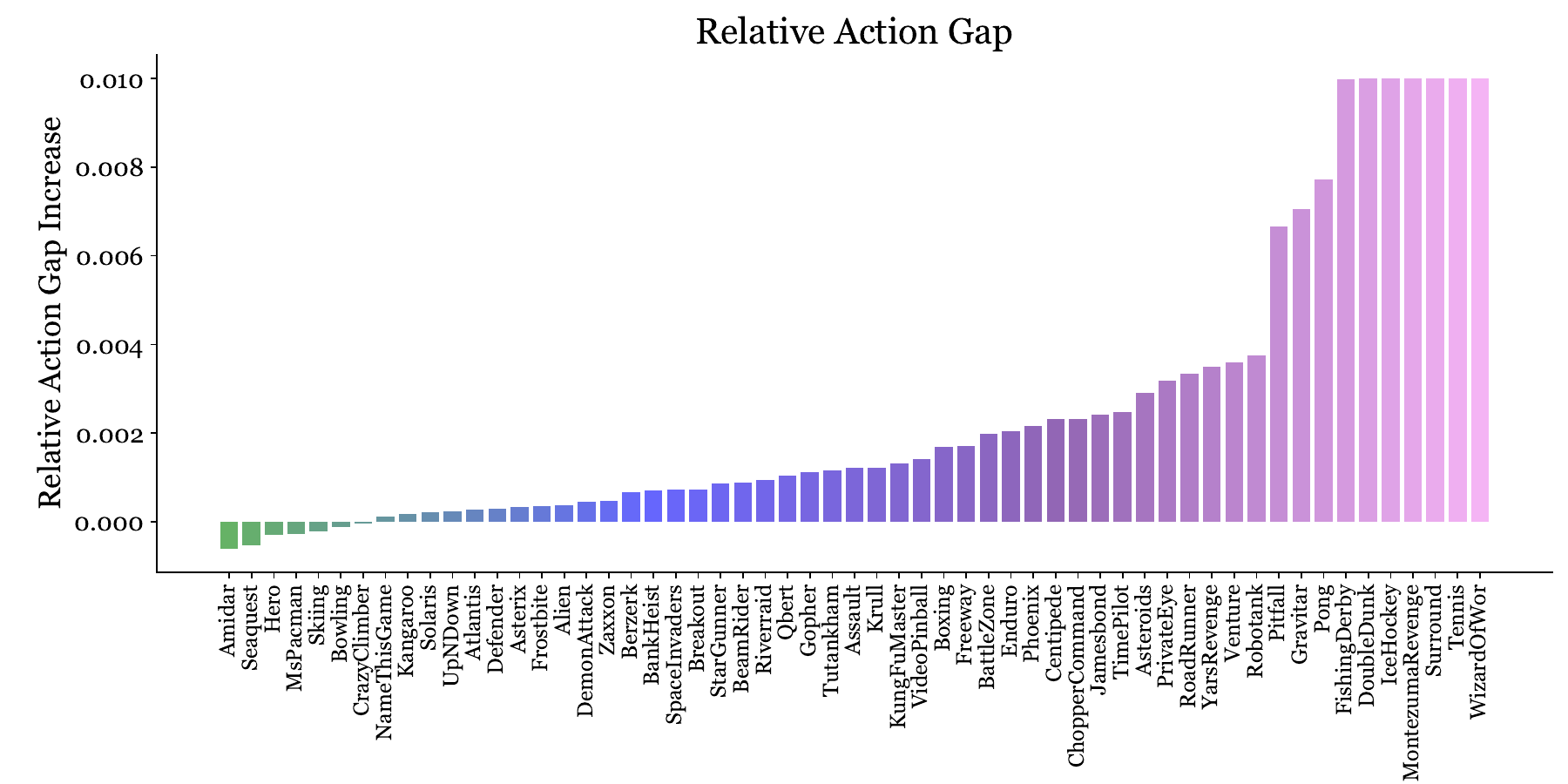}
    \end{subfigure}
    \caption{(\textit{left}) The percentage \textit{reduction} in overestimation when using IB-DQN over DQN, as a percentage of DQN’s average overestimation area under the curve.
    In all games, IB-DQN reduces overestimation over DQN. (\textit{right}) The \textit{increase} in relative action gap from using IB-DQN instead of DQN.
    In the vast majority of games, IB-DQN increases the relative action gap compared to DQN. Values are clipped at 0.01 for visibility.}
    \label{fig:aux_metrics}
\end{figure*}

\subsection{Value Function Stability: Overestimation and Action Gaps}
Aside from performance and sample efficiency, there are other metrics related to value-function dynamics that are known to be correlated with stability and performance in deep RL.
In particular, we examined two such metrics, overestimation and action gaps.
Overestimation refers to the action-value prediction exceeding the true expected return~\citep{ddql}.
Reducing overestimation has historically been correlated with improved stability and performance in deep RL~\citep{double_dqn,td3}.

Figure~\ref{fig:aux_metrics} (\textit{left}) shows the percentage reduction in value overestimation when using IB-DQN($n$) over DQN, measured as percentage of DQN's overestimation.
Overestimation is measured by comparing predicted returns to achieved returns during evaluation phases (see Appendix~\ref{appendix:experiments}).
As DQN overestimates in all games (see Figure~\ref{Atari57:Overestimation:page_2} in Appendix~\ref{appendix:results}), a reduction exceeding 100\%
indicates underestimation, while a reduction between 0\% and 100\% represents a reduction in, but not elimination of, overestimation.
IB-DQN exhibits reduced overestimation in all games.

In standard DQN, a positive TD error increases the action-value of one action and indirectly modifies the other action-values via changes to the penultimate layer's outputted features.
In IB-DQN, the updates are explicit such that positive TD error increases the values of all actions and a negative TD error decreases the value of all actions.
A TD error on one action is distributed across actions, perhaps tempering the outlier increases in the maximum action-value used for bootstrapping.
This sensitivity to global changes in value and the implicit penalty on increasing the mean of the residuals (Equation~\ref{vst_optimization}) may be why we observe reduced overestimation.
The full curves depicting overestimation throughout training are provided in Figure~\ref{Atari57:Overestimation:page_2} in Appendix~\ref{appendix:results}.
These curves also depict the overestimation of IB-DQN($1$), where we find that it exhibits less overestimation than DQN, but more than IB-DQN($n$).

Note that IB-DQN should be complementary to other methods for overestimation reduction, like Double DQN~\citep{double_dqn}.
Double DQN uses both the Q-network and target network to produce bootstrap targets.
The use of separate networks in the bootstrap target mitigates the maximization bias.
The use of two networks in the bootstrap target is still present when the ME layer is added to Double DQN.
Thus, the two methods should be complementary.

The second metric we measured was the action gap.
The \textit{action gap}~\citep{action_gap} refers to the difference between the maximum action-value and the second-highest action-value.
In deep RL, gaps between action-values in a state are often quite small.
For example, \citet{dueling} report that the average action-value across states of a Double DQN agent trained on the game Seaquest was 15.
However, the average action gap was a mere 0.04.
A consequence of small action gaps is that minor changes to action-values can rapidly change the greedy actions.
For these reasons, increasing the action gap is known to be generally beneficial in deep RL~\citep{increasing_action_gap}.

Figure~\ref{fig:aux_metrics} (\textit{right}) shows the increase in the relative action gap when transitioning from DQN to IB-DQN($n$).
The relative action gap refers to the action gap divided by the absolute value of the average Q-value.
We measure the relative action gap because IB-DQN produces much smaller action-values than DQN, as suggested by our overestimation results.
In a large majority of environments, IB-DQN exhibits an increased relative action gap.
This increase in relative action gap may be because our learned representations are themselves residuals.
Consequently, the majority of the model's capacity need not be devoted to fitting a large baseline value common to all actions, and can instead be used to model the relative differences between action-values.

\subsection{Sensitivity Analysis of Mean-Scaling Coefficient}
\begin{figure*}[t]
    \centering
    \begin{subfigure}[b]{0.28\textwidth}
        \centering
        \includegraphics[width=\linewidth]
        {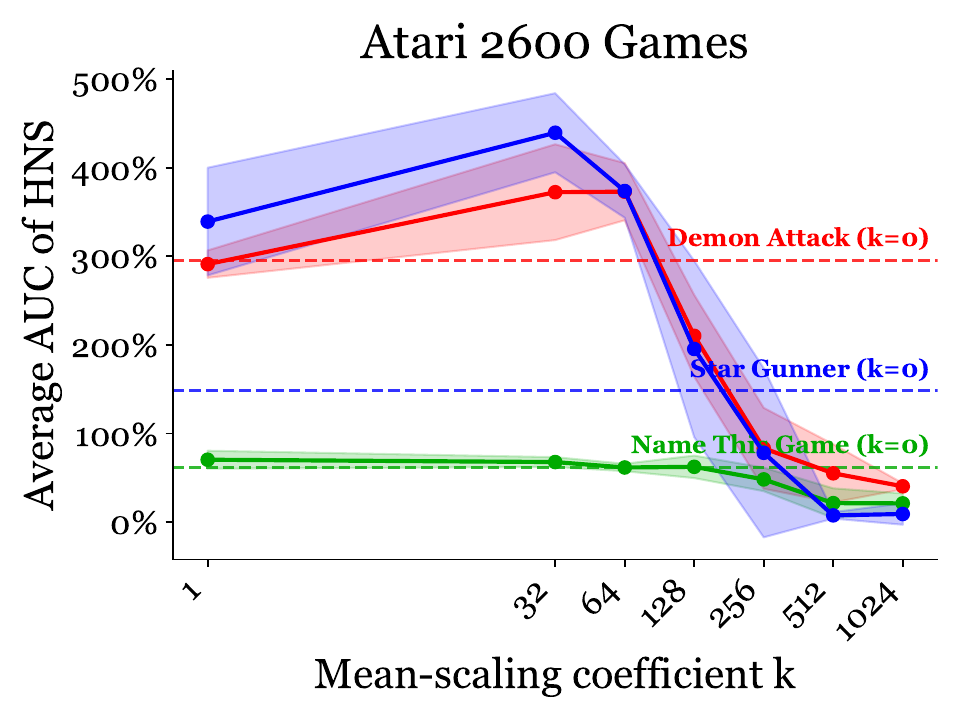}
    \end{subfigure}
 \begin{subfigure}[b]{0.28\textwidth}
        \centering
        \includegraphics[width=\linewidth]
        {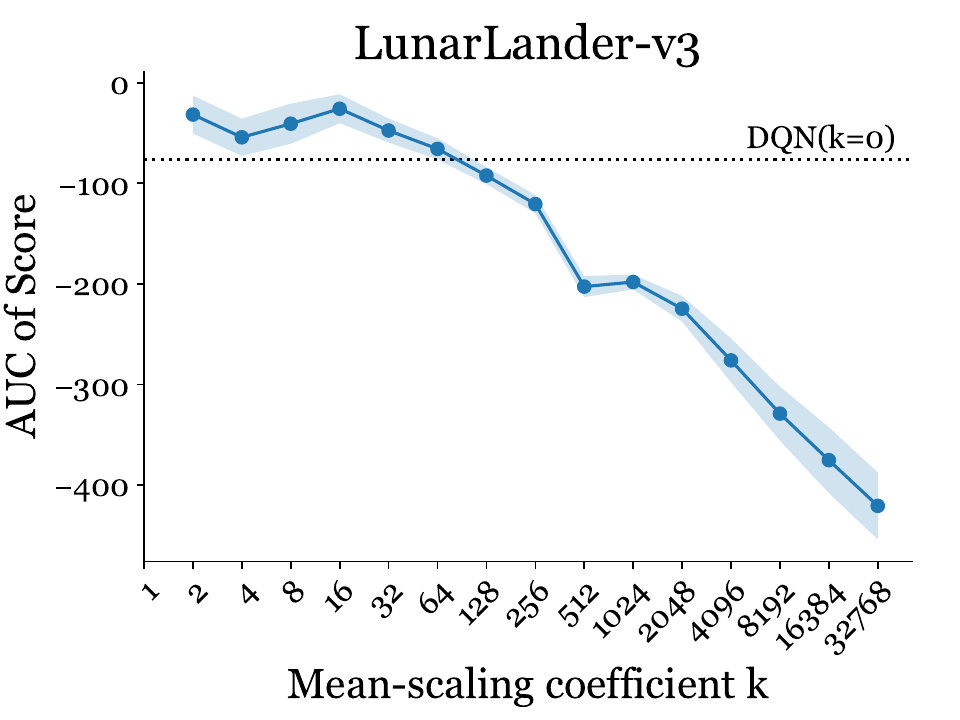}
    \end{subfigure}
    \begin{subfigure}[b]{0.28\textwidth}
        \centering
        \includegraphics[width=\linewidth]{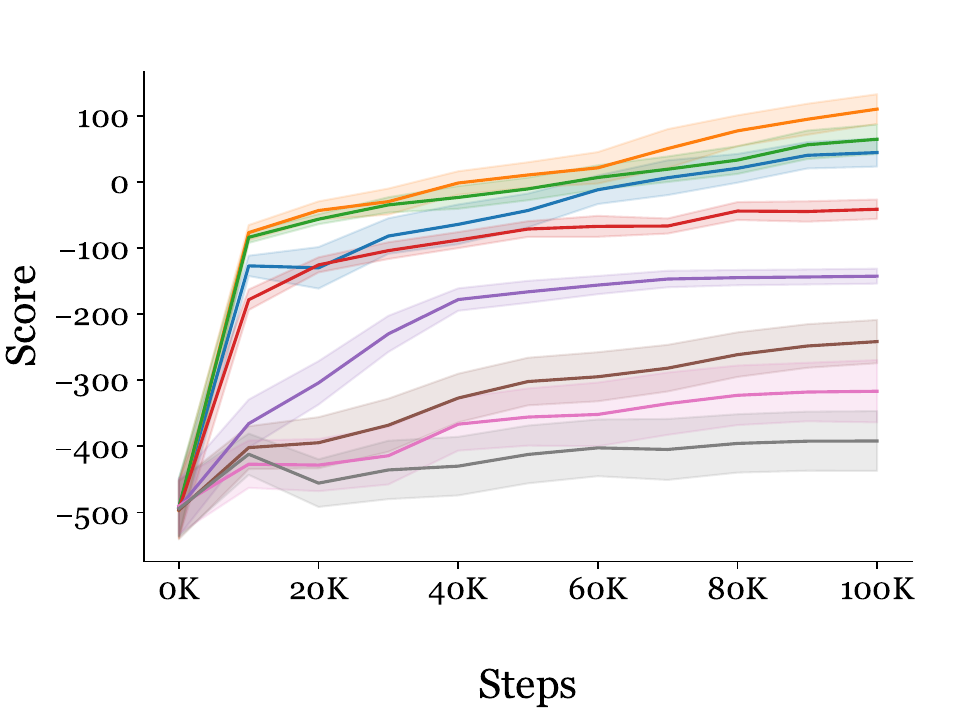}
    \end{subfigure}
    \begin{subfigure}[b]{0.1\textwidth}
        \centering
        \raisebox{4ex}{\includegraphics[width=\linewidth]{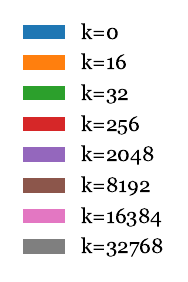}}
    \end{subfigure}
  \caption{\textbf{Sensitivity Analysis of $k$.} (\textit{left}) The average area under the curve (AUC) of the human-normalized score for several values of $k$ (log scale) on Atari 2600 games, where the shaded region depicts the standard deviation across five seeds.
  (\textit{center}) Similar to the \textit{left} figure, but for LunarLander-v3, where the shaded region depicts a  95\% confidence interval for the mean across 120 seeds.
  (\textit{right}) The corresponding LunarLander learning curves for a subset of values of $k$, again with the shaded region depicting a 95\% confidence interval.}
  \label{fig:lunar_sensitivity}    
\end{figure*}

To examine the sensitivity of IB-DQN to its mean-scaling coefficient $k$, we ran it for several values of $k$ on three Atari 2600 games (Figure~\ref{fig:lunar_sensitivity} (\textit{left})) for five seeds each and on Gymnasium's~\citep{gymnasium} LunarLander-v3 environment (Figure~\ref{fig:lunar_sensitivity} (\textit{center})) for 120 seeds each.
All environments share similar qualitative results.
For some values of $k$, performance exceeds that of DQN, which is equivalent to $k=0$.
Then, as $k$ grows too large, performance drops dramatically, likely due to gradients being scaled too much along the all-ones direction.
Figure~\ref{fig:lunar_sensitivity} (right) depicts the corresponding LunarLander-v3 learning curves for several values of $k$, where we can visibly see the degradation as $k$ grows large.

In this sensitivity analysis, we use the same step-size for all $k$, and our results should be interpreted accordingly.
In principle, we should tune the step size as we modify $k$, as we did in the tabular experiments.
Smaller step sizes should be more suitable for larger values of $k$.
Nonetheless, the general result is that for any fixed step-size, there is some large value of $k$ beyond which performance degrades.

\section{Conclusion and Discussion} \label{sec:discussion}

In this paper, we presented the \textit{\textbf{mean-expansion layer}} (ME layer), a simple way to represent and share values across actions.
The ME layer allows us to avoid learning an explicit baseline parameter and merely learn a vector of residuals.
This layer introduces no extra learnable parameters and does not change the underlying algorithm.
Applying this layer at the end of a Q-network produces reliable improvements to DQN and IQN in terms of sample efficiency and performance in aggregate across 57 Atari games.
It also provides large reductions in overestimation and increases the action gap, which are generally desirable in deep RL.

The ME layer applied to RL has its limitations.
The first is that performance degrades as $k$ grows large.
Therefore, selecting $k=n$ in large action spaces may lead to instabilities.
This limitation is not specific to the ME layer; large action spaces generally create more difficult learning problems, including for Q-learning methods.
Fortunately, selecting $k=1$ is well-conditioned and implies a baseline of $\nicefrac{\mu_\vq}{2}$, regardless of the number of actions.
As such, setting $k=1$ can be better than the implicit baseline of $0$ in standard Q-learning, as supported by our empirical results.

A second limitation is that there is a discrepancy between the conceptual framing of the ME layer and its practical usage.
While the ME layer is used to produce vector-valued outputs of action-values, feedback is in the form of sample-based updates to individual state-action pairs.
This discrepancy is also present in Dueling DQN, where updates to individual state-action pairs update the baseline $B(s)$, which represents the mean action-value in the state.
The updates are according to the behavior policy, which typically sample some actions in different proportions to others.
Nonetheless the mean is updated without accounting for variation in action-selection.
Examining the impact of this discrepancy is left for future work.

Our work casts a novel perspective on accelerating learning across actions in deep RL, especially as compared to Dueling DQN.
The original dueling network architecture introduces extra parameters, a modified architecture, and a learned state-specific baseline.
Any of one these could play an important role in improving performance.
Our work suggests that the key ingredient may be the state-specific baseline, as we improve performance without extra parameters or large architectural modifications.
We show that value-sharing does not necessarily require explicit separation of shared and independent components.
Instead, we can have each action's output carry both individual and shared information effectively.

Despite understanding fundamental aspects of the ME layer, there are several open questions.
We lack a comprehensive understanding of how the ME layer impacts the value function's learning path, how and when it improves performance, and its impact on convergence guarantees.
Moreover, the generalization of the ME layer to continuous actions remains an open question, and we still lack a mechanistic explanation for the ME layer's impact on overestimation and action gaps.

\section*{Acknowledgements}
We thank the anonymous reviewers, as well as Alex Lewandowski, Diego Gomez Noriega, Esraa Elelimy, and Roshan Shariff for their helpful feedback which improved the paper and its presentation. 

This research was supported in part by the Natural Sciences and Engineering Research
Council of Canada (NSERC), the Canada CIFAR AI Chair Program, and the Alberta Machine Intelligence Institute (Amii).
Prabhat Nagarajan is supported by the Alberta Innovates Graduate Student Scholarship.
Computational
resources were provided in part by the Digital Research Alliance of Canada.

\section*{Impact Statement}
This paper presents work whose goal is to advance the field of Machine Learning. There are many potential societal consequences of our work, none
which we feel must be specifically highlighted here.

\bibliography{citations}
\bibliographystyle{icml2026}


\clearpage
\newpage
\appendix

\section{Theoretical Results} \label{appendix:theory}

\subsection{Baseline Analysis}
\OptimalBaseline*

\begin{proof}
    Let $f(b)$ be the squared Euclidean norm of $\vu(\vq, b)$.
    \begin{align}
        f(b) = \|\vu(\vq, b)\|_2^2 = \sum_{i=1}^{n} (q_i - b)^2 + b^2.
    \end{align}
    This function is convex, so we can apply the first derivative test to find its minimum.
    \begin{align}
        \frac{df}{db} = 2b - 2\sum_{i=1}^n (q_i - b).
    \end{align}
    Setting the derivative to zero, we obtain
    \begin{align}
        0 &= 2b - 2\sum_{i=1}^n (q_i - b) \\
        0 &= b - \sum_{i=1}^n q_i + nb \\
        \sum_{i=1}^n q_i &= (n+1)b \\
        b &= \frac{\sum_{i=1}^n q_i}{n+1}.
    \end{align}
\end{proof}

\begin{restatable}{prop}{GoodBaselines}{\text{(Norm-reducing baselines)}.}
\label{prop:good_baselines}
Let $\mu_\vq = \nicefrac{\sum_{i=1}^{n} q_i}{n}$.
If $\mu_\vq < 0$, then the strictly norm-reducing baselines $b$ are the ones that satisfy $2b^{*} < b < 0$.
If $\mu_\vq > 0$, the strictly norm-reducing baselines satisfy $0 < b < 2b^{*}$.
\end{restatable}

\begin{proof}

    Let $f(b)$ be the squared Euclidean norm of $\vu(\vq, b)$.
    \begin{align}
        f(b) = \|\vu(\vq, b)\|_2^2 = \sum_{i=1}^{n} (q_i - b)^2 + b^2.
    \end{align}

Expanding,
\[
\begin{aligned}
f(b)
&= \sum_{i=1}^n \left( q_i^2 + b^2 - 2 q_i b \right) + b^2 \\
&= \sum_{i=1}^n q_i^2 - 2 b \sum_{k=1}^n q_k + (n+1) b^2 .
\end{aligned}
\]

Using $\sum_{i=1}^n q_i = n \mu_\vq$, we obtain
\[
f(b) = \|\vq\|_2^2 - 2 b n \mu_\vq + (n+1) b^2 .
\]

Thus the norm of $\vu(\vq, b)$ is less than the norm of $\vq$ when
\begin{align*}
    -2 b n \mu_q + (n+1) b^2 &< 0 \\
    (n+1) b^2 &< 2 b n \mu_q \\
    b^2 &< 2 b b^{*} \\
    b(b - 2 b^{*}) &< 0 .
\end{align*}

This inequality holds if either $b < 0$ or $(b - 2b^{*}) < 0$, but not both.

If $\mu_\vq < 0$, then $b^{*} < 0$.
Conversely, if $\mu_\vq > 0$, then $b^{*} > 0$.
Thus, if $b^\star < 0$, then the norm is reduced for $2 b^\star < b < 0$.
If $b^\star > 0$, then the norm is reduced for $0 < b < 2 b^\star$.
This implies that if $\mu_\vq \neq 0$, then $\|\vu(\vq,b)\|_2^2$ is less than $\|\vq\|_2^2$.
\end{proof}

\begin{restatable}{prop}{ResidualMean}
\label{thm:residual_mean}
Let $\vq \in \mathbb{R}^n$ and $b \in \mathbb{R}$.
Assume $\mu_{\vq} \neq 0$.
The absolute value of the mean of the residual vector, $g_{\vq}(b) = |\frac{1}{n}\sum_{i=1}^{n} (q_i - b)| = |\mu_{\vq} - b| = \mu_{\vz}$, decreases from $b=0$ to $b=\mu_{\vq}$, regardless of the sign of $\mu_{\vq}$.
\end{restatable}
\begin{proof}
    \[
    g_{\vq}(b) =
    \begin{cases}
    \mu_{\vq} - b & b \le \mu_{\vq} , \\
    b - \mu_{\vq}  & b > \mu_{\vq}.
    \end{cases}
    \]
So, if $b < \mu_{\vq}$, $\frac{\mathrm{d}g_{\vq}}{\mathrm{d}b} = -1$, and if $b > \mu_{\vq}$, $\frac{\mathrm{d}g_{\vq}}{\mathrm{d}b} = 1$.
Thus, if $\mu_{\vq} > 0$, then on the interval $b \in (0, \mu_{\vq})$ $g_{\vq}(b)$ is decreasing.
If $\mu_{\vq} < 0$ then on the interval $b \in (\mu_{\vq}, 0)$ $g_{\vq}(b)$ is increasing.
Thus, $g_{\vq}(b)$ is decreasing from $b=0$ to $b = \mu_{\vq}$.
\end{proof}

\begin{restatable}{prop}{ResidualNorm}
\label{thm:residual_norm}
Let $\vq \in \mathbb{R}^n$ and $b \in \mathbb{R}$.
Assume $\mu_{\vq} \neq 0$.
The squared Euclidean norm of the residual vector $\|\vq - b\vone\|_2^2$ has decreasing magnitude from $b=0$ to $b=\mu_{\vq}$, regardless of the sign of $\mu_{\vq}$.
\end{restatable}

\begin{proof}
Define $f_\vq(b) = \lVert \vq - b\mathbf{1} \rVert_2^2 = \sum_{i=1}^n (q_i - b)^2  $.

To understand whether $f_\vq(b)$ is increasing or decreasing as a function of $b$, we examine its derivative:
\[
\frac{d f_\vq}{db}
= -2\sum_{i=1}^n (q_i - b)
= 2\sum_{i=1}^n (b - q_i)
= 2n(b - \mu_{\bm{q}}).
\]

Thus, $f_\vq(b)$ is decreasing when
\[
2n(b - \mu_{\bm{q}}) < 0 \quad \Longleftrightarrow \quad b < \mu_{\bm{q}},
\]
and increasing when $b > \mu_{\bm{q}}$.
The norm of the residual vector is minimized at $b = \mu_\vq$, a known fact. 
We can examine the two cases.
In the first case, if $\mu_{\bm{q}} < 0$, then for $b \in (\mu_{\bm{q}}, 0)$ we have $b > \mu_{\bm{q}}$, so $f_\vq(b)$ is increasing on this interval and hence decreases from $b=0$ to $\mu_{\bm{q}}$.
In the second case, if $\mu_{\bm{q}} > 0$, then for $b \in [0, \mu_{\bm{q}})$ we have $b < \mu_{\bm{q}}$, so $f_\vq(b)$ is decreasing on this interval.

Thus, from $b=0$ to $b=\mu_{\bm{q}}$, the norm of the residual vector decreases regardless of the sign of $\mu_{\bm{q}}$.
\end{proof}

\subsection{Analysis of the Mean-Expansion Layer}

\BaselineRelationship*

\begin{proof}   
    Let $S_\vq = \sum_{i=1}^n q_i$ and let $S_\vz = \sum_{i=1}^{n} z_i$.
    As Q-values are constructed with $Q(s,a) = Z(s,a) + \frac{k}{n}\sum_{a'} Z(s,a')$, then for $k \ge 0$, we can sum all Q-values:
    \begin{align*}
        S_\vq &= S_\vz  + \frac{k}{n}S_z \\
        \frac{k}{n}S_\vq &= \frac{k}{n}S_\vz  + \frac{k^2}{n}S_z \\
        \frac{k}{n}S_\vq &= \frac{k}{n}S_\vz (k+1) \\
        \frac{\frac{k}{n}S_\vq}{k+1} &= \frac{k}{n}S_\vz \\
        \frac{S_\vq}{\frac{n}{k}(k+1)} &= \frac{k}{n}S_\vz \\
        k \mu_\vz &= \frac{S_\vq}{n + \frac{n}{k}}.
    \end{align*}
\end{proof}

\subsection{Properties of the Mean-Expansion Layer} \label{appendix:theory:properties}

In Equation~\ref{equation:geometry}, we described the geometric properties of the mean-expansion layer as stretching the mean component of the vector.
This framing can be seen below.
\begin{align*}
    \vq &= \left(\mI+\frac{k}{n}\mJ\right)\vx \\
               &= \left(\mI + \frac{k}{n}\mJ + \frac{1}{n}\mJ - \frac{1}{n}\mJ\right)\vx \\
               &= \left(\frac{k}{n}\mJ\vx\right) + \left(\frac{1}{n}\mJ\vx\right) + \left(\mI - \frac{1}{n}\mJ\right)\mathbf{x} \\
               &= (k+1)\left(\frac{1}{n}\mJ\vx\right) + \left(\mI - \frac{1}{n}\mJ\right)\vx \\
               &= (k+1)\left(\frac{1}{n}\mJ\vx\right) + \left(\vx - \frac{1}{n}\mJ\vx\right).
\end{align*}
The first term stretches the mean and the second term subtracts the mean component from the original vector.

\RegularizationInterpretation*

\begin{proof}
    Let
\begin{align*}
    f(\vz) &= \frac{1}{n}\normtwo{\vz - \vq}^2 + k \left(\frac{1}{n}\sum_{i=1}^n z_i\right)^2 \\
     &= \frac{1}{n}(\vz - \vq)^\top(\vz - \vq) + k \left(\frac{1}{n}\sum_{i=1}^n z_i\right)^2 \\
     &= \frac{1}{n}(\vz - \vq)^\top(\vz - \vq) + k \left(\frac{1}{n}\mathbf{1}^\top\vz\right)^2 \\
     &= \frac{1}{n}(\vz - \vq)^\top(\vz - \vq) + \frac{k}{n^2} \vz^{\top}\mathbf{1}\mathbf{1}^\top\vz \\
     &= \frac{1}{n}(\vz - \vq)^\top(\vz - \vq) + \frac{k}{n^2} \vz^\top\mJ\vz
\end{align*}

We can then take the gradient with respect to $\vz$,
\begin{align*}
    \nabla_{\vz} f(\vz) &= \frac{2}{n}(\vz - \vq) + \frac{2k}{n^2}\mJ\vz.
\end{align*}
Setting the gradient to 0 and dividing by 2, we get:
\begin{align*}
    0 &= \frac{1}{n}(\vz - \vq) + \frac{k}{n^2}\mJ\vz \\
    0 &= (\vz - \vq) + \frac{k}{n}\mJ\vz \\
    \vq &= \vz + \frac{k}{n}\mJ\vz \\
    \vq &= \left(\mI+\frac{k}{n}\mJ\right)\vz.
\end{align*}
\end{proof}

\section{Tabular Update Rules} \label{appendix:tabular_update}
To derive the update rules for our residuals, we do Q-learning on the implied Q-values.
Let the Q-learning error at time $t$ be $\delta_t = r_{t+1} + \gamma \max_{a'}Q(s_{t+1},a') - Q(s_t,a_t)$.
We then apply gradient descent to the Q-learning error and minimize $\mathcal{L} = \frac{1}{2}\delta_t^2$.
For transition $(s_t,a_t,r_{t+1},s_{t+1})$, the derivative with respect to $Z(s_t,a)$, for action $a$ in state $s_t$ is 
\begin{align}
    \frac{\partial \mathcal{L}}{\partial Z(s_t,a)} &= \frac{\partial \mathcal{L}}{\partial Q(s_t,a)} \frac{\partial Q(s_t,a)}{\partial Z(s_t,a)} \\
    & = (- \delta_t) \frac{\partial Q}{\partial Z(s_t,a)} \\
     & = (- \delta_t) \left(\mathds{1}_{\{a_t = a\}} + \frac{k}{n}\right).
\end{align}

The interpretation, as shown in Equations~\ref{action_update} and~\ref{baseline_update}, is that the chosen action is updated in greater proportion to the TD error than the other actions.
That is, the update adds $\alpha_t(1 + \nicefrac{k}{n})\delta_t$ to the chosen action's residual, whereas the update adds $\alpha_t(\nicefrac{k}{n})\delta_t$ to the residuals of the chosen action.

\section{Semantics of the Dueling Network Architecture} \label{appendix:dueling}

In this Appendix, we discuss dueling methods more closely.

\subsection{Dueling Deep Q-networks}
Dueling DQN~\citep{dueling} in fact learns a mean-residual decomposition of action-values, under the terms ``value'' and ``advantage'' for the mean and residuals, respectively.
Its modification to the DQN architecture is to learn $B(s; \bm{\theta})$ in a baseline stream and $Z(s,\cdot; \vtheta)$  in a residual stream.
The baseline and residual streams are aggregated to construct the Q-values
\begin{align} \label{equation:dueling}
    Q(s,a; \bm{\theta}) &= B(s;\bm{\theta}) + Z(s,a; \bm{\theta}),
\end{align}
or $Q(s,\cdot; \bm{\theta}) = B(s;\bm{\theta})\mathbf{1} + Z(s,\cdot; \bm{\theta})$.
Prior to outputting the residual vector $Z(s, \cdot; \bm{\theta})$, the residual stream first produces raw outputs $X(s,a;\bm{\theta})$, which are then zero-centered to produce the action residuals $Z(s, \cdot;\bm{\theta})$:
\begin{align} \label{equation:residual_centering}
Z(s,a; \bm{\theta}) = X(s,a; \bm{\theta}) - \frac{1}{n}\sum_{a' \in \mathcal{A}} X(s,a'; \bm{\theta}).
\end{align}
\citet{dueling} find this architecture to be effective, providing a substantial improvement over DQN.

We can analyze the semantics of the dueling network architecture through its construction of action-values.
First, note that due to the zero-centering in Equation~\ref{equation:residual_centering}, the sum of residuals is zero:
\begin{align*}
\sum_{a} Z(s,a; \bm{\theta}) &= \sum_a \Big(X(s,a; \bm{\theta}) - \frac{1}{n} \sum_{a'} X(s,a'; \bm{\theta})\Big) & \\
&= \sum_a \Big(X(s,a; \bm{\theta})\Big) - \frac{n}{n} \sum_{a'} X(s,a'; \bm{\theta}) \\
& =0.
\end{align*}

Thus, summing the Q-values, we get
\begin{align*}
\sum_a Q(s,a; \vtheta) &= \sum_a \Big( B(s; \vtheta) + Z(s,a; \vtheta)\Big) \\
\sum_a Q(s,a; \bm{\theta}) &= n B(s; \bm{\theta}) + \sum_{a} Z(s,a;\vtheta) \\
\sum_a Q(s,a; \bm{\theta}) &= n B(s; \bm{\theta}).
\end{align*}

Dividing by $n$:
\begin{equation*}
\frac{1}{n} \sum_a Q(s,a; \bm{\theta}) = B(s; \bm{\theta}).
\end{equation*}

Thus, Dueling DQN enforces
\begin{equation*}
\boxed{B(s; \bm{\theta}) = \frac{1}{n} \sum_a Q(s,a; \bm{\theta})},
\end{equation*}
i.e., the output of the baseline stream $B(s;\bm{\theta})$ is the mean action-value for the state. As~$Q(s,a; \bm{\theta}) = B(s;\vtheta) + Z(s,a; \bm{\theta})$, each $Z(s,a);\vtheta$ is a residual with respect to the mean action-value.
Thus, Dueling DQN implements a \textit{mean-residual decomposition}.

\subsection{Regularized Dueling Q-learning}

The recently proposed Regularized Dueling Q-learning (RDQ)~\citep{rdq} uses the same Dueling DQN architecture and aggregates values according to Equation~\ref{equation:dueling}, but without the centering of the raw network outputs in Equation~\ref{equation:residual_centering}.
In Dueling DQN, Equation~\ref{equation:residual_centering} plays an important role in providing identifiability to the system in Equation~\ref{equation:dueling} to ensure that the $n+1$ variables, i.e., a single baseline and $n$ residuals, are not unconstrained.

RDQ adopts a different approach to constrain values through explicit regularization of the network outputs.
In particular, if the network outputs $B(s;\bm{\theta})$ and $Z(s,a;\bm{\theta})$ for all $a \in \mathcal{A}$, RDQ adds the following penalty term to the DQN loss:
\begin{align*}
    \beta \left(B(s; \bm{\theta})^2 + \sum_{a}Z(s,a; \bm{\theta})^2\right),
\end{align*}
where $\beta$ is a regularization coefficient.
In effect, this method encourages the baseline-residual vector of the optimal baseline in Equation~\ref{equation:optimal_baseline}.
Unlike IB-DQN or Dueling DQN, RDQ does not use the residuals and baseline to represent any specific predefined baseline.
Dueling DQN and IB-DQN both structurally enforce their baseline terms to be specific scalar multiples of the average action-value.
RDQ simply uses $B(s;\bm{\theta})$ generally as a baseline, and penalizes it to ensure a low-norm representation.

\onecolumn
\section{Mean-expansion Layer PyTorch Code} 
\label{appendix:mel_pseudocode}
The PyTorch~\citep{pytorch} code for the ME layer can be implemented in a dozen lines.

\begin{lstlisting}[language=Python]
import torch

class MeanExpansionLayer(torch.nn.Module):
    def __init__(self, mean_scaling_coefficient):
        super().__init__()
        self.register_buffer('scale', torch.tensor(1 + mean_scaling_coefficient))

    def forward(self, vec):
        mean = vec.mean(dim=-1, keepdim=True)
        residual = vec - mean
        output = self.scale * mean + residual
        return output
\end{lstlisting}

\section{Experimental Setup} \label{appendix:experiments}

In this Appendix, we discuss the training and evaluation details for our agents.

\subsection{Environment and training settings}
We train agents in the Arcade Learning Environment~\citep{ale}, using the environment protocol proposed by~\citet{revisitingale}.
We use sticky actions, where the agent's most recently executed simulator action is repeated in the subsequent frame with probability 0.25.
We expose the agents to the full action set of 18 actions in all games.
The only termination signal the agent receives is when the game is over, corresponding to the end of an episode.

The agents are trained for 50M timesteps.
These 50M timesteps are divided into 200 training phases, each lasting 250k timesteps.
After each training phase is an evaluation phase lasting 125k timesteps.
During evaluation, the agents deploy an $\epsilon$-greedy policy with $\epsilon=0.001$.
During both training and evaluation, episodes that reach the 30-minute time limit, which corresponds to 27,000 timesteps, are truncated.

\subsection{Algorithm training details and hyperparameters}

The training details of DQN largely follow~\citet{dqn}.
We use their architecture, pre-processing schema, and other details unless otherwise mentioned.
One distinction worth noting is that we have our agents deploy an $\epsilon$-greedy policy during training that is annealed from $\epsilon=1$ to $\epsilon=0.01$ over the first 1M timesteps.
Another distinction is that we use the Adam optimizer and the squared error loss as opposed to RMSprop with Huber loss.
Table~\ref{hypers} contains more implementation details and hyperparameters.
For IQN, our training details largely follow~\citet{iqn}. We use their same architecture. Our exploration $\epsilon$-decay schedule matches DQN. We use Adam with step size $5 \times 10^{-5}$ and $\epsilon=3.125 \times 10^{-4}$, and other hyperparameters follow the original paper.

\begin{table*}[]
\centering
\caption{DQN training hyperparameters. We use * to indicate hyperparameters that differ from the original DQN paper.}
\begin{small}
\resizebox{\textwidth}{!}{%
\begin{tabular}{@{}p{0.3\linewidth}|p{0.1\linewidth}|p{0.5\linewidth}@{}}\toprule
  \textbf{Hyperparameter} & \textbf{Value} & \textbf{Description} \\ \midrule
minibatch size & 32 & Sample batch size of gradient updates. \vspace{1mm} \\
replay memory capacity & 1,000,000 & Number of recent transitions stored in the replay buffer. \vspace{0.5mm} \\
agent history length & 4 & Number of previous frames stacked in state representation. \vspace{1mm} \\
target network~{update frequency} & 2,500 & Frequency (in terms of gradient updates) of target network updates. \vspace{1mm} \\
discount rate & 0.99 &  Value of $\gamma$. \vspace{1mm} \\
action repeat & 4 &  In one timestep, actions are repeated for multiple simulator frames. \vspace{1mm} \\
update frequency & 4 &  Frequency (in timesteps) of parameter updates. \vspace{1mm}\\
replay start size & 50K &  Minimum number of transitions in the replay buffer required before parameter updates begin. \vspace{1mm} \\
initial exploration & 1.0 &  Initial value of $\epsilon$ used for $\epsilon$-greedy exploration. \vspace{1mm} \\
*final exploration & 0.01 &  Final value of $\epsilon$ used for $\epsilon$-greedy exploration. \vspace{1mm} \\
final exploration timestep & 1,000,000 &  The timesteps over which $\epsilon$ is linearly annealed to its final $\epsilon$. \vspace{1mm} \\
maximum episode length & 27,000 &  Timesteps after which an episode is truncated and the environment is reset. \vspace{1mm} \\
step size & $6.25 \times 10^{-5}$ &  The step size used by Adam. \vspace{1mm} \\
*Adam $\epsilon$ & $1.5 \times 10^{-4}$ &  The $\epsilon$ used by Adam. \vspace{1mm} \\
*Adam $\beta_1$ & 0.9 &  $\beta_1$ hyperparameter value in Adam. \vspace{1mm} \\
*Adam $\beta_2$ & 0.999 &  $\beta_2$ hyperparameter value in Adam. \\
\bottomrule
\end{tabular}
}
\end{small}
\label{hypers}
\end{table*}

\subsection{Metrics}
\textit{Human-normalized scores} (HNS) provide a mechanism for evaluating an agent's score in a manner that is comparable across games.
The human-normalized score~\citep{dqn} of an agent can be computed as
\begin{equation} \label{hns_eq}
    \text{score}_{\text{hns}} = \frac{\text{score}_{\text{agent}} - \text{score}_{\text{random}}}{\text{score}_{\text{human}} - \text{score}_{\text{random}}}.
\end{equation}
The human and random scores are taken from~\citet{dqnzoo}.

\paragraph{Measuring Overestimation} \label{appendix:overestimation}

In this paper, we follow \citet{ddql}'s protocol for measuring overestimation.
During evaluations, we deploy a near-greedy evaluation policy.
We only consider completed evaluation episodes, which are episodes that terminate or truncate due to reaching the maximum allowed episode length of 27,000 timesteps.
Any episodes that are truncated prematurely due to the evaluation phase ending are discarded.
Each state-action pair in a completed episode has a received empirical return.
On truncations due to time limit, we bootstrap the value of the final state to compute the empirical return.
Overestimation for a state-action pair is computed as the difference between the action-value prediction and the received empirical return.
We average these over all state-action pairs in all evaluation episodes across all evaluation phases in a training run across all seeds to produce a single scalar measurement of overestimation.
We then report the difference in overestimation between the two algorithms.

DQN agents (on Atari) employ reward clipping, where rewards are clipped to the range $[-1,1]$.
Since agents are trained to predict return estimates under this reward, we ensure that this clipping is also applied to compute returns.
Combined with discounting, clipping often causes returns to be much smaller than the raw game scores, which are undiscounted and unclipped.

\paragraph{Measuring the Action Gap}

To measure the action gap, we sample a minibatch from the buffer after each target network update and measure the action gap, i.e., the difference between the highest and second-highest action-value averaged across all (pre-transition) states in the minibatch.
These quantities are averaged across all measurements in a training phase.
We also maintain a running average of the last 1000 average minibatch action-values for each training phase.
The relative action gap for a training phase is computed as the average action gap divided by the absolute value of the average action-value with a stability term of 1e-8 added to the denominator.
These relative action gaps are averaged across all training phases and seeds.
Figure~\ref{fig:aux_metrics} (\textit{right}) reports the difference in the relative action gap between the two algorithms, clipped at 0.01 for presentation.

\subsection{Baselines}
For our Dueling DQN implementation, we used the exact same settings as DQN, but replace the network with the dueling network architecture.
The final convolution layer outputs $64$ filters of size $7 \times 7$ feature maps which are flattened into 3,136 outputs. 
Following these flattened outputs are two independent streams of two-layer fully connected networks.
Each stream has a hidden of layer of 512 units and ReLU activations~\citep{relu_activation}.
The first stream outputs a single scalar $B(s; \bm{\theta})$, and the second stream outputs an $n$-dimensional vector $\mathbf{x}(s;\bm{\theta})$, which is centered to output the residuals $\mathbf{z}(s; \bm{\theta}) = \mathbf{x}(s;\bm{\theta}) - \frac{1}{n}J\mathbf{x}(s;\bm{\theta})$.

RDQ shares the same dueling network architecture, but outputs the residuals $\mathbf{z}(s;\bm{\theta})$ directly, as opposed to first centering the raw outputs.
It also includes a regularization coefficient $\beta$ on the penalty $B(s;\bm{\theta})^2 + \sum_{a} Z(s,a;\bm{\theta})^2$.
We tested 5 different values of $\beta$ on Atari-5-Val~\citep{atari_5}, a subset of environments used for validation.
As long as $\beta$ was not too large, we found RDQ to be relatively robust.
We did not observe an improvement for changing $\beta$ from $0.001$, as was set by \citet{rdq}, and thus kept this value.

\section{Mean-Expansion Layer with RMSprop} 
\label{appendix:rmsprop}

\begin{wrapfigure}{r}{0.33\linewidth}
    \centering
    \includegraphics[width=\linewidth]{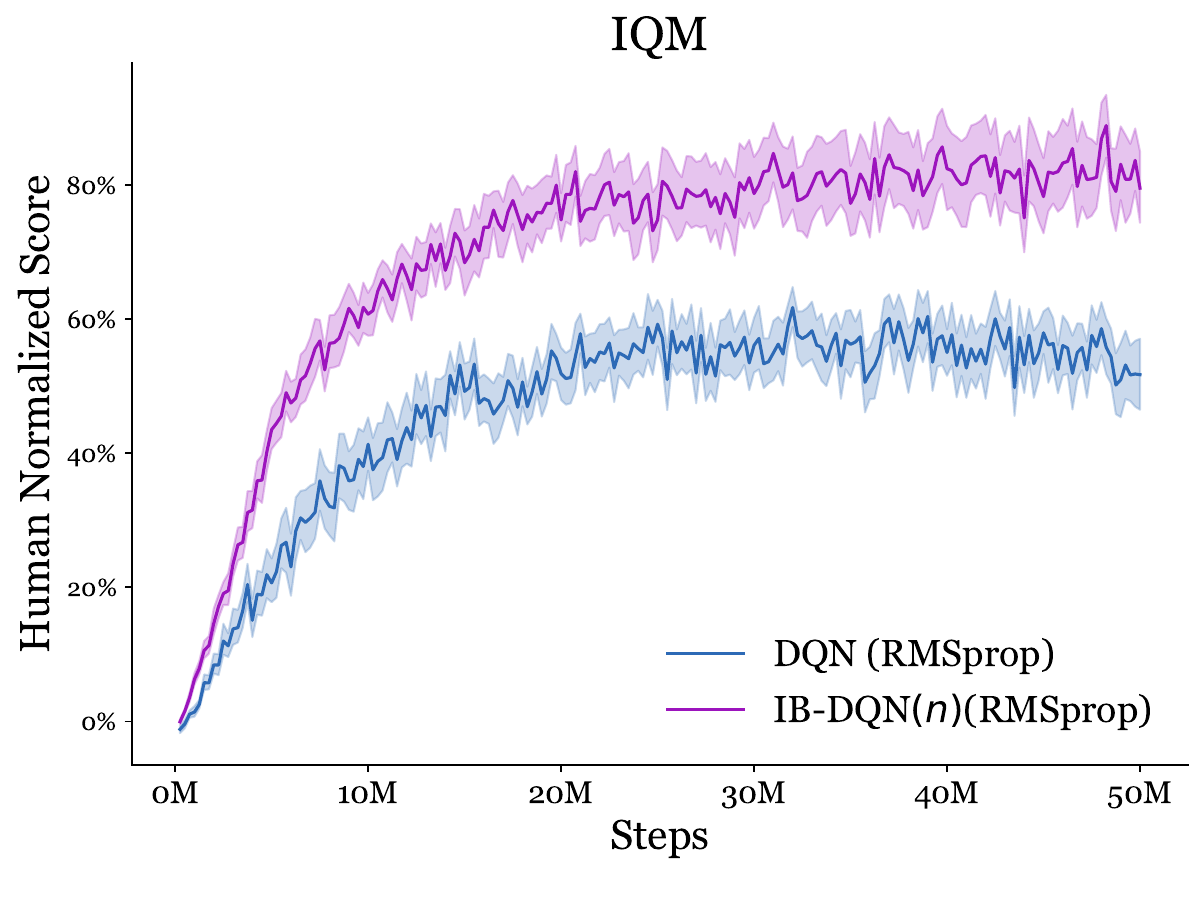}
    \caption{\textbf{The ME layer with RMSprop and the Huber loss.} The plot shows the interquartile mean of DQN with and without the ME layer across 55 games. All algorithms were run for three seeds per game.
    The shaded region depicts the 95\% stratified bootstrap confidence interval~\citep{statistical_precipice}.}
    \label{fig:hns_rmsprop}
\end{wrapfigure}
To examine how the ME layer impacts performance in deep RL outside of the Adam optimizer, we also tested with the RMSprop~\citep{rmsprop} optimizer and the Huber loss, with the optimizer hyperparameters matched that of the original DQN~\citep{dqn}.
In particular, we compare, across 55 environments and three seeds per environment, the interquartile mean in performance between DQN and IB-DQN($n$).
Consistent with our prior results, we find that IB-DQN($n$) substantially outperforms DQN in this setting, providing further evidence for the general efficacy of the ME layer.

\section{Per-Game Results} \label{appendix:results}

In this Appendix, we include the per-game results corresponding to the Atari 2600 experiments in Section~\ref{sec:experiments}.
We include the full learning curves for each game in Figure~\ref{Atari57:score:page_2}.
We also include curves that show the overestimation of DQN and IB-DQN throughout training in Figure~\ref{Atari57:Overestimation:page_2}.
Lastly, Table~\ref{table_results} reports the mean score across the final three evaluation phases for each algorithm and game.

\begin{table*}[!t] \centering
    \resizebox{0.85\columnwidth}{!}{
    \begin{tabular}{@{}l|c|c|c|c|c||c|c@{}}\toprule \textbf{Environment} & DQN & Dueling DQN & IB-DQN($1$) & IB-DQN($n$) & RDQ($\beta=0.001$) & IQN & IB-IQN($n$) \\ 
 \midrule 
\textsc{Alien} & 3446.7& 3686.3& 3602.8& 3751.2& \textbf{5075.3} & 3886.3& \textbf{4298.1} \\ 
\textsc{Amidar} & 704.1& 721.5& 1076.5& 1104.7& \textbf{1262.4} & 952.9& \textbf{980.4} \\ 
\textsc{Assault} & 1864.5& 1481.4& 1815.6& \textbf{2039.6} & 1698.4& 3479.8& \textbf{5159.2} \\ 
\textsc{Asterix} & 11725.6& 13975.8& 13689.3& 17452.8& \textbf{24383.4} & 30104.1& \textbf{39998.3} \\ 
\textsc{Asteroids} & 1071.2& 1167.6& 1012.9& 1073.4& \textbf{2856.2} & 1144.7& \textbf{1178.5} \\ 
\textsc{Atlantis} & 821470.0& 841251.7& 886831.7& 856981.7& \textbf{922795.0} & \textbf{941850.0} & 865951.7\\ 
\textsc{BankHeist} & 966.4& 1168.1& 1028.1& \textbf{1270.8} & 1245.3& 1203.8& \textbf{1463.0} \\ 
\textsc{BattleZone} & 28035.5& 28205.0& 27649.0& \textbf{33482.2} & 32937.7& 27673.7& \textbf{37347.3} \\ 
\textsc{BeamRider} & 5028.0& 3838.2& 4315.6& 4982.4& \textbf{5669.1} & 7803.1& \textbf{10219.5} \\ 
\textsc{Berzerk} & 628.5& 493.9& \textbf{691.2} & 576.8& 597.2& \textbf{718.2} & 593.9\\ 
\textsc{Bowling} & 28.6& 25.4& 30.5& \textbf{33.5} & 30.4& \textbf{39.4} & 35.8\\ 
\textsc{Boxing} & 94.5& \textbf{96.2} & 95.1& 95.2& 94.5& 98.0& \textbf{98.6} \\ 
\textsc{Breakout} & 91.3& \textbf{172.7} & 61.0& 133.8& 133.3& \textbf{309.3} & 288.0\\ 
\textsc{Centipede} & 4714.9& \textbf{5696.2} & 3308.1& 5013.4& 4606.5& 5862.1& \textbf{6395.2} \\ 
\textsc{ChopperCommand} & 5987.9& 5173.5& 6508.0& \textbf{8174.8} & 6142.7& 7857.3& \textbf{10551.8} \\ 
\textsc{CrazyClimber} & 106037.2& 121540.2& 110589.1& \textbf{124345.3} & 121425.0& 119368.5& \textbf{120757.1} \\ 
\textsc{Defender} & 9260.5& 9195.2& 7623.7& \textbf{22271.0} & 13766.2& 19813.2& \textbf{26517.9} \\ 
\textsc{DemonAttack} & 7450.4& 6052.2& 6607.9& \textbf{16491.9} & 6205.8& 47948.2& \textbf{55029.9} \\ 
\textsc{DoubleDunk} & -2.3& -5.3& \textbf{9.6} & -7.0& 3.1& -1.8& \textbf{2.0} \\ 
\textsc{Enduro} & 1773.8& 1695.6& 1839.1& 1748.2& \textbf{1897.2} & \textbf{2204.7} & 2127.2\\ 
\textsc{FishingDerby} & 14.6& 30.8& 15.5& 39.0& \textbf{44.0} & 42.5& \textbf{44.9} \\ 
\textsc{Freeway} & 33.8& 33.9& 33.8& \textbf{33.9} & 33.8& 33.9& \textbf{33.9} \\ 
\textsc{Frostbite} & 6151.2& 5714.5& 6710.5& 6697.5& \textbf{6775.6} & 6913.1& \textbf{8556.4} \\ 
\textsc{Gopher} & 15184.2& 13862.0& 13380.0& 13337.7& \textbf{17110.1} & \textbf{18495.5} & 14897.3\\ 
\textsc{Gravitar} & 1058.7& \textbf{1334.3} & 1261.3& 846.3& 1196.5& 876.2& \textbf{1125.9} \\ 
\textsc{Hero} & 26142.7& \textbf{27788.5} & 24745.0& 20885.0& 19005.1& 22822.0& \textbf{26173.6} \\ 
\textsc{IceHockey} & -7.1& -5.1& -4.5& -2.5& \textbf{4.9} & \textbf{9.8} & 7.9\\ 
\textsc{Jamesbond} & \textbf{848.4} & 770.2& 844.1& 648.9& 624.2& \textbf{679.4} & 539.6\\ 
\textsc{Kangaroo} & 7690.1& 1930.7& 8704.5& 11540.8& \textbf{12261.9} & \textbf{12918.2} & 11563.1\\ 
\textsc{Krull} & 8112.6& 8666.3& 8560.9& 8283.4& \textbf{8758.2} & 8887.2& \textbf{9273.0} \\ 
\textsc{KungFuMaster} & 23924.3& \textbf{26705.8} & 24500.2& 24845.0& 25119.4& 25362.0& \textbf{28962.8} \\ 
\textsc{MontezumaRevenge} & \textbf{0.0} & 0.0& 0.0& 0.0& 0.0& \textbf{0.0} & 0.0\\ 
\textsc{MsPacman} & 3236.6& 3882.0& 3301.2& 4222.4& \textbf{4318.7} & 4261.9& \textbf{4857.0} \\ 
\textsc{NameThisGame} & 6810.0& 5134.5& 6421.4& 6203.6& \textbf{8021.4} & 14222.0& \textbf{18136.8} \\ 
\textsc{Phoenix} & 5128.0& 4656.4& 7005.4& 4930.3& \textbf{8814.6} & 6782.4& \textbf{8448.9} \\ 
\textsc{Pitfall} & -62.2& -59.7& -157.4& \textbf{-8.9} & -16.1& \textbf{-12.3} & -23.5\\ 
\textsc{Pong} & 19.8& 19.8& \textbf{19.8} & 18.9& 19.6& 18.7& \textbf{19.6} \\ 
\textsc{PrivateEye} & \textbf{68.5} & -105.4& -241.1& -47.9& -103.0& \textbf{125.2} & 100.0\\ 
\textsc{Qbert} & 14445.4& 13913.9& 14310.4& \textbf{14499.8} & 13709.0& 15153.6& \textbf{16185.5} \\ 
\textsc{Riverraid} & 11169.2& 13658.7& 12110.3& 14068.4& \textbf{14854.6} & 13305.8& \textbf{13487.1} \\ 
\textsc{RoadRunner} & 53174.3& \textbf{59137.8} & 53886.1& 54619.0& 54881.5& 58900.4& \textbf{60329.1} \\ 
\textsc{Robotank} & 62.0& 62.7& 60.8& \textbf{67.4} & 67.0& \textbf{70.5} & 68.4\\ 
\textsc{Seaquest} & 4620.3& 9953.0& 7760.2& 7155.8& \textbf{18852.1} & 8418.8& \textbf{15313.7} \\ 
\textsc{Skiing} & -23539.1& \textbf{-10948.9} & -24026.5& -11007.6& -30000.0& -22549.1& \textbf{-10119.8} \\ 
\textsc{Solaris} & 770.5& 4.6& 913.6& \textbf{1189.4} & 1159.3& \textbf{1490.9} & 1273.4\\ 
\textsc{SpaceInvaders} & 2922.8& 2843.5& 3091.7& 3906.0& \textbf{7166.5} & \textbf{11839.3} & 10930.3\\ 
\textsc{StarGunner} & 28529.9& 46965.0& 56808.9& \textbf{62851.0} & 51615.8& 82905.3& \textbf{84907.1} \\ 
\textsc{Surround} & -0.5& -2.6& -10.0& 4.2& \textbf{5.3} & 6.7& \textbf{7.8} \\ 
\textsc{Tennis} & 22.3& 22.1& 6.7& 22.9& \textbf{23.3} & 17.1& \textbf{22.5} \\ 
\textsc{TimePilot} & 8569.5& 10599.5& 8405.2& 12774.8& \textbf{13706.2} & 9316.3& \textbf{16053.7} \\ 
\textsc{Tutankham} & 102.1& 171.5& 36.6& 195.5& \textbf{214.1} & \textbf{225.4} & 221.7\\ 
\textsc{UpNDown} & 10570.1& \textbf{46005.7} & 11025.4& 20121.6& 29311.0& \textbf{27168.5} & 25773.2\\ 
\textsc{Venture} & 1208.6& 0.0& \textbf{1295.9} & 238.7& 208.4& 476.6& \textbf{581.6} \\ 
\textsc{VideoPinball} & 239194.8& 219113.3& 283539.1& \textbf{315799.9} & 270117.5& 315433.3& \textbf{498255.0} \\ 
\textsc{WizardOfWor} & 4031.7& 4185.5& 6285.3& 7660.8& \textbf{7826.4} & \textbf{11630.1} & 8146.0\\ 
\textsc{YarsRevenge} & 53891.5& 61301.1& 58932.9& 56972.5& \textbf{66570.2} & 51722.7& \textbf{83357.6} \\ 
\textsc{Zaxxon} & 10034.8& 6837.9& 10370.4& 10653.4& \textbf{12278.4} & 12468.9& \textbf{14341.8} \\ 

    \bottomrule
    \end{tabular}
    }
    \captionsetup{width=0.85\linewidth}
    \caption{The mean evaluation score across the last 3 evaluations during training for each algorithm across five seeds. The highest scores for an environment are bolded, with separate comparisons for the IQN variants.}
    \label{table_results}
    \end{table*}
\label{appendix:results:learning_curves}
\begin{figure}[p]
        \centering
    	\includegraphics[width=0.21\linewidth]{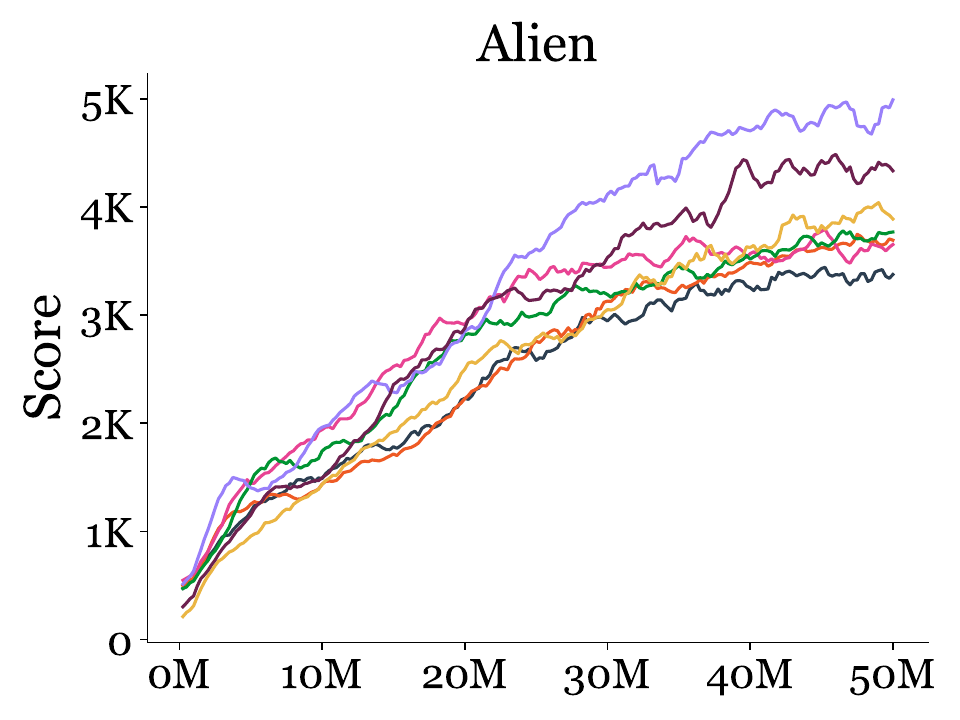} 
	\includegraphics[width=0.21\linewidth]{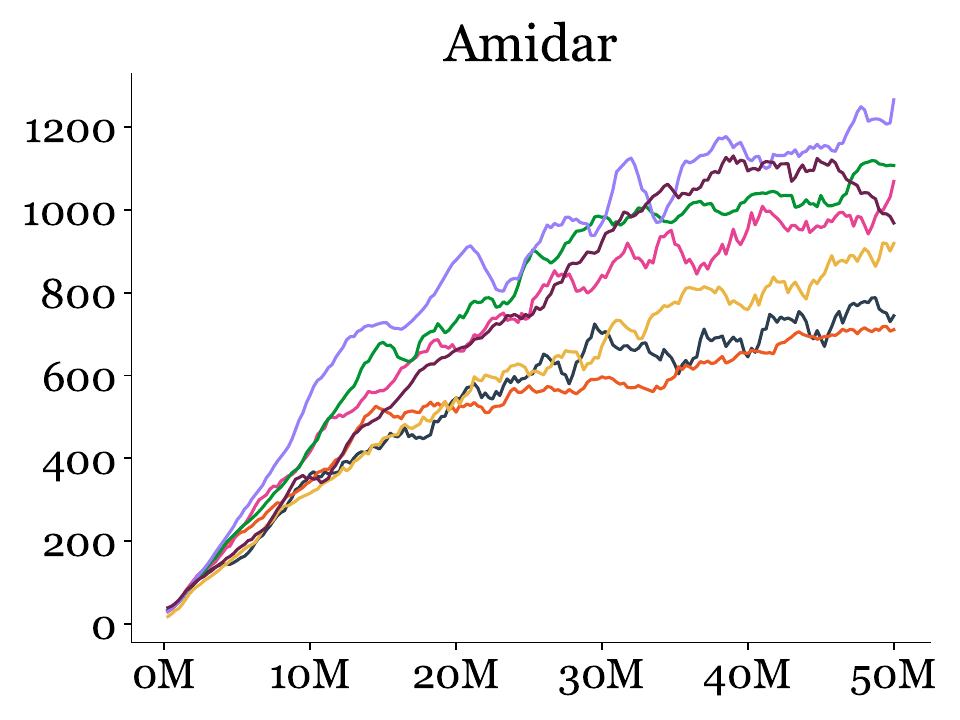} 
	\includegraphics[width=0.21\linewidth]{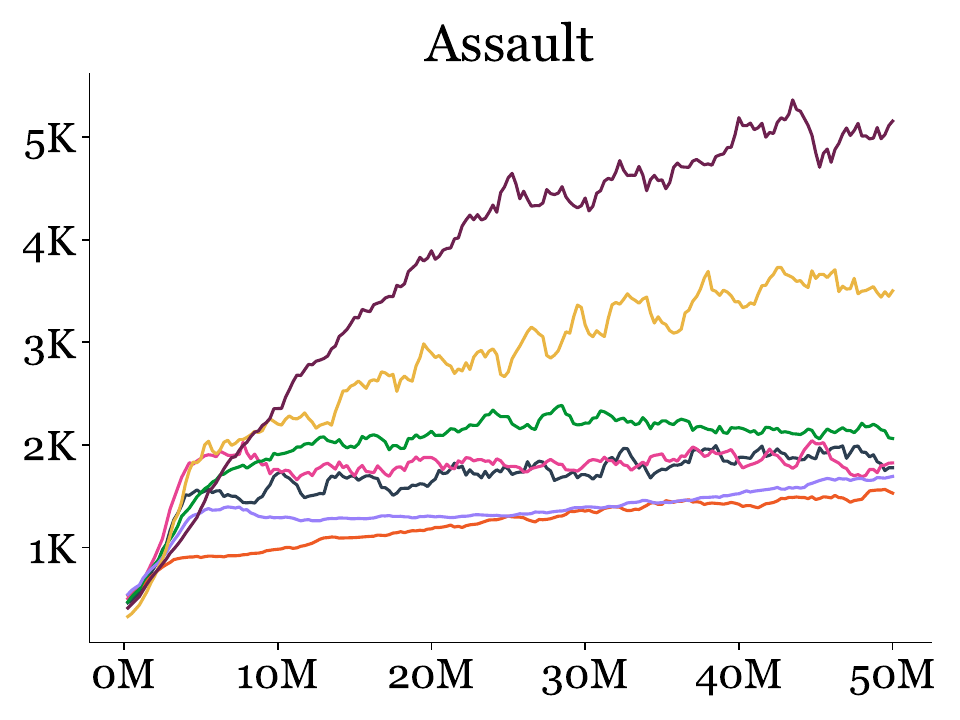} 
	\includegraphics[width=0.21\linewidth]{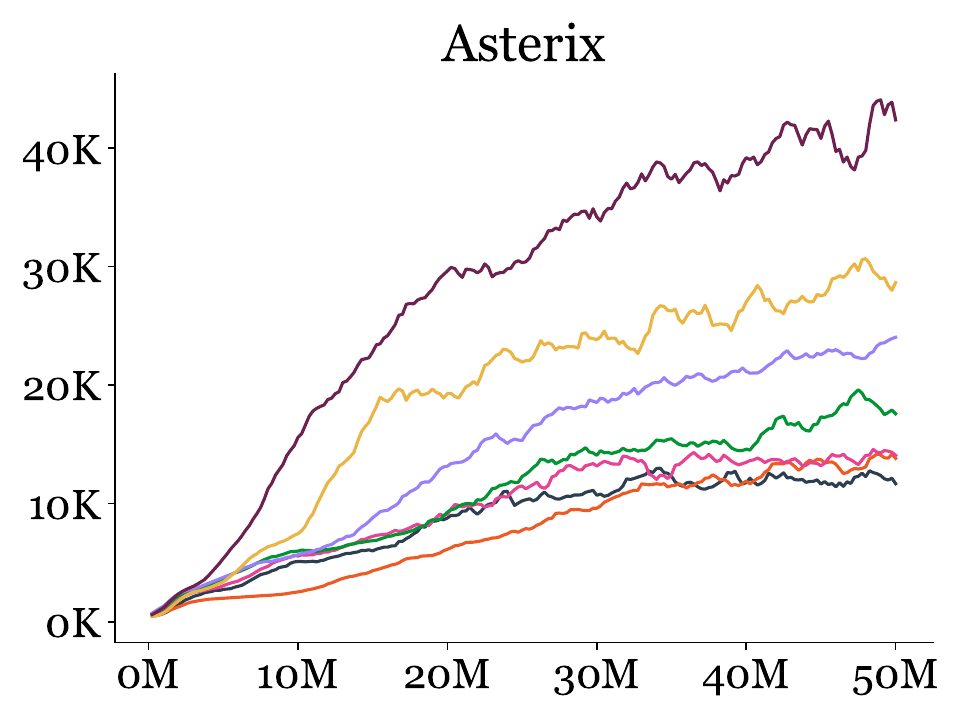} 
	\includegraphics[width=0.21\linewidth]{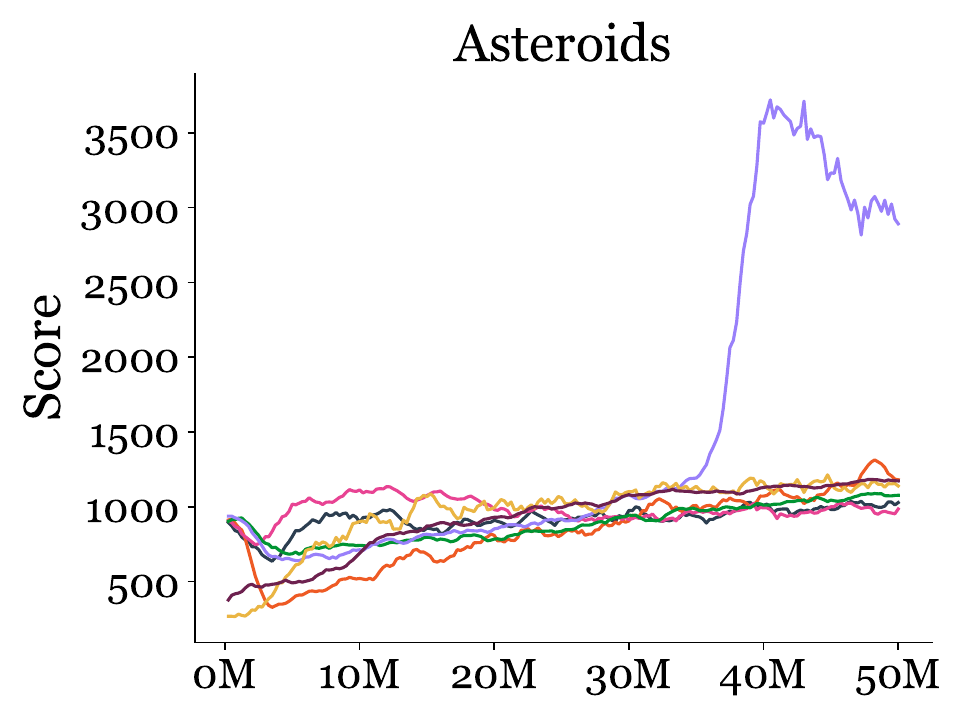} 
	\includegraphics[width=0.21\linewidth]{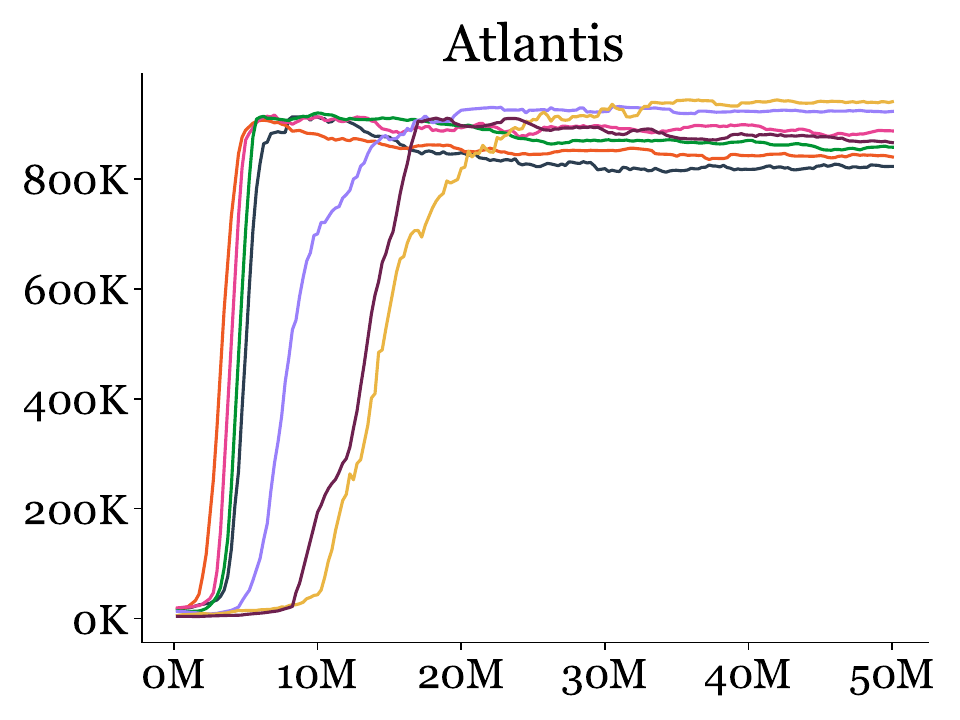} 
	\includegraphics[width=0.21\linewidth]{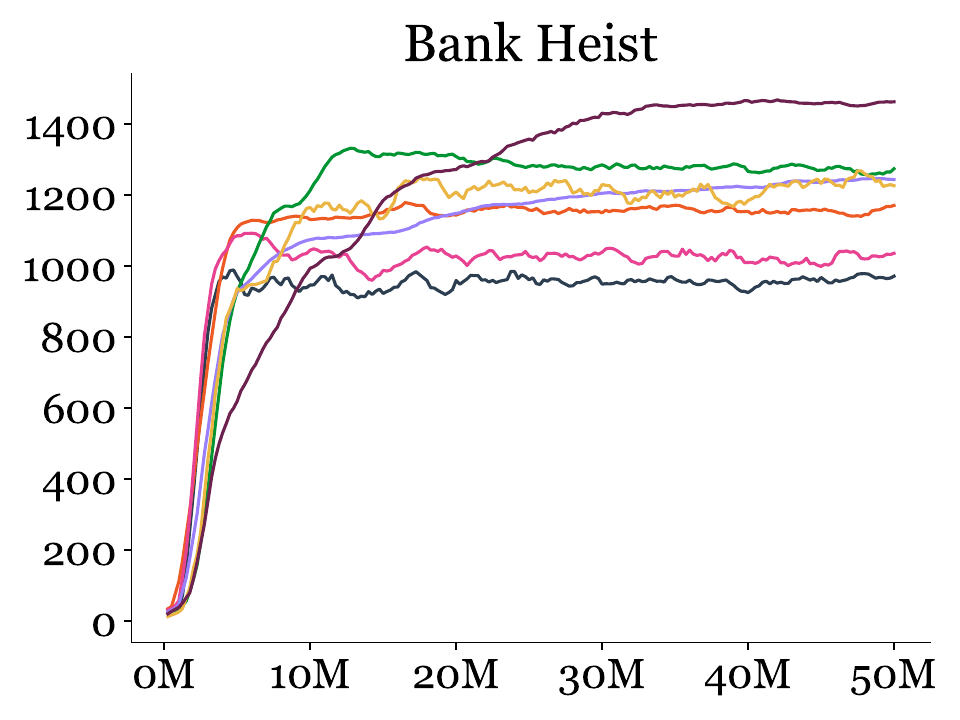} 
	\includegraphics[width=0.21\linewidth]{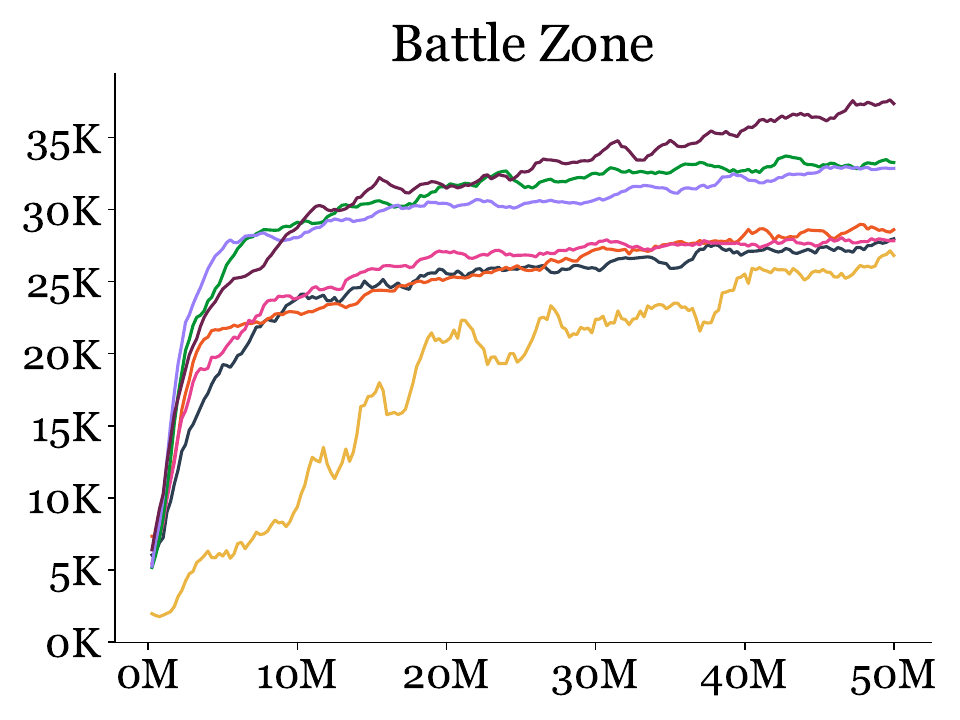} 
	\includegraphics[width=0.21\linewidth]{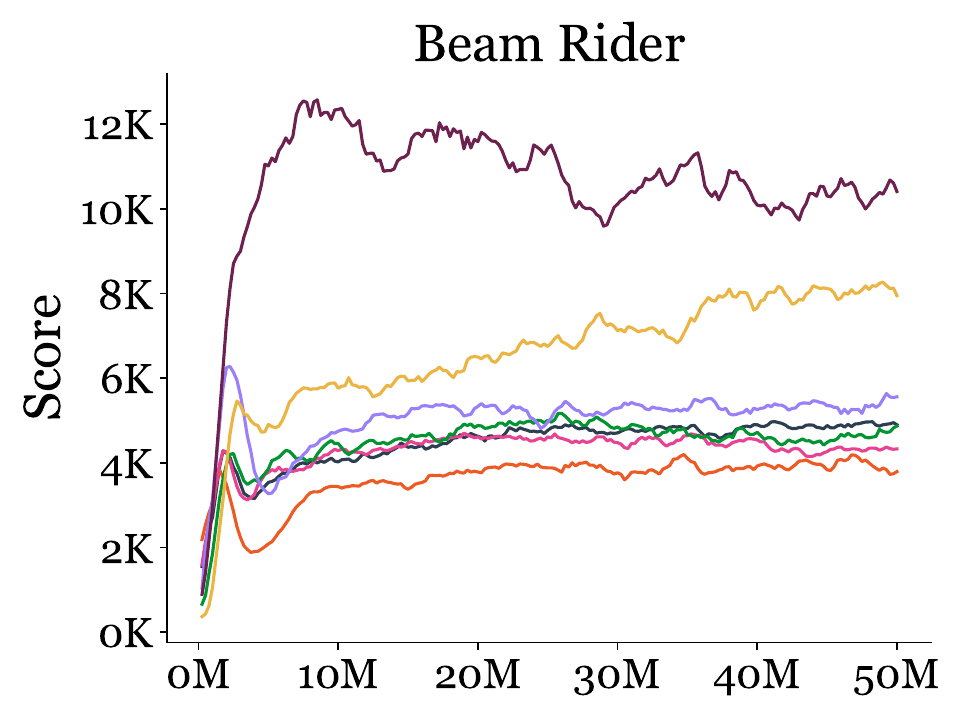} 
	\includegraphics[width=0.21\linewidth]{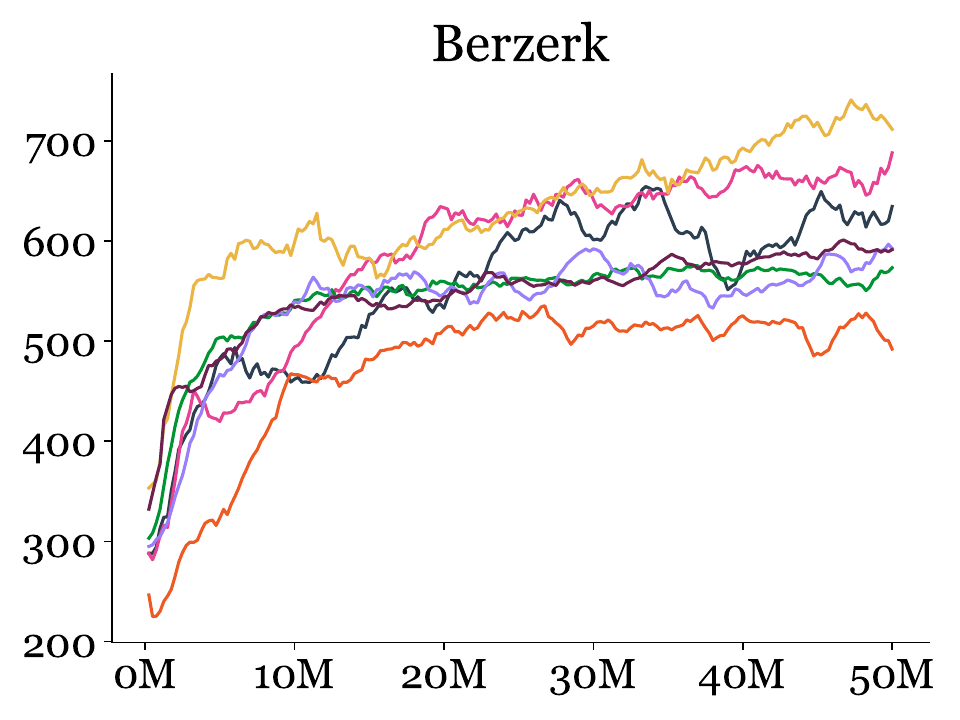} 
	\includegraphics[width=0.21\linewidth]{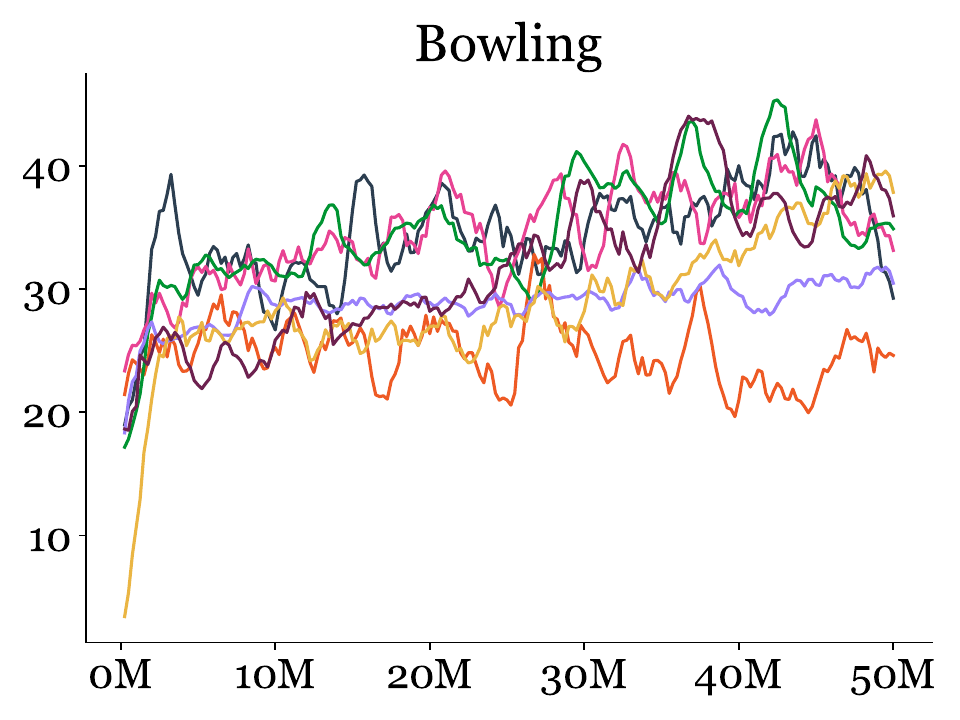} 
	\includegraphics[width=0.21\linewidth]{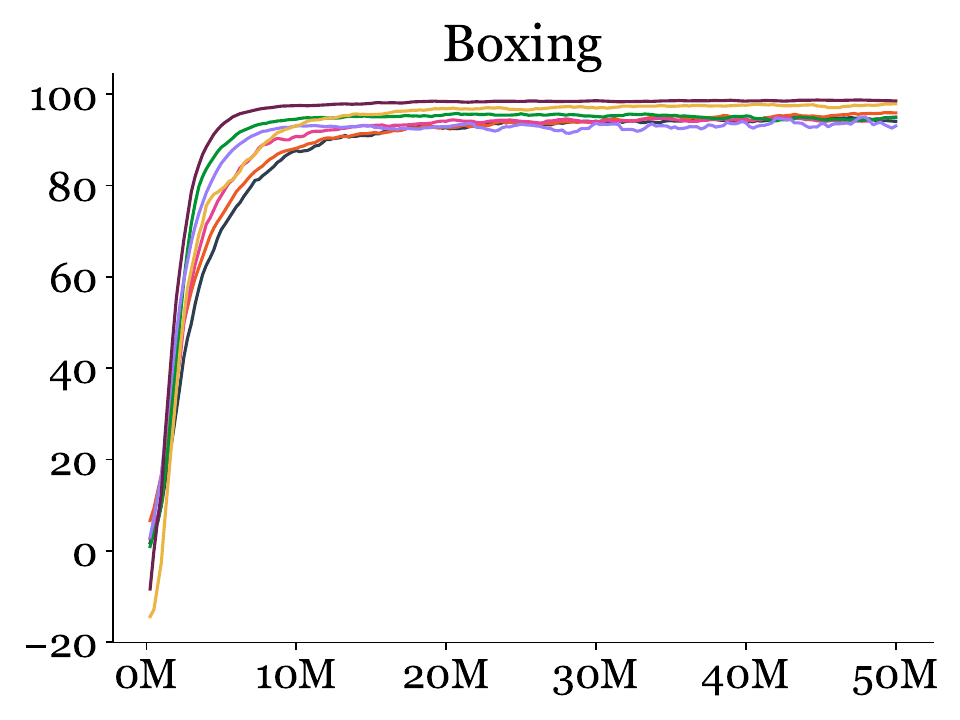} 
	\includegraphics[width=0.21\linewidth]{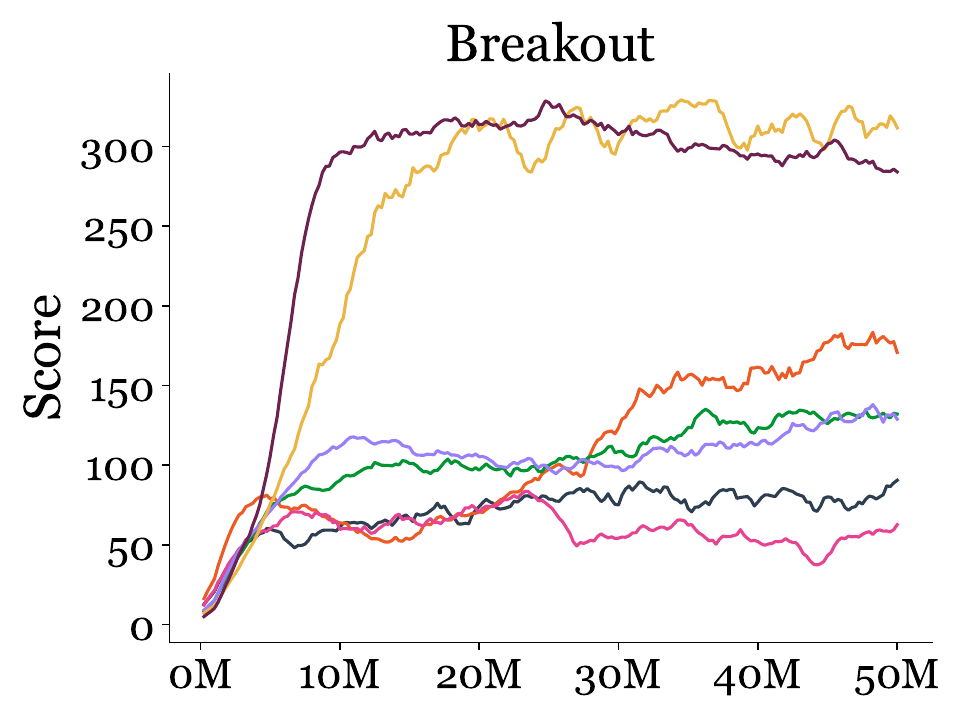} 
	\includegraphics[width=0.21\linewidth]{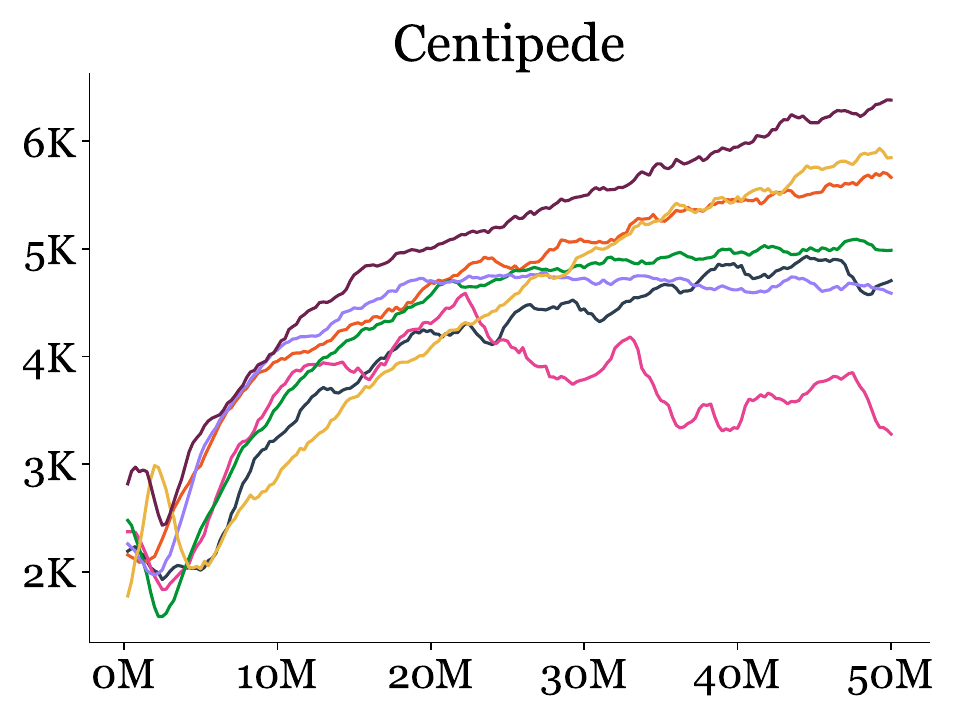} 
	\includegraphics[width=0.21\linewidth]{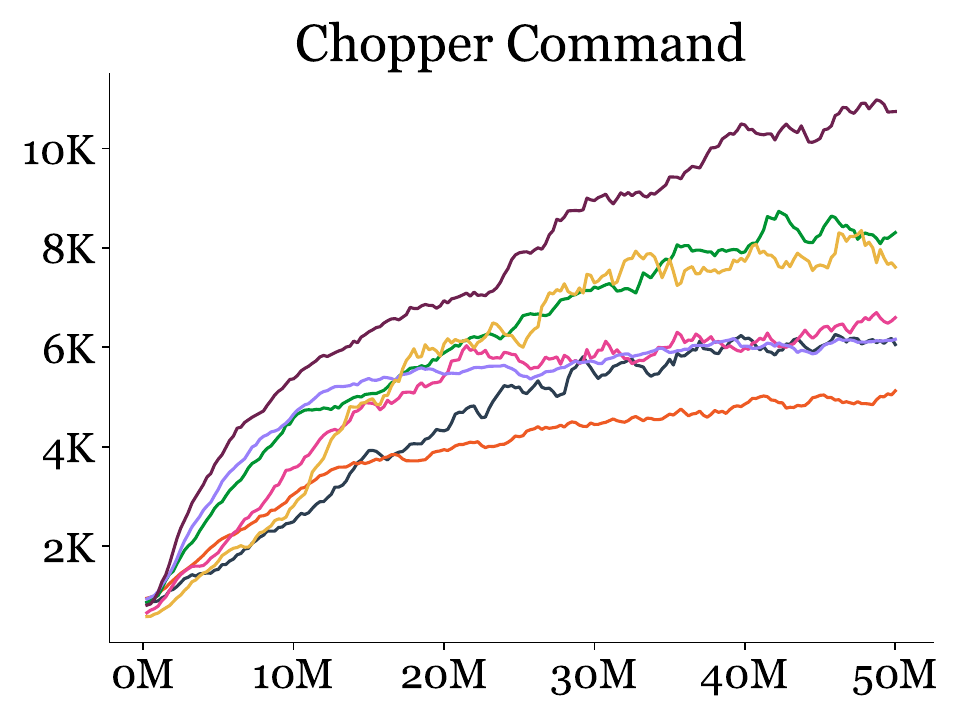} 
	\includegraphics[width=0.21\linewidth]{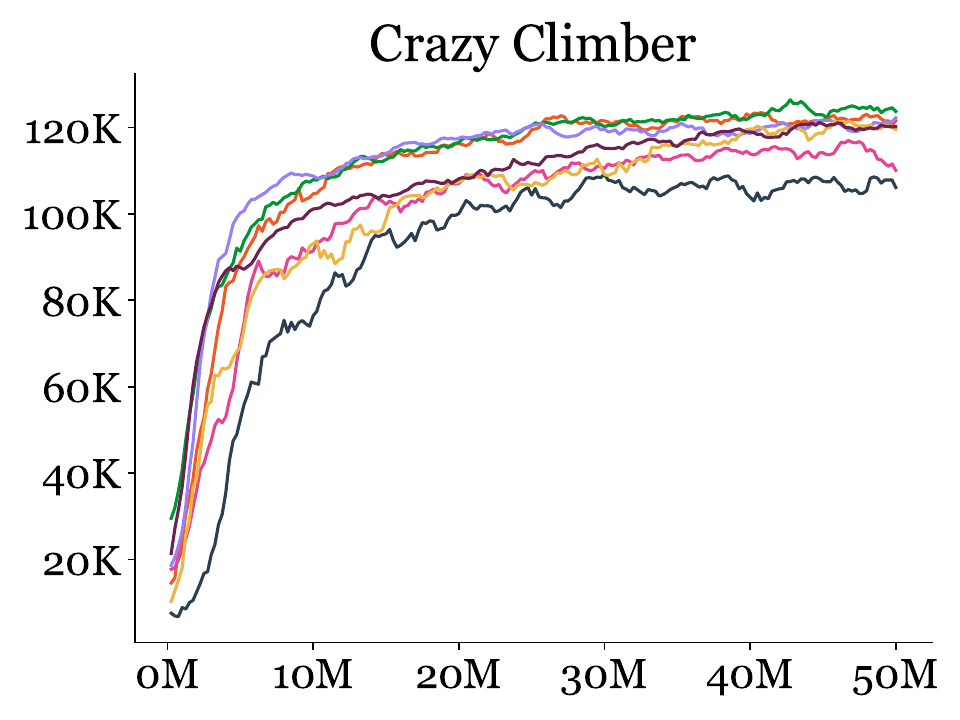} 
	\includegraphics[width=0.21\linewidth]{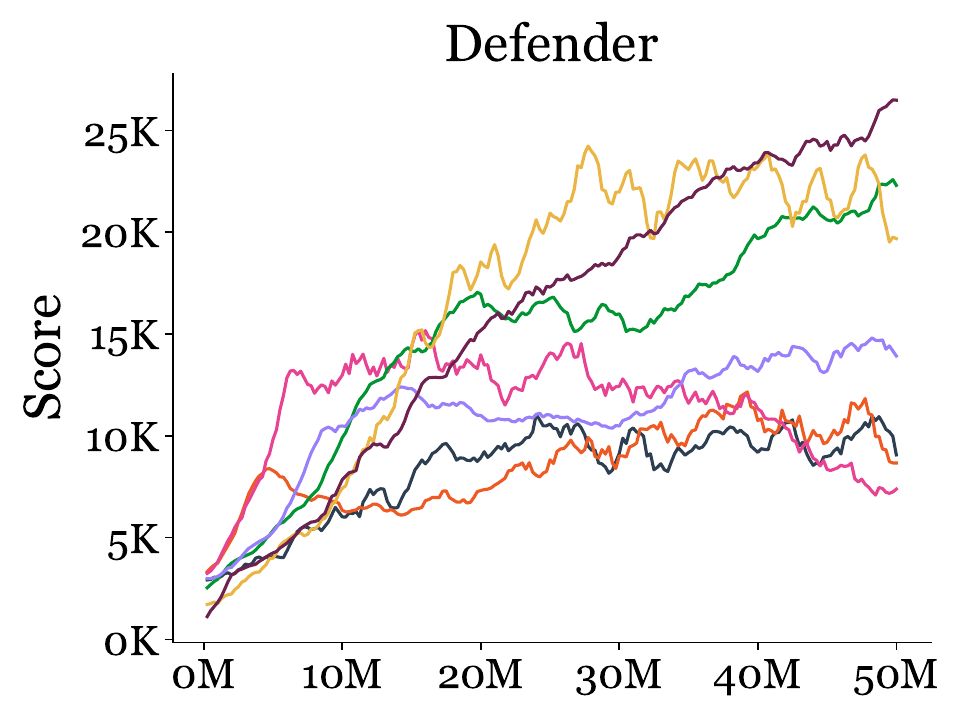} 
	\includegraphics[width=0.21\linewidth]{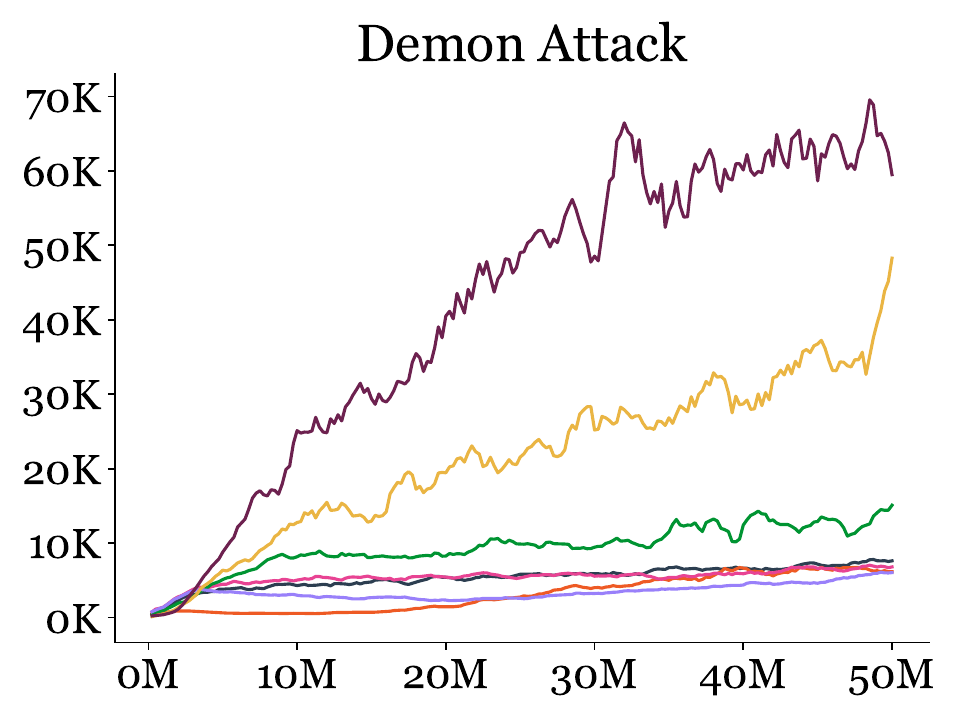} 
	\includegraphics[width=0.21\linewidth]{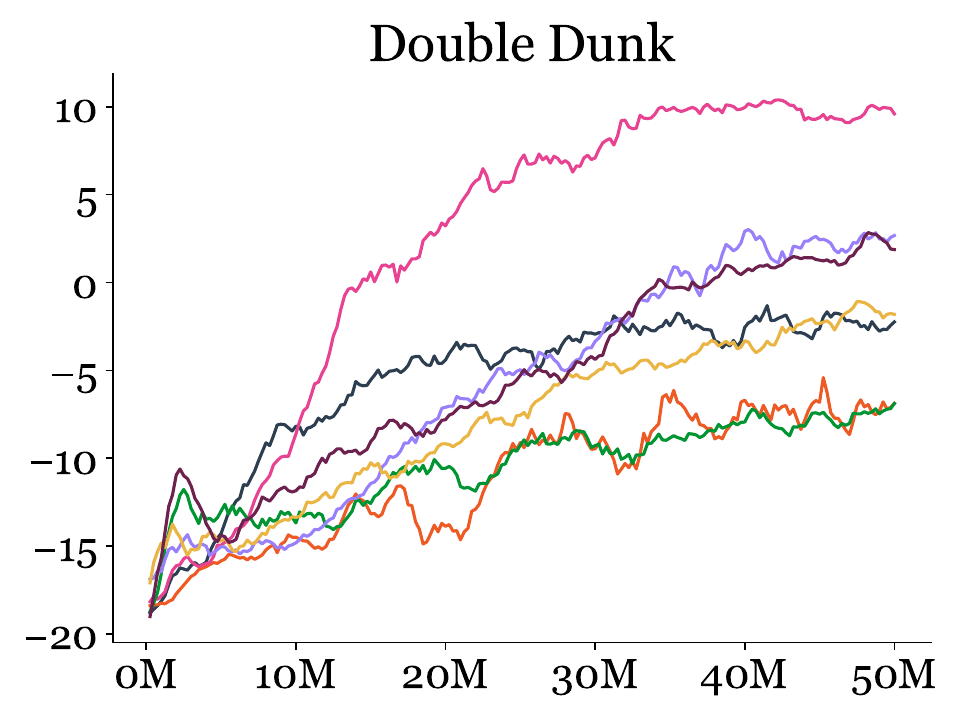} 
	\includegraphics[width=0.21\linewidth]{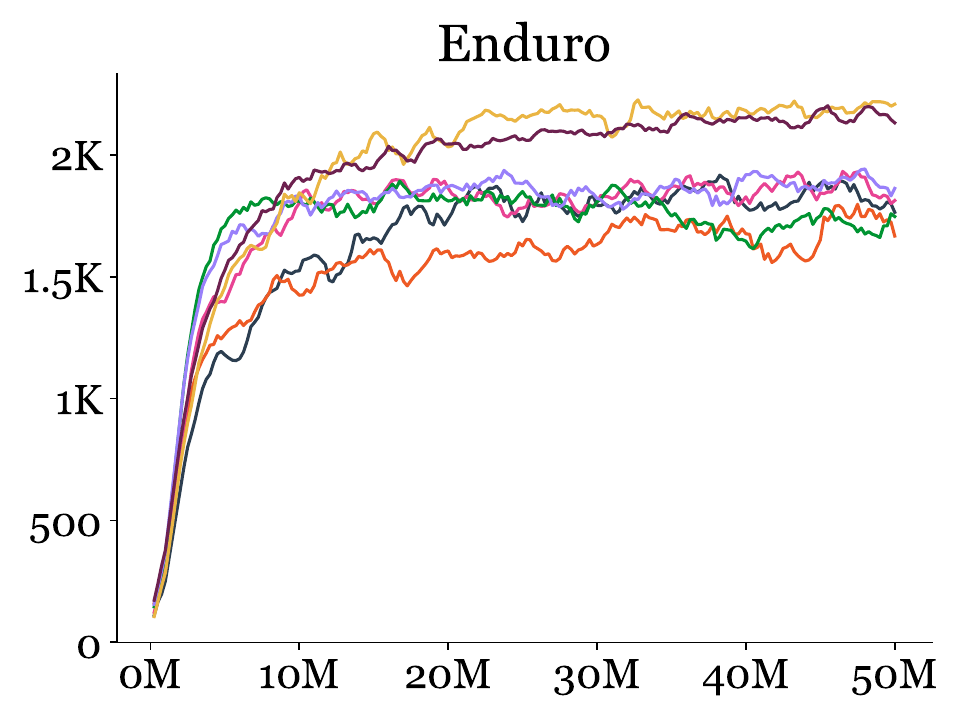} 
	\includegraphics[width=0.21\linewidth]{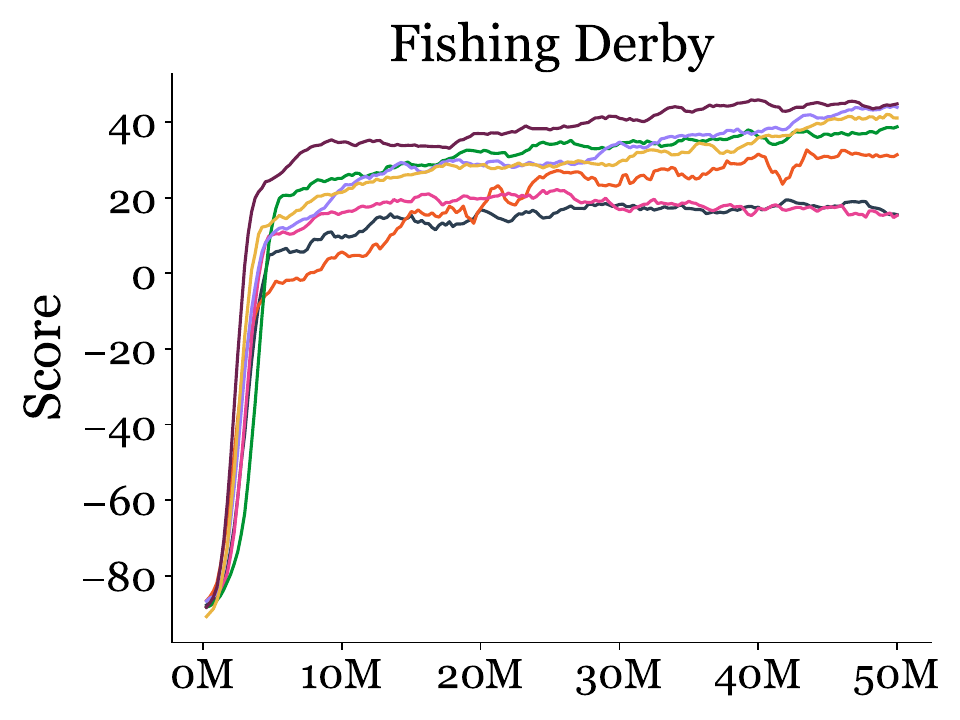} 
	\includegraphics[width=0.21\linewidth]{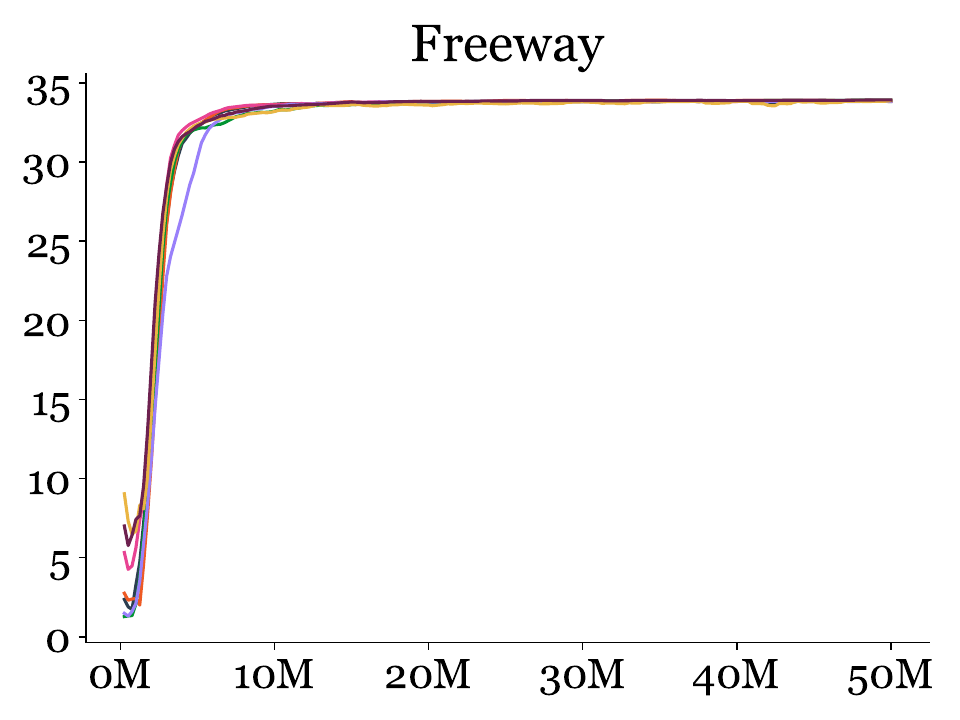} 
	\includegraphics[width=0.21\linewidth]{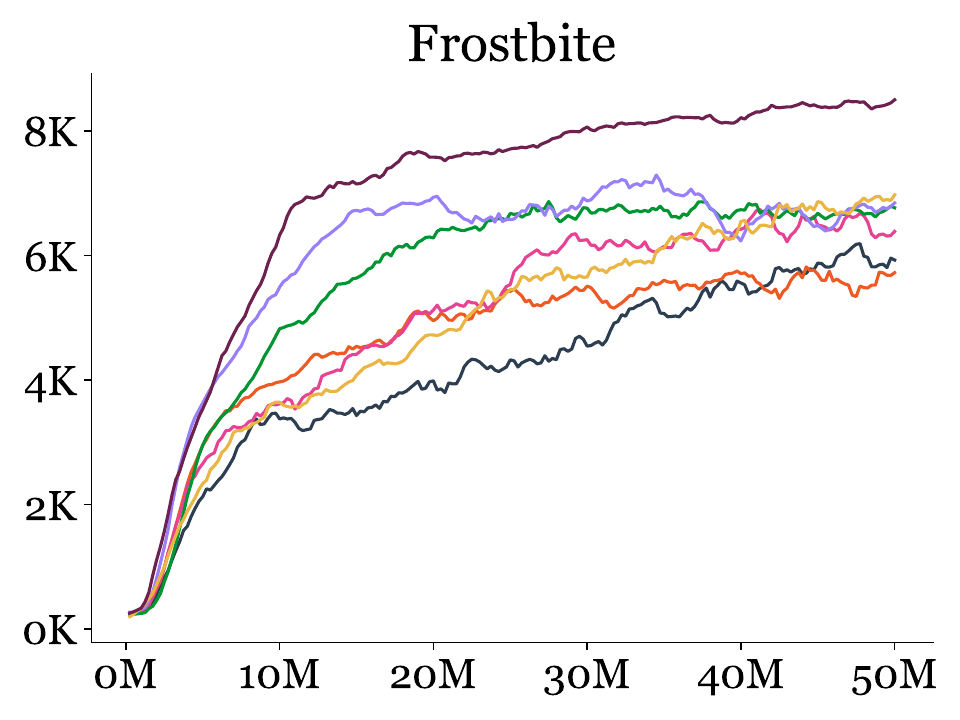} 
	\includegraphics[width=0.21\linewidth]{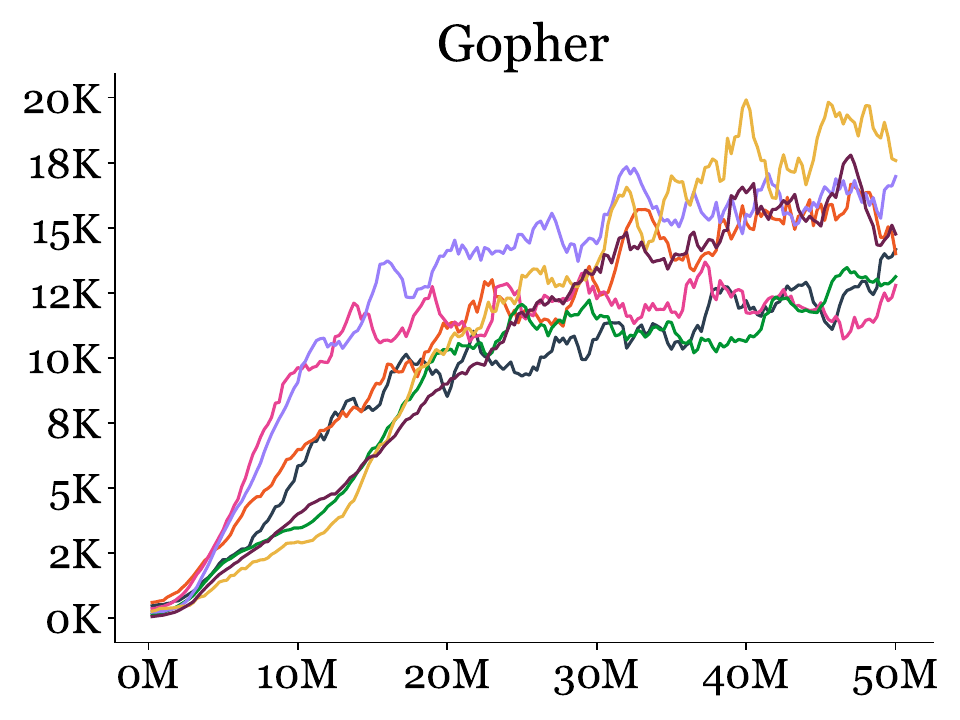} 
	\includegraphics[width=0.21\linewidth]{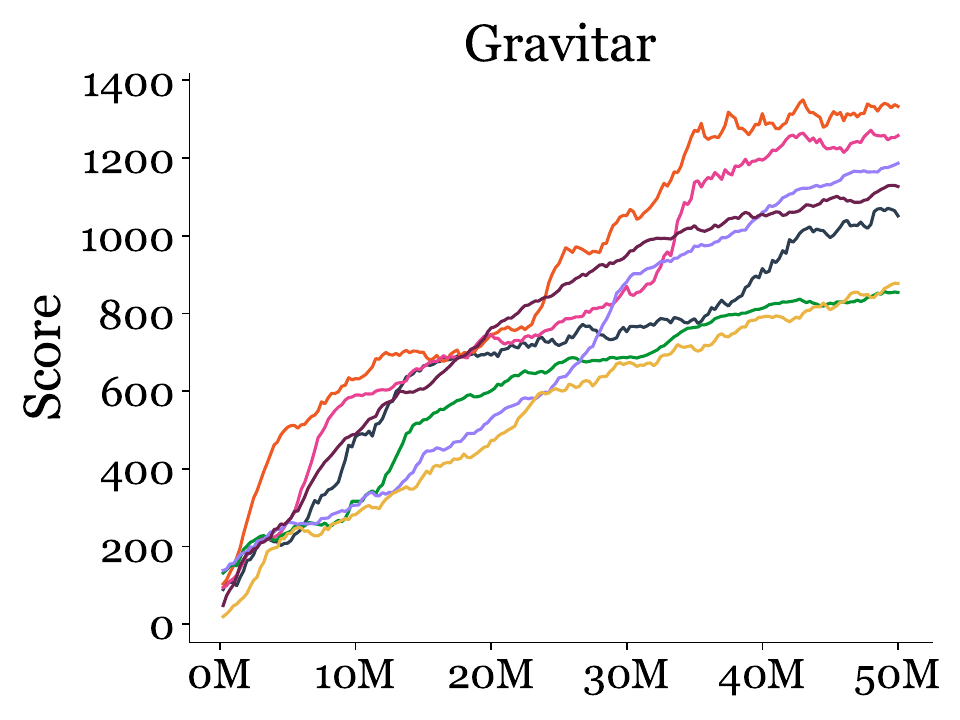} 
	\includegraphics[width=0.21\linewidth]{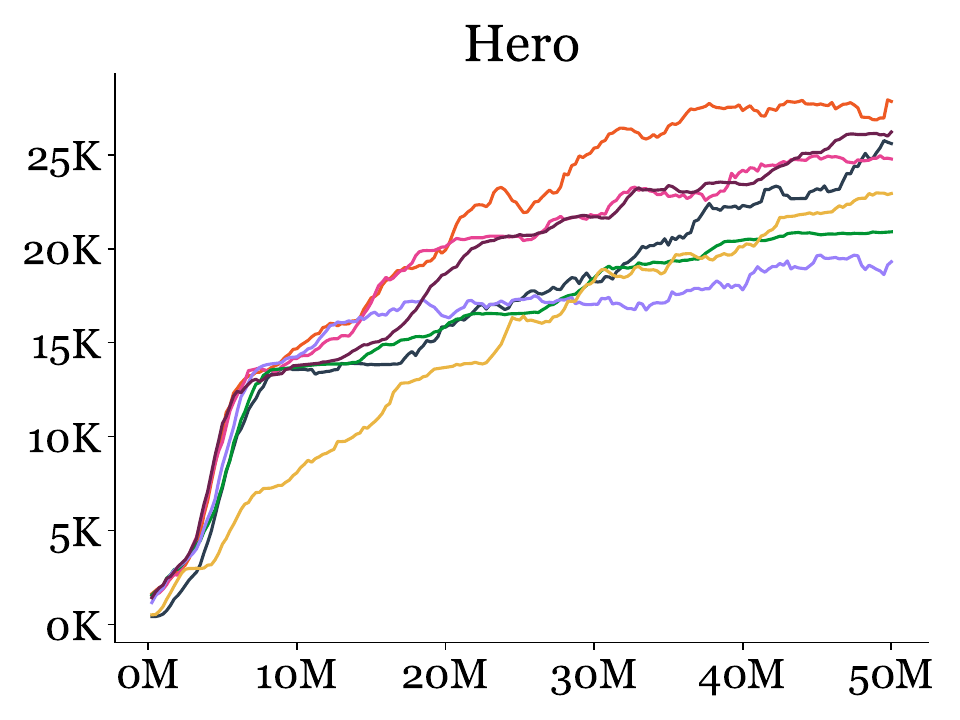} 
	\includegraphics[width=0.21\linewidth]{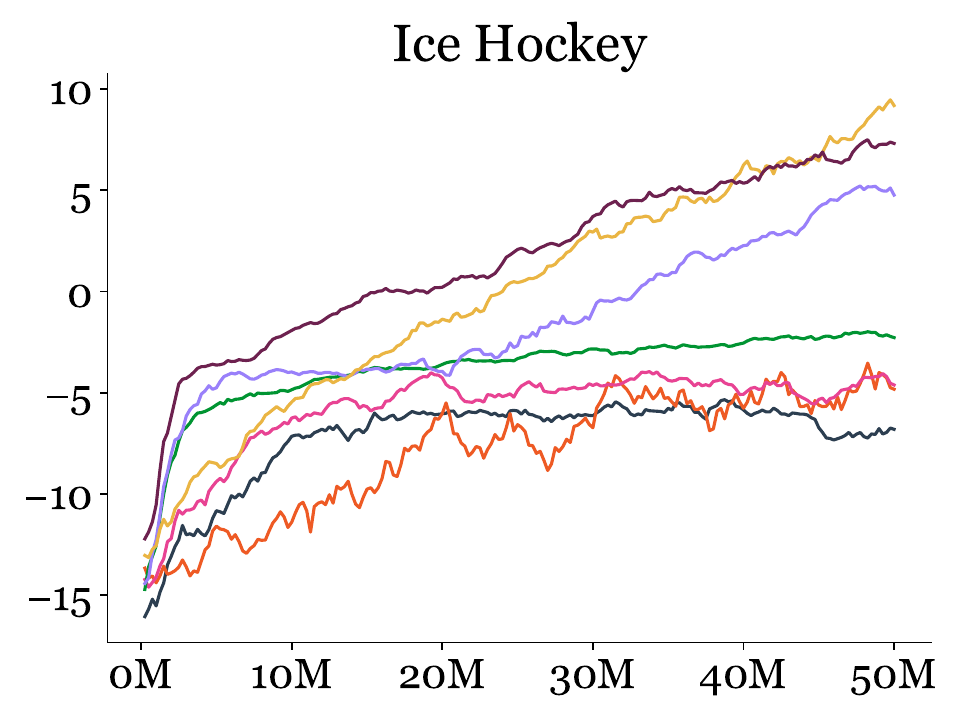} 
	\includegraphics[width=0.21\linewidth]{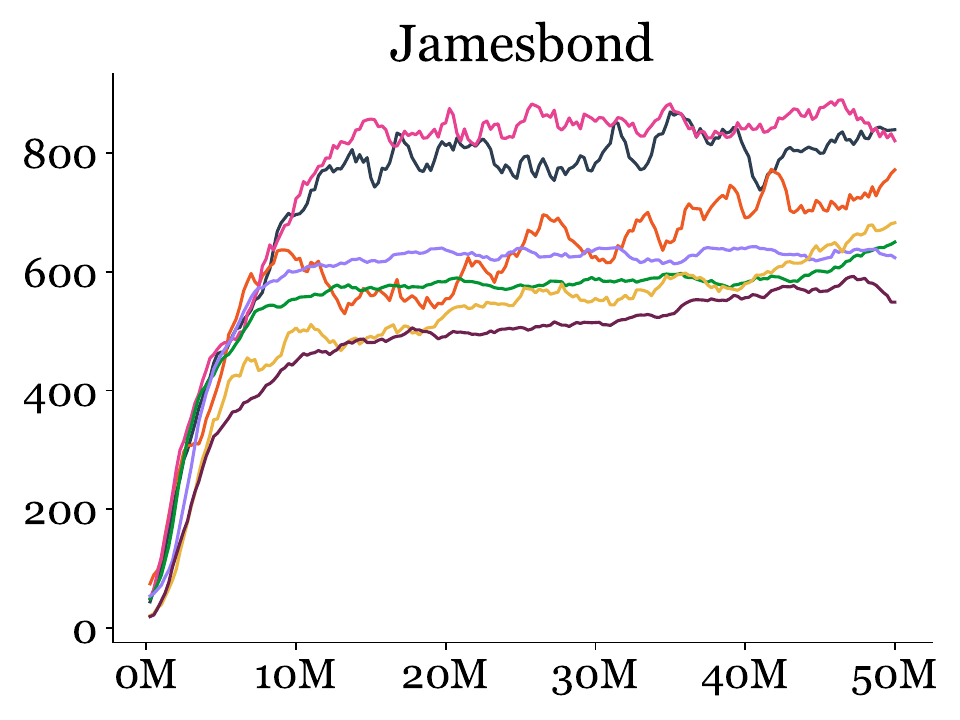} 
	\label{Atari57:score:page_1}
\end{figure}\begin{figure}[p]
        \centering
    	\includegraphics[width=0.21\linewidth]{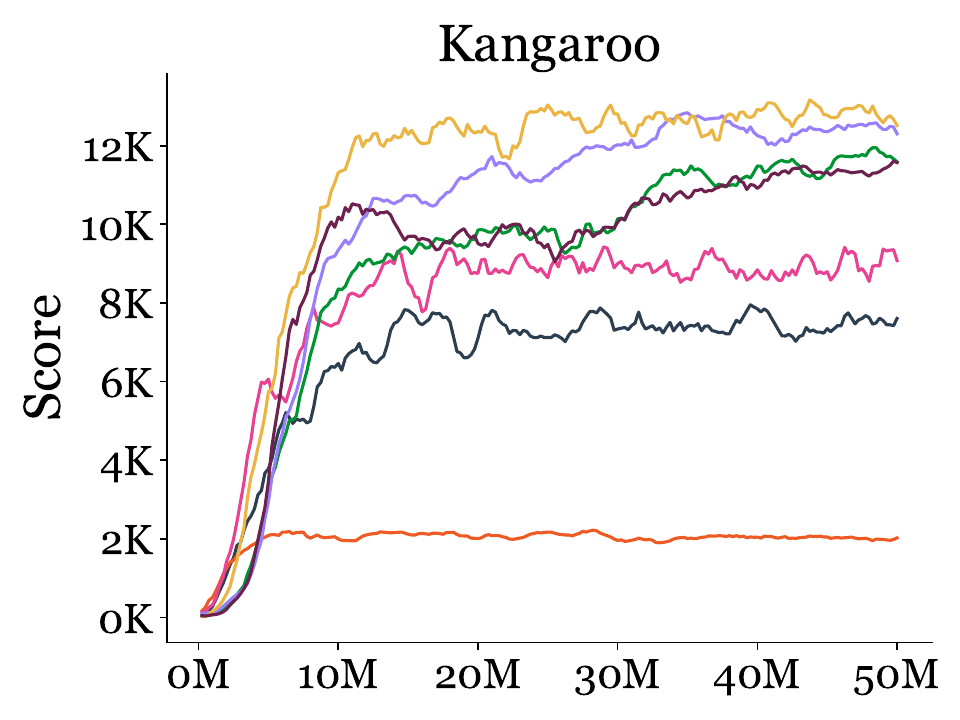} 
	\includegraphics[width=0.21\linewidth]{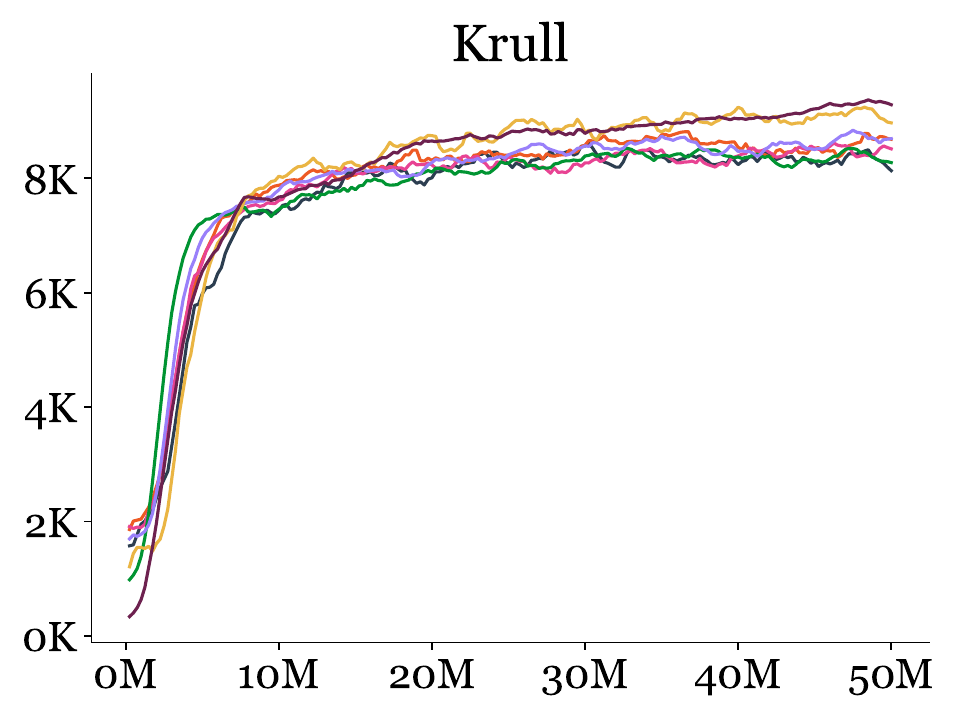} 
	\includegraphics[width=0.21\linewidth]{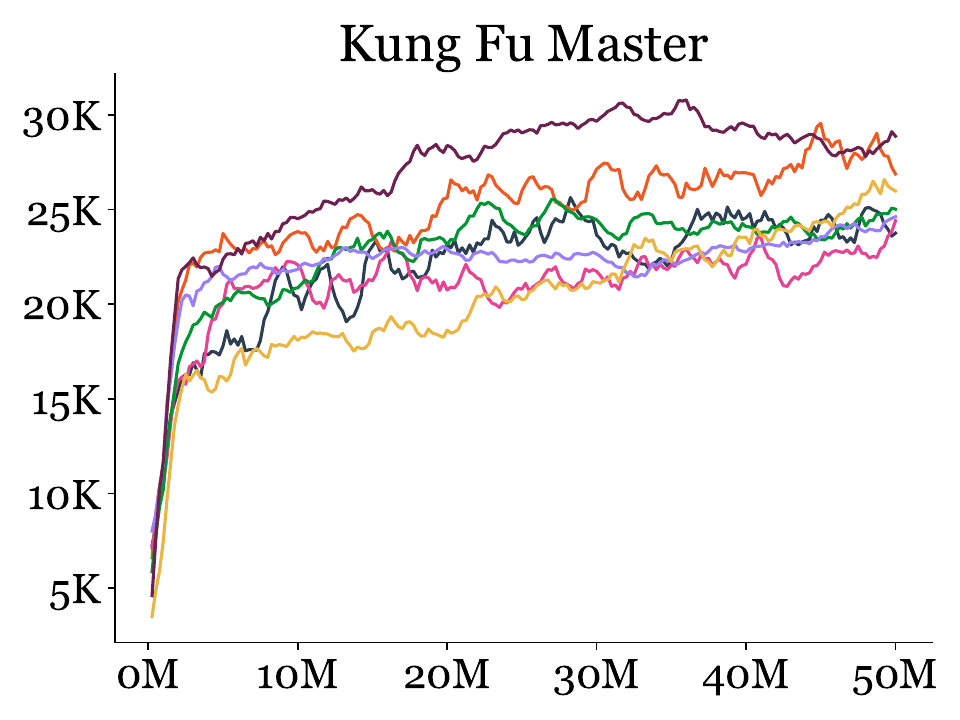} 
	\includegraphics[width=0.21\linewidth]{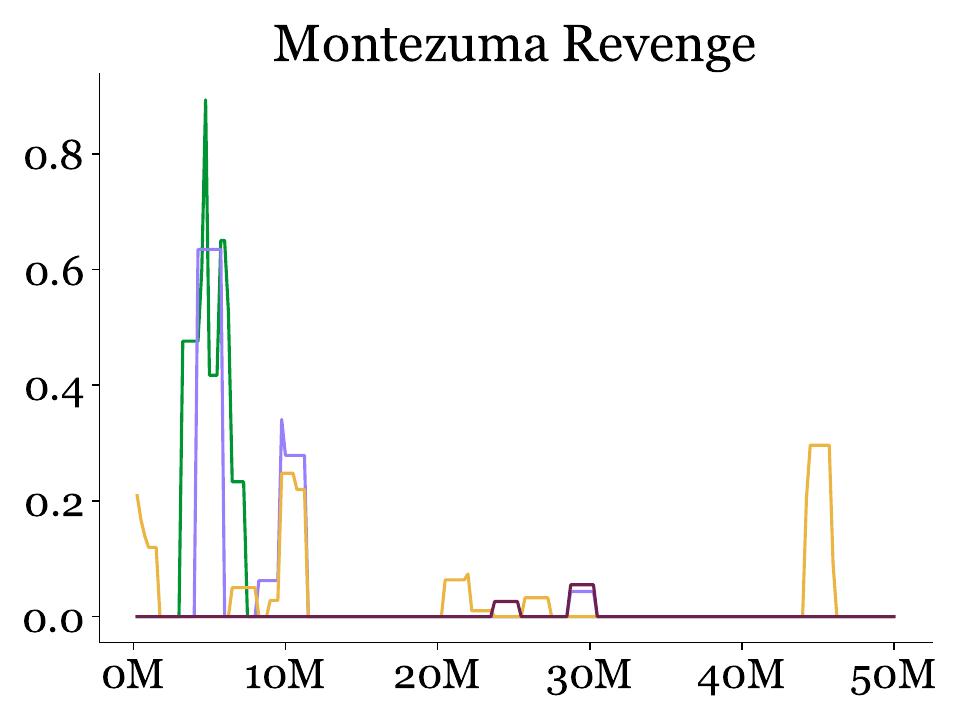} 
	\includegraphics[width=0.21\linewidth]{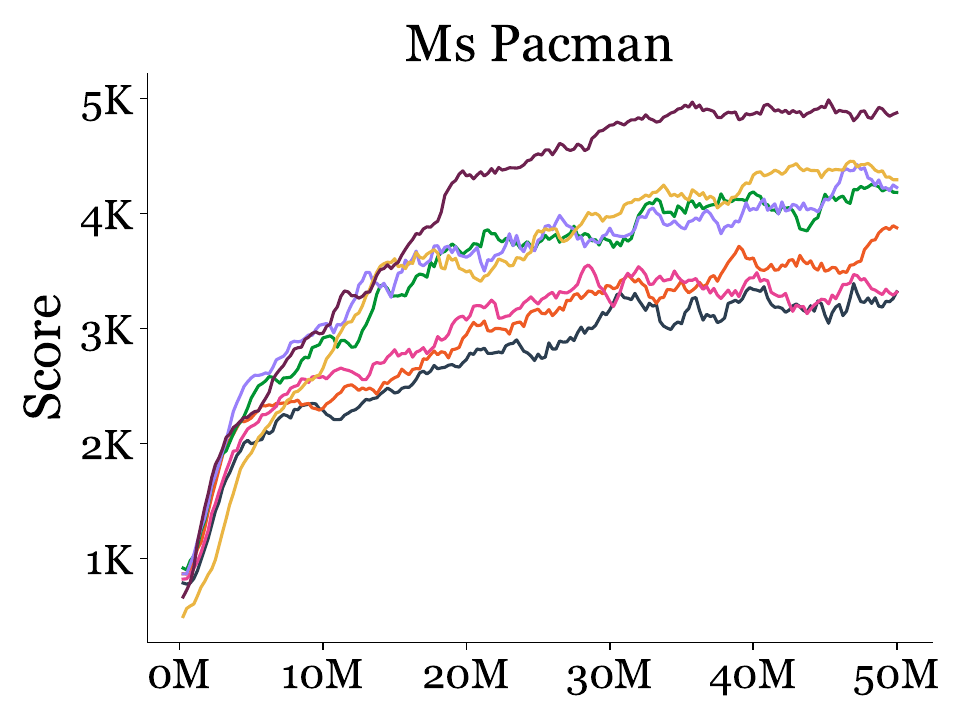} 
	\includegraphics[width=0.21\linewidth]{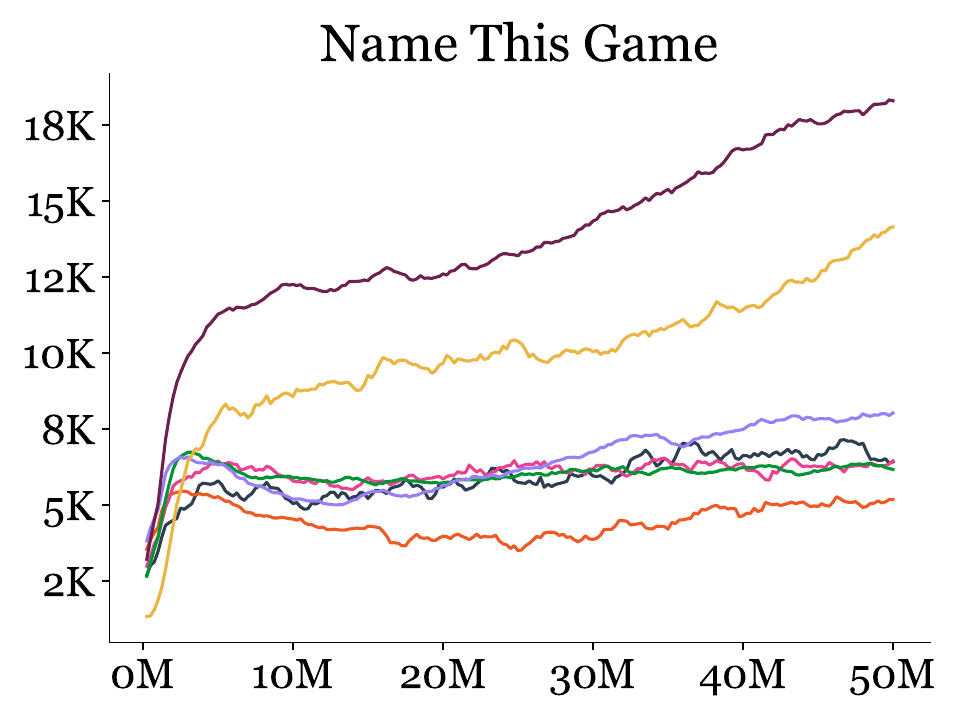} 
	\includegraphics[width=0.21\linewidth]{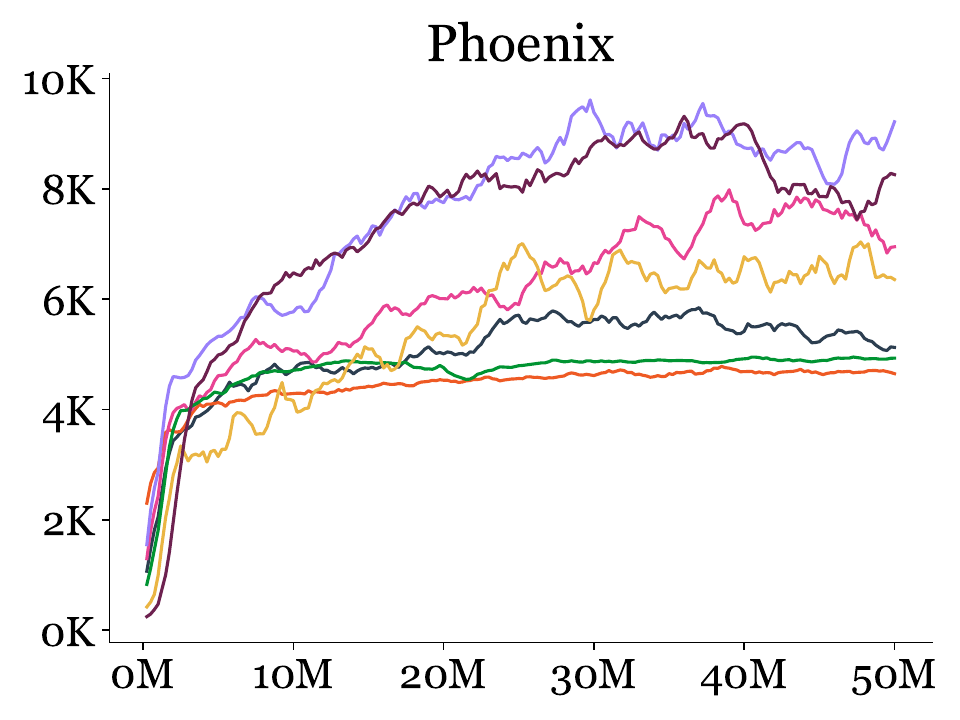} 
	\includegraphics[width=0.21\linewidth]{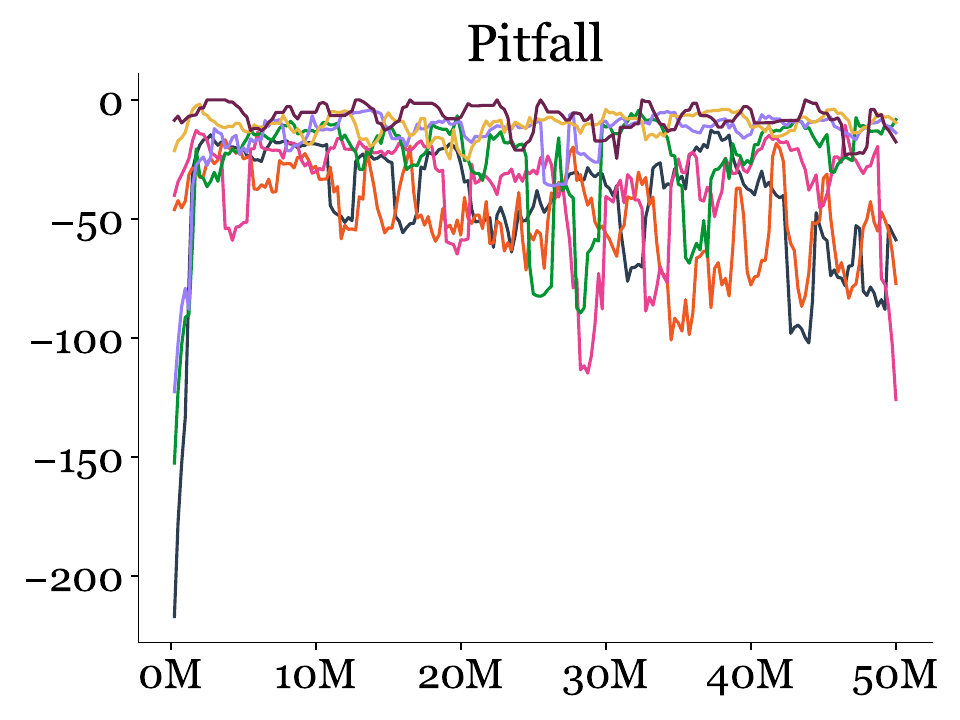} 
	\includegraphics[width=0.21\linewidth]{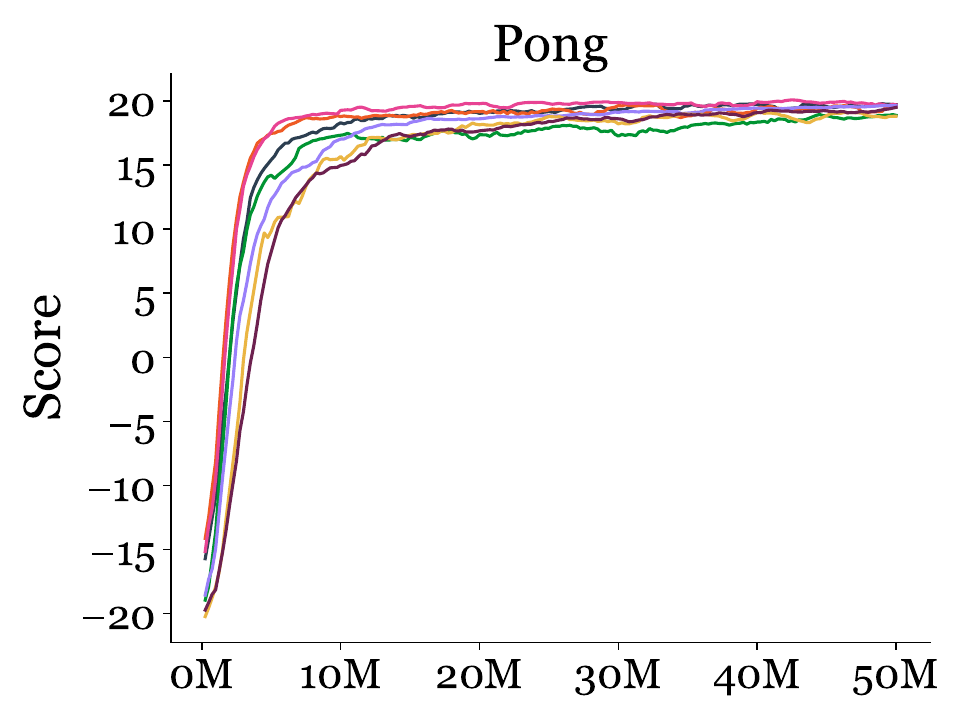} 
	\includegraphics[width=0.21\linewidth]{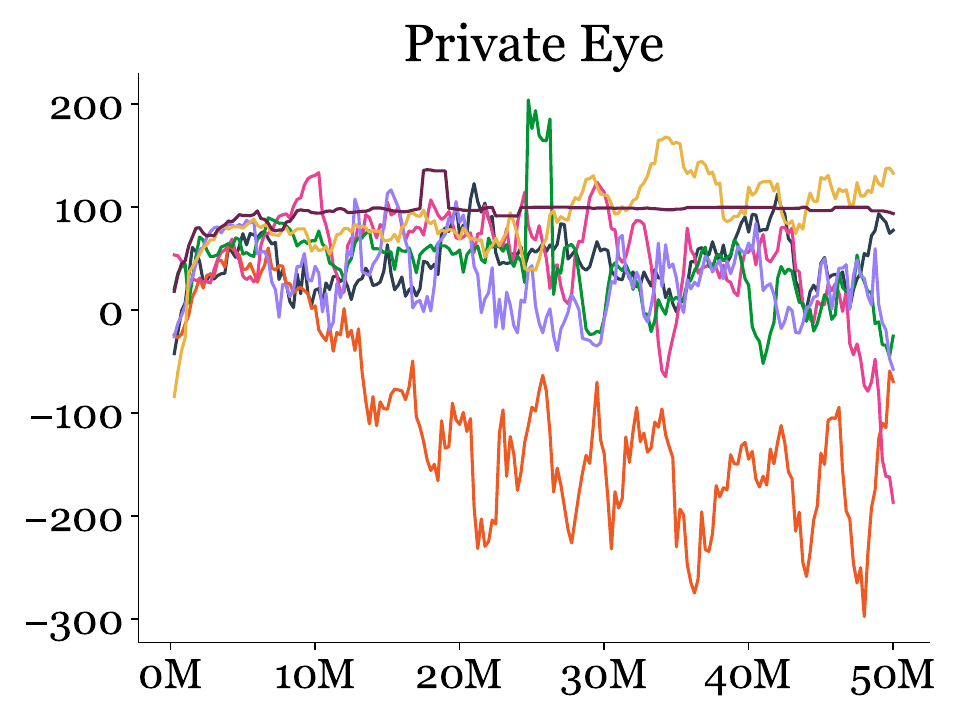} 
	\includegraphics[width=0.21\linewidth]{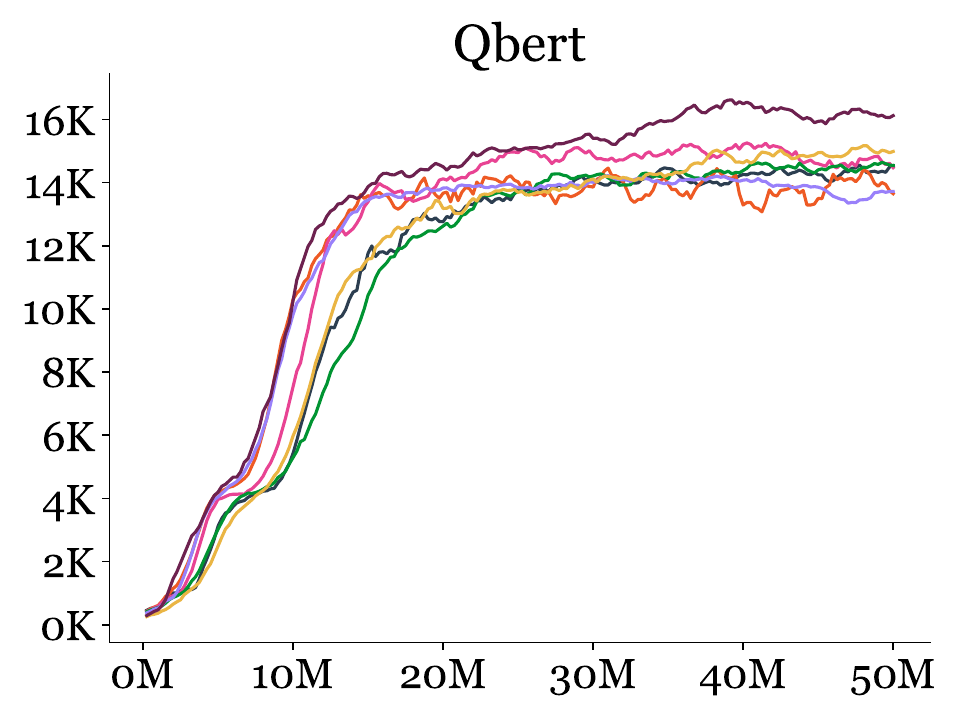} 
	\includegraphics[width=0.21\linewidth]{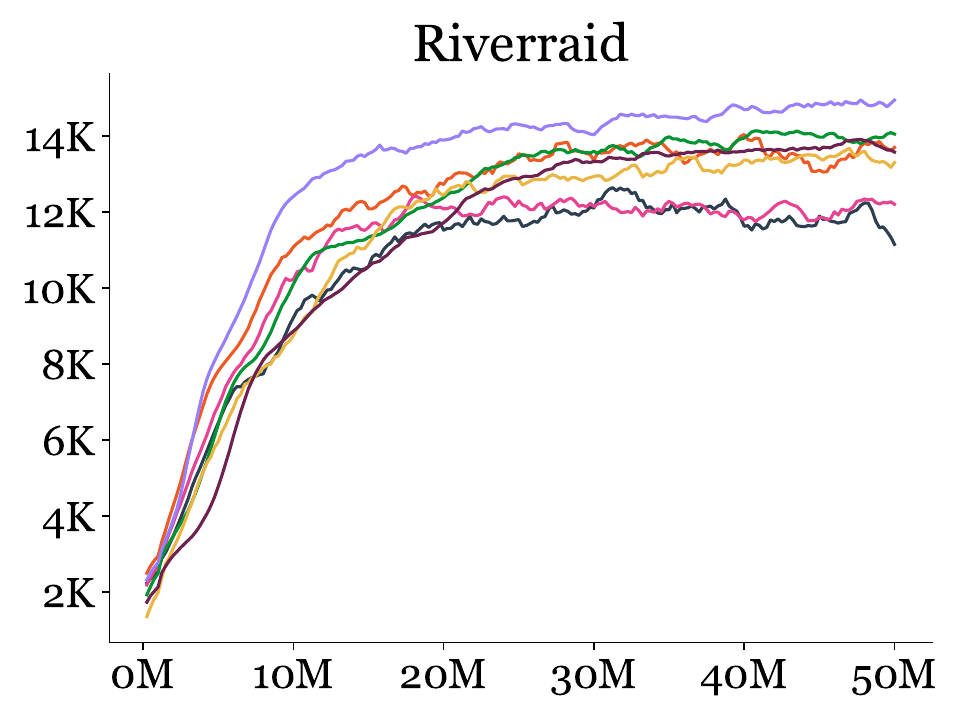} 
	\includegraphics[width=0.21\linewidth]{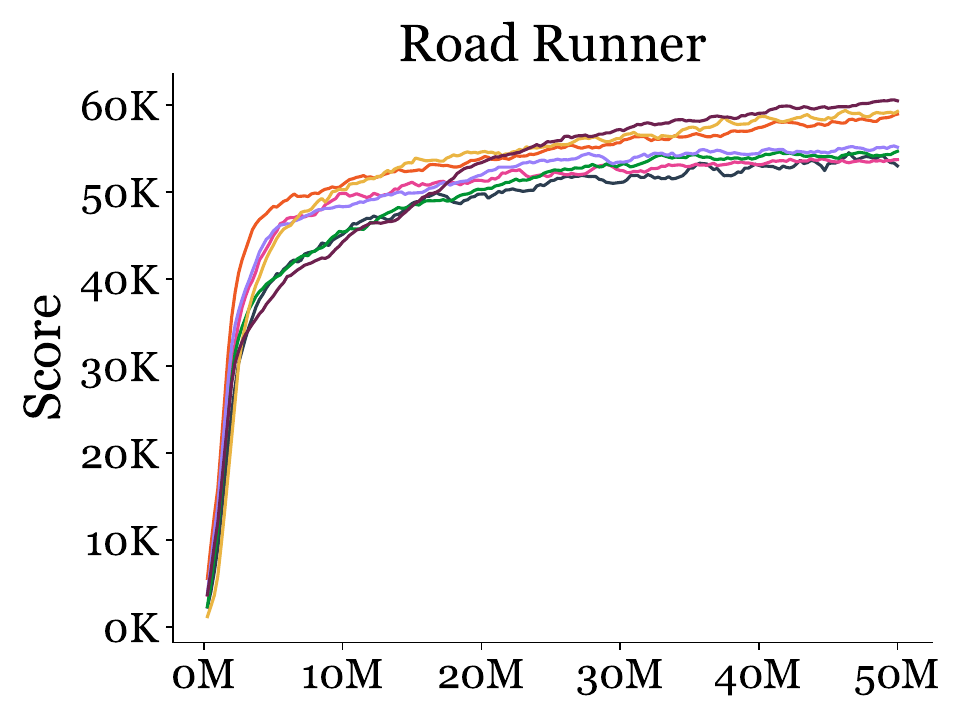} 
	\includegraphics[width=0.21\linewidth]{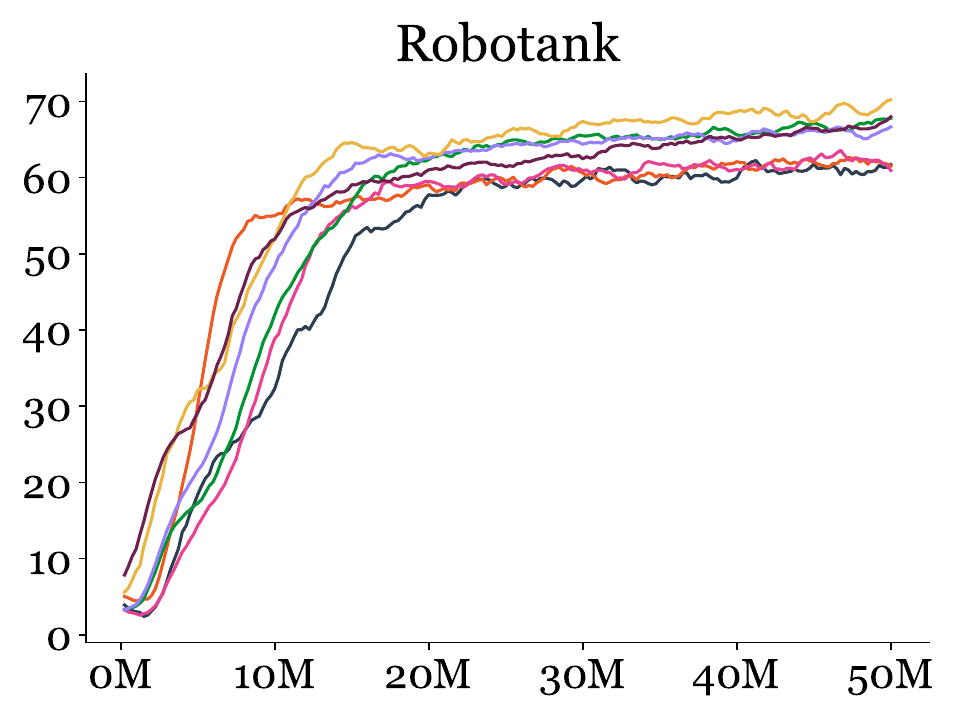} 
	\includegraphics[width=0.21\linewidth]{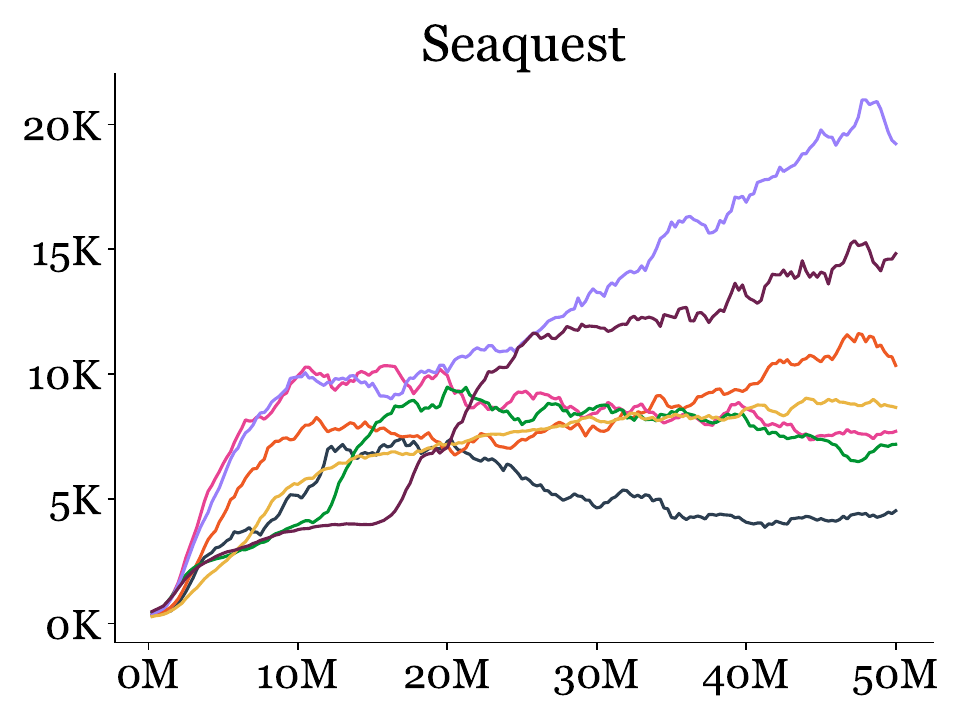} 
	\includegraphics[width=0.21\linewidth]{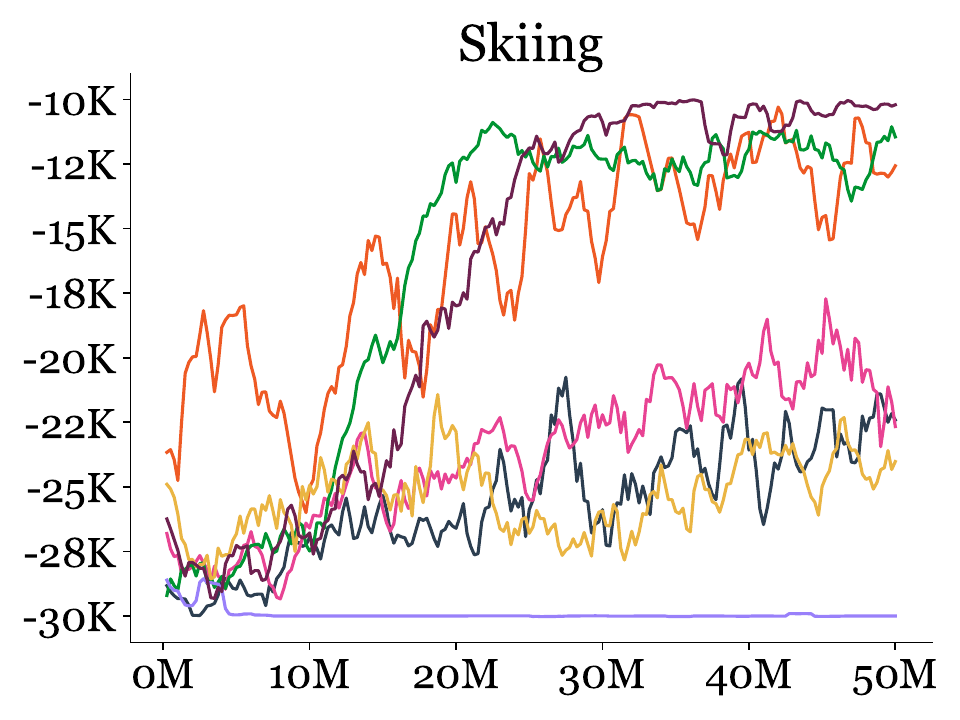} 
	\includegraphics[width=0.21\linewidth]{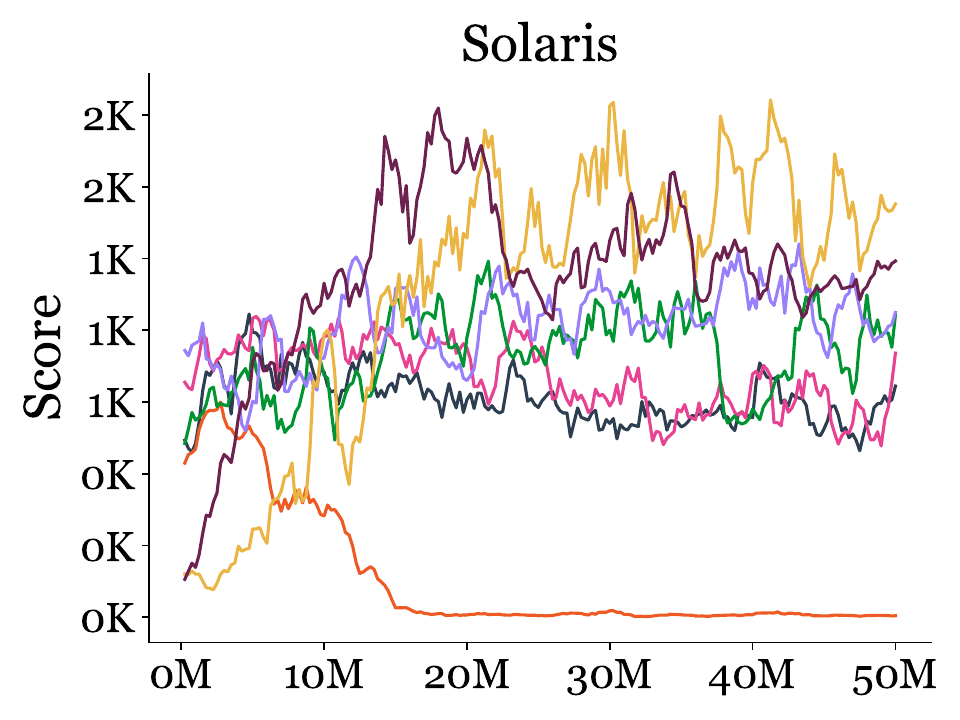} 
	\includegraphics[width=0.21\linewidth]{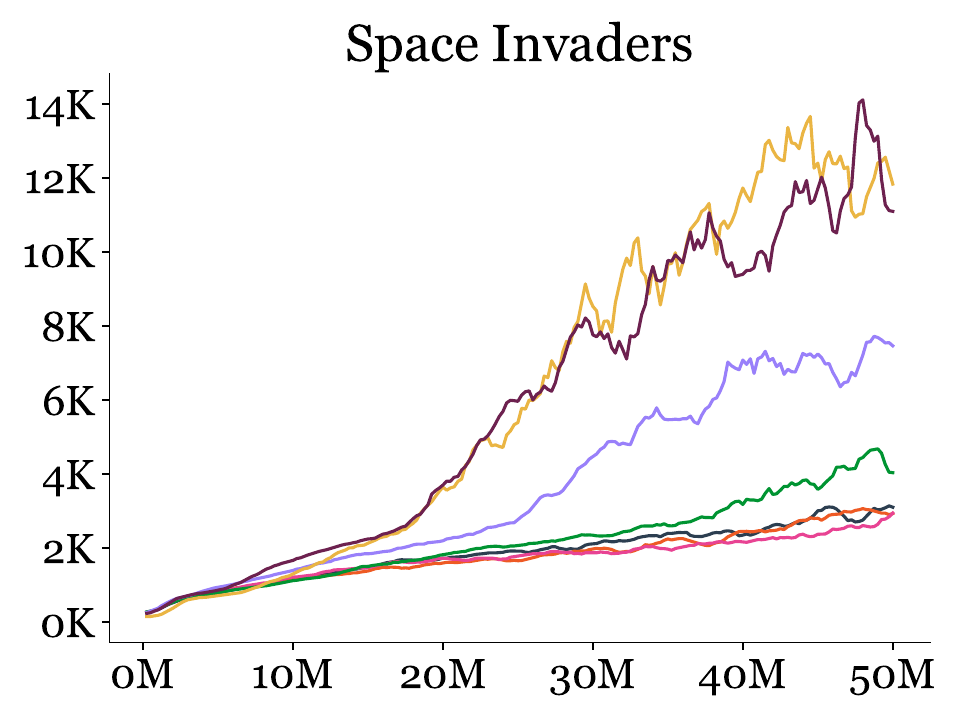} 
	\includegraphics[width=0.21\linewidth]{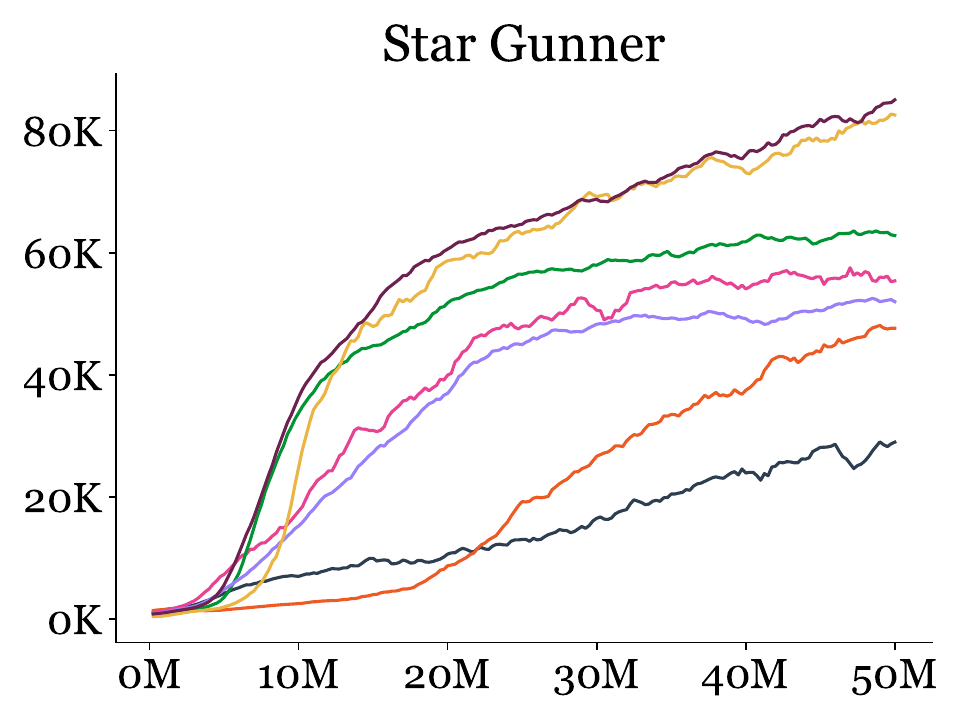} 
	\includegraphics[width=0.21\linewidth]{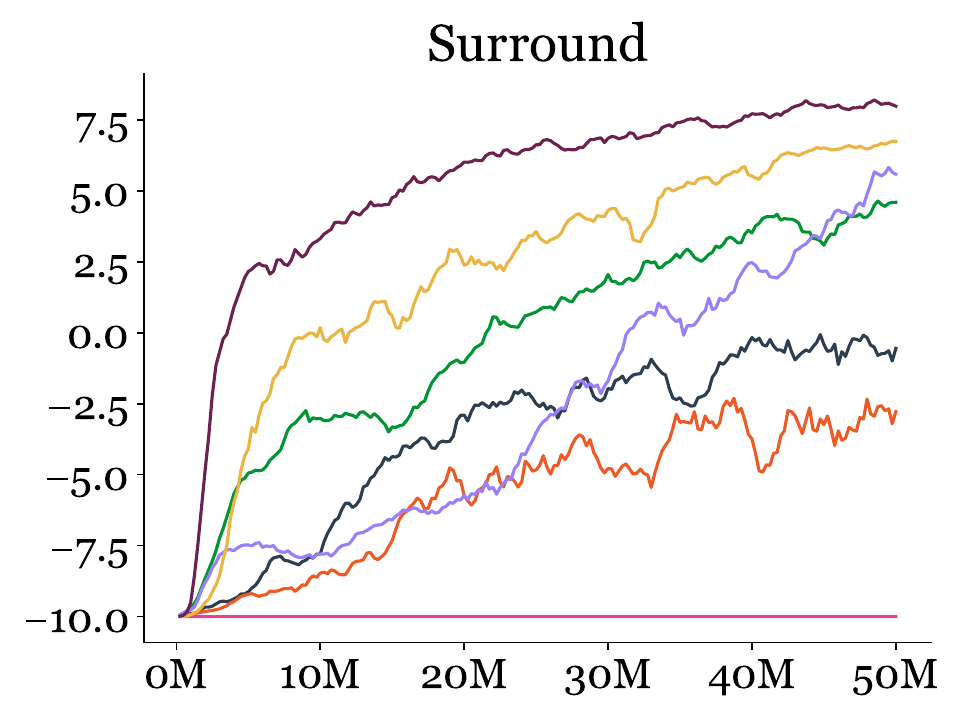} 
	\includegraphics[width=0.21\linewidth]{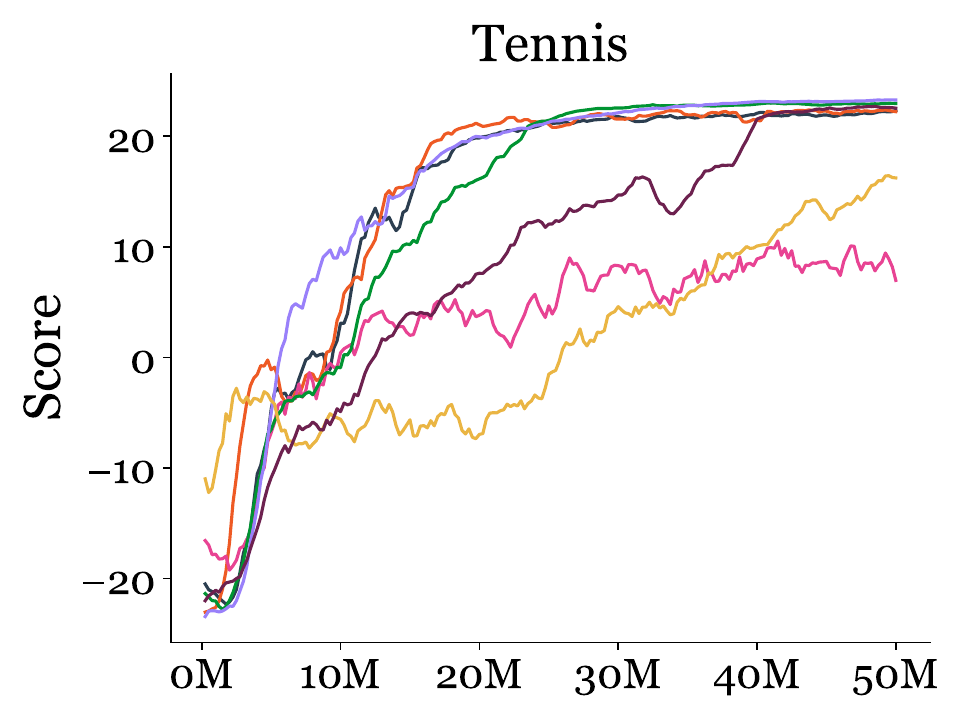} 
	\includegraphics[width=0.21\linewidth]{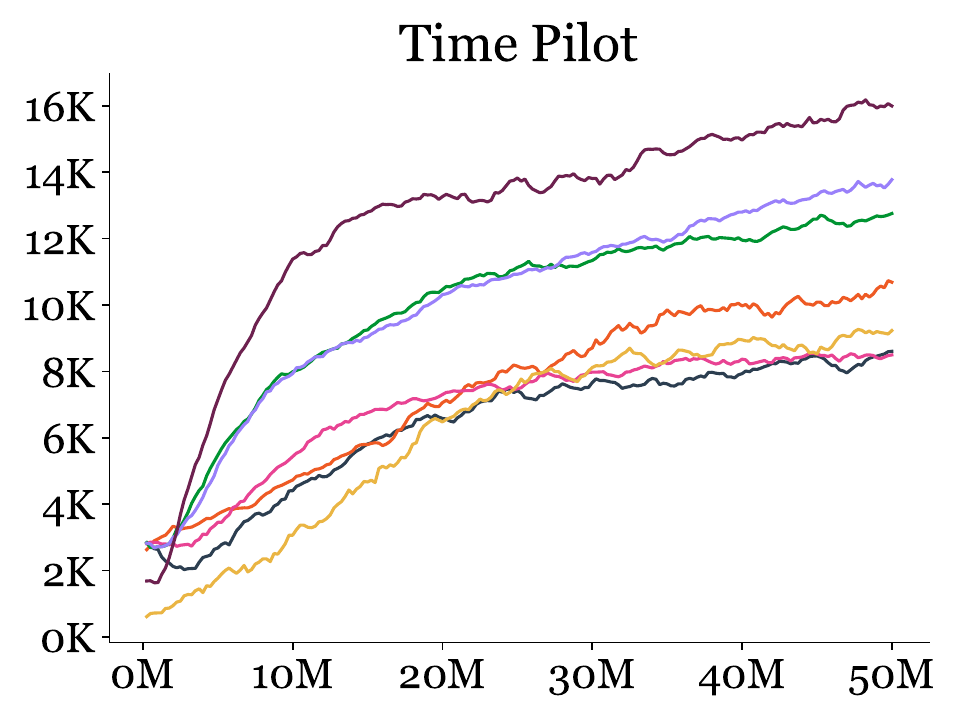} 
	\includegraphics[width=0.21\linewidth]{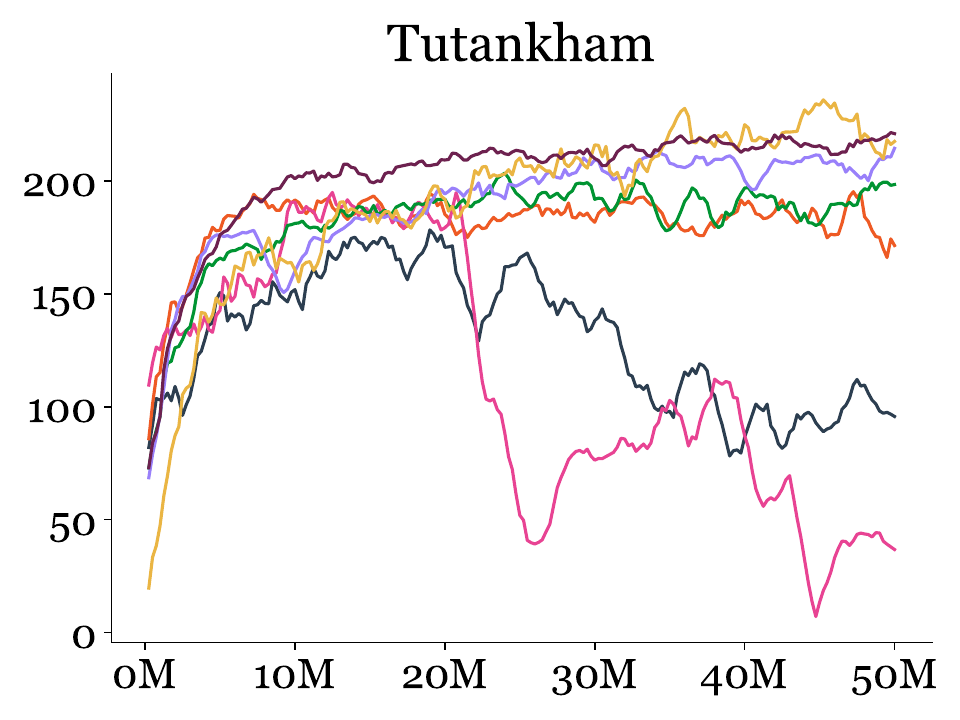} 
	\includegraphics[width=0.21\linewidth]{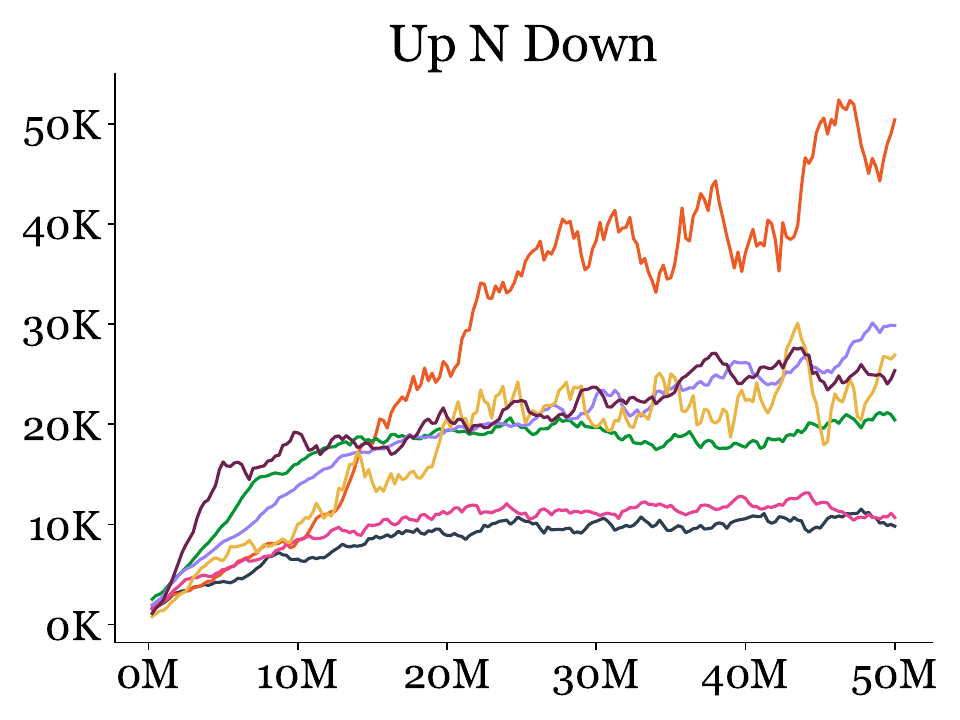} 
	\includegraphics[width=0.21\linewidth]{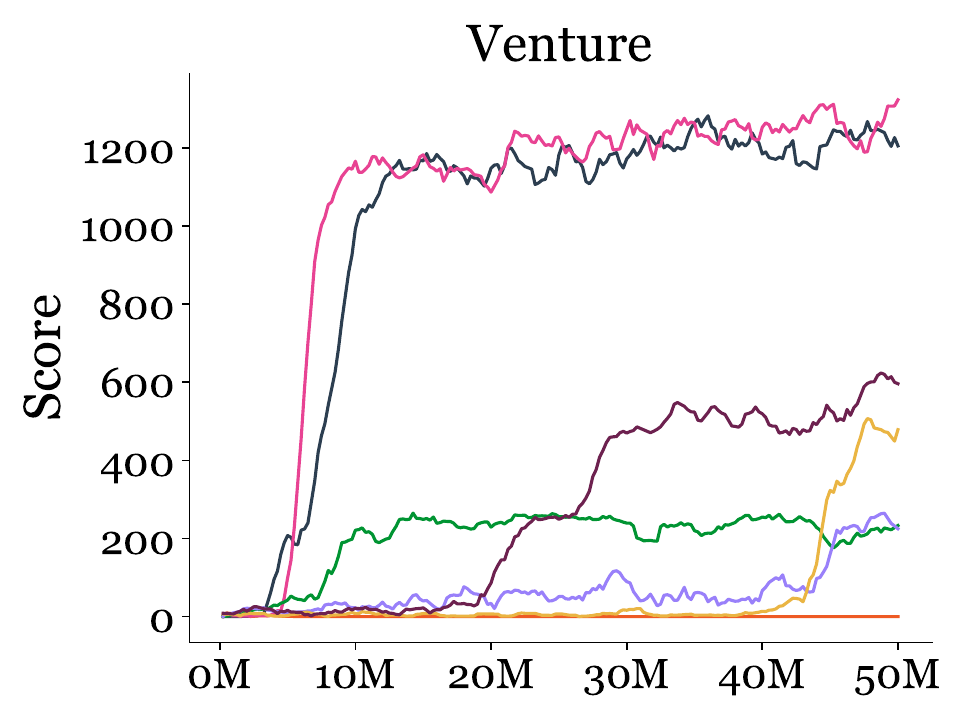} 
	\includegraphics[width=0.21\linewidth]{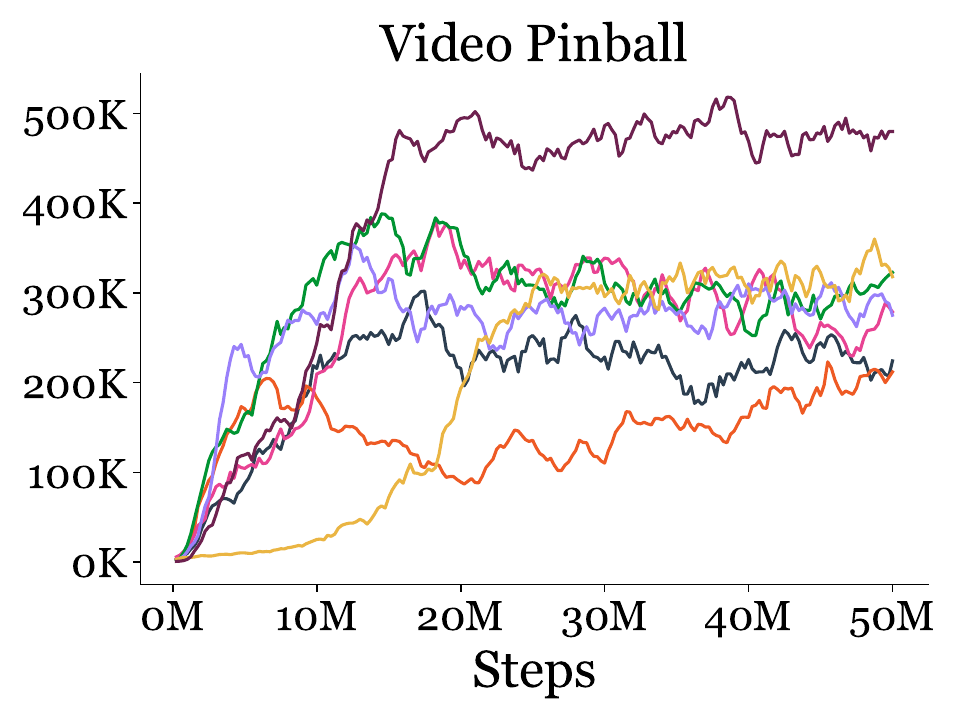} 
	\includegraphics[width=0.21\linewidth]{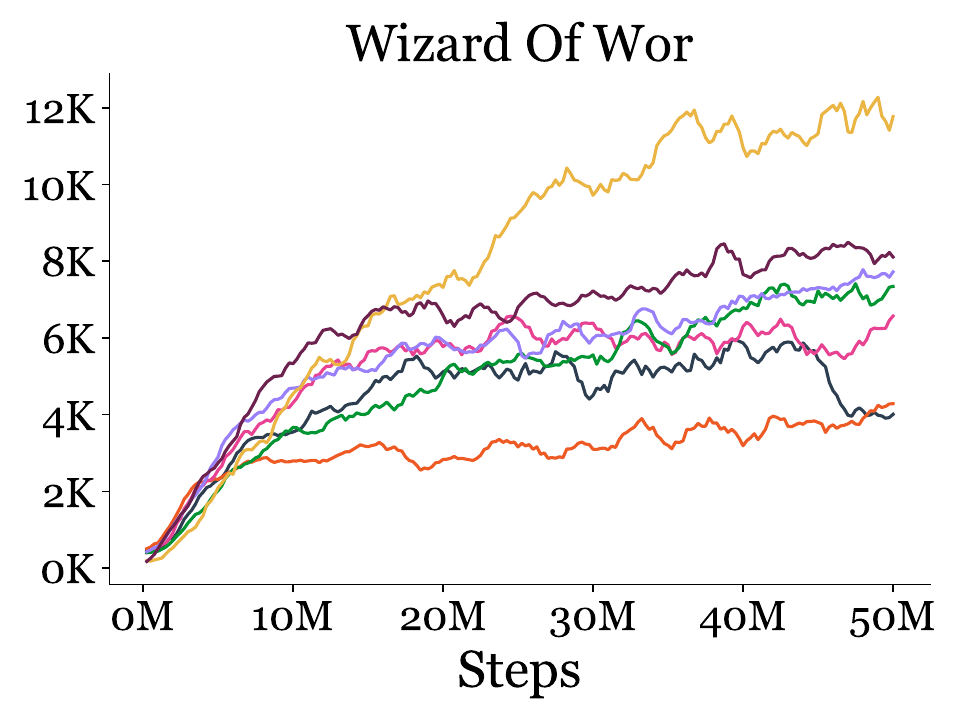} 
	\includegraphics[width=0.21\linewidth]{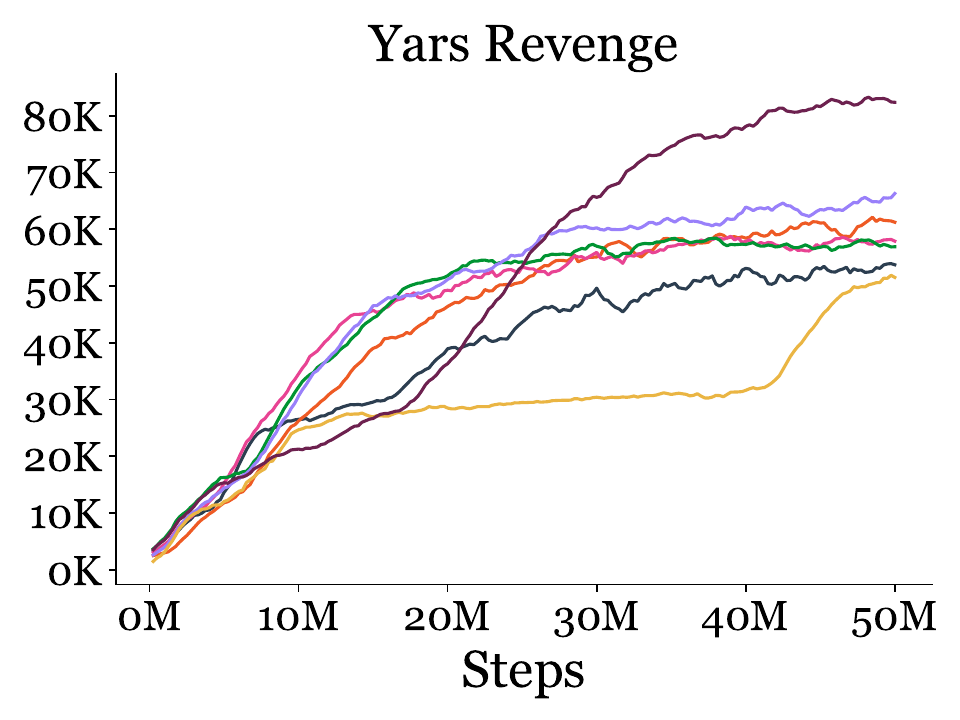} 
	\includegraphics[width=0.21\linewidth]{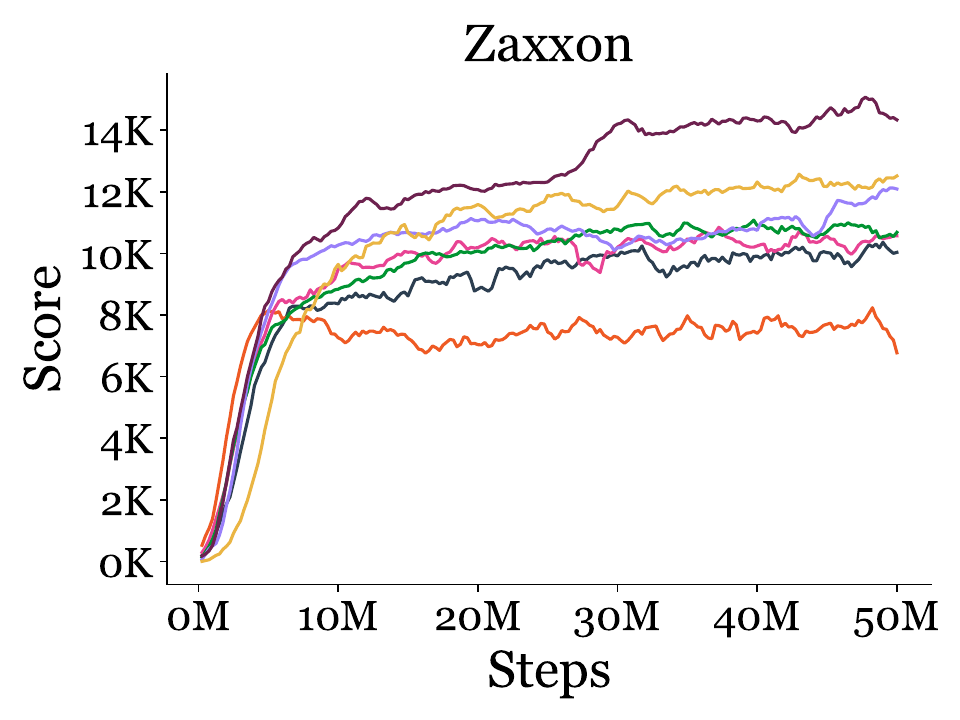} 
	\hspace{0.01\linewidth} \raisebox{2.2em}[0pt][0pt]{
    {\includegraphics[width=0.4\linewidth]
    {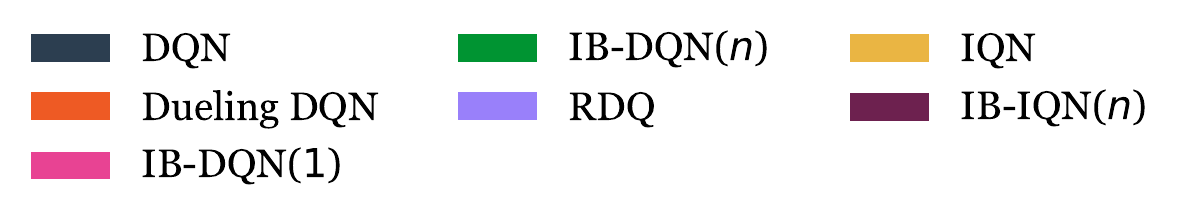}}}
    \hspace{0.205\linewidth}
	\caption{Mean score across 50M timesteps over five seeds per game. For readability, plots are smoothed with a moving average of seven.}
	\label{Atari57:score:page_2}
\end{figure}

\begin{figure}[p]
        \centering
    	\includegraphics[width=0.21\linewidth]{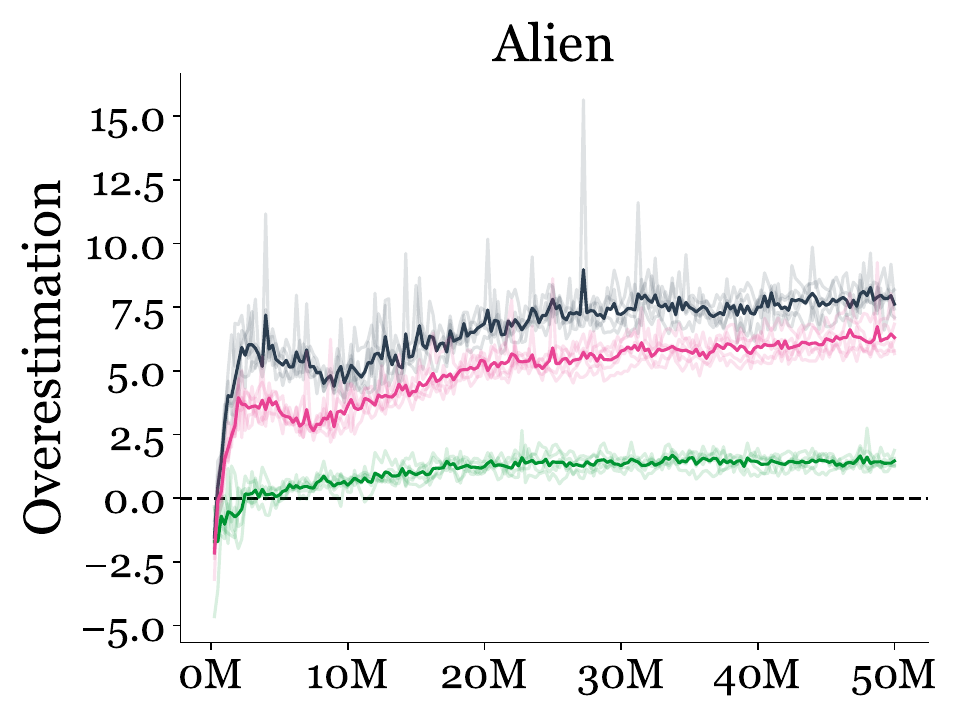} 
	\includegraphics[width=0.21\linewidth]{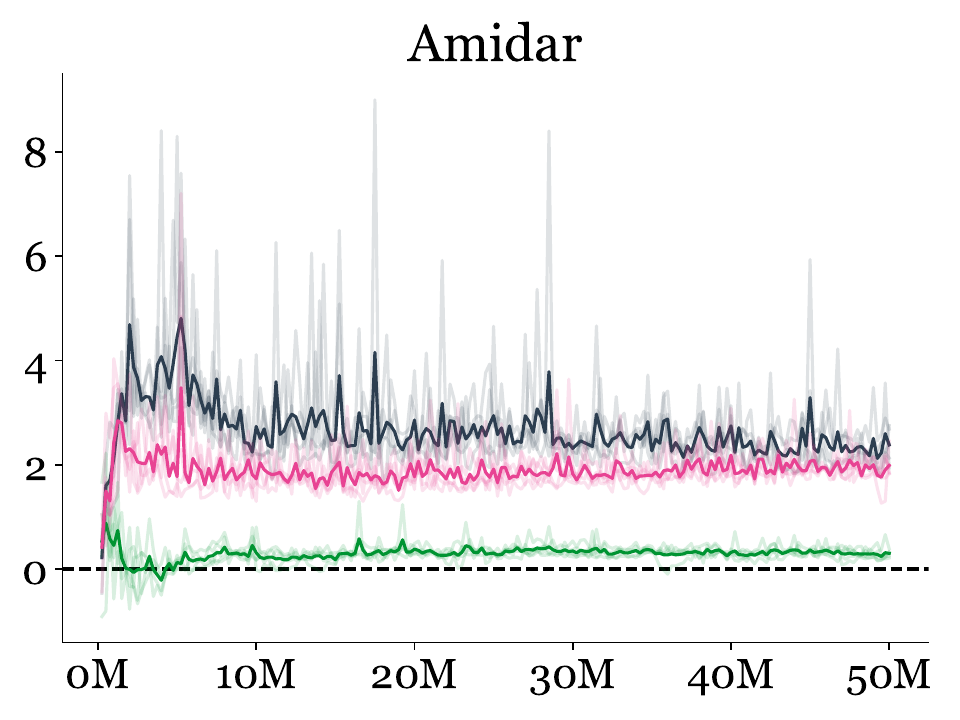} 
	\includegraphics[width=0.21\linewidth]{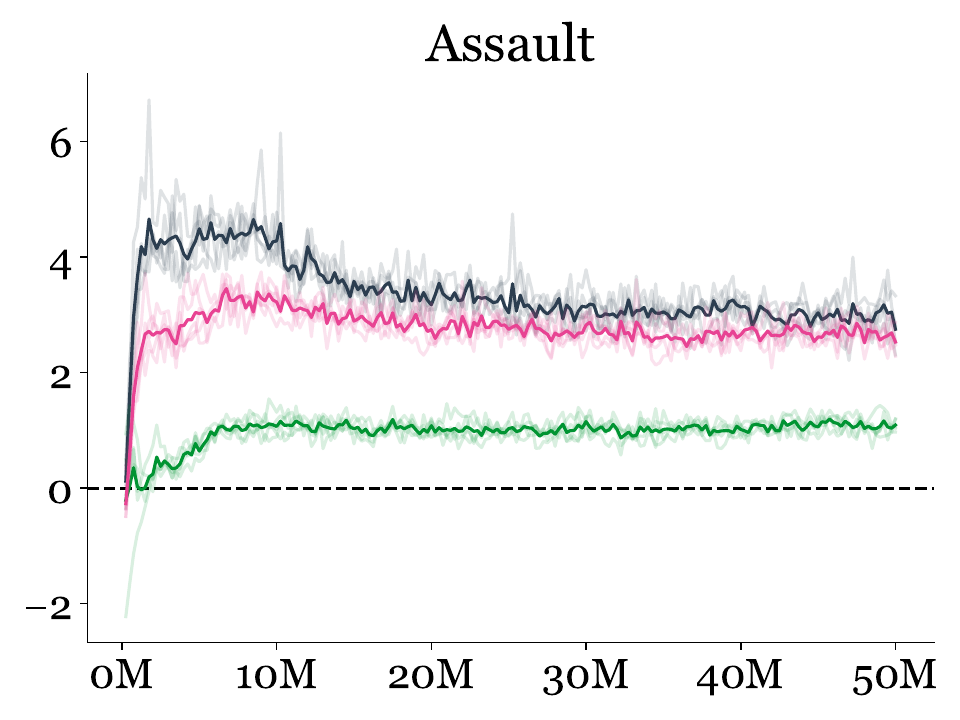} 
	\includegraphics[width=0.21\linewidth]{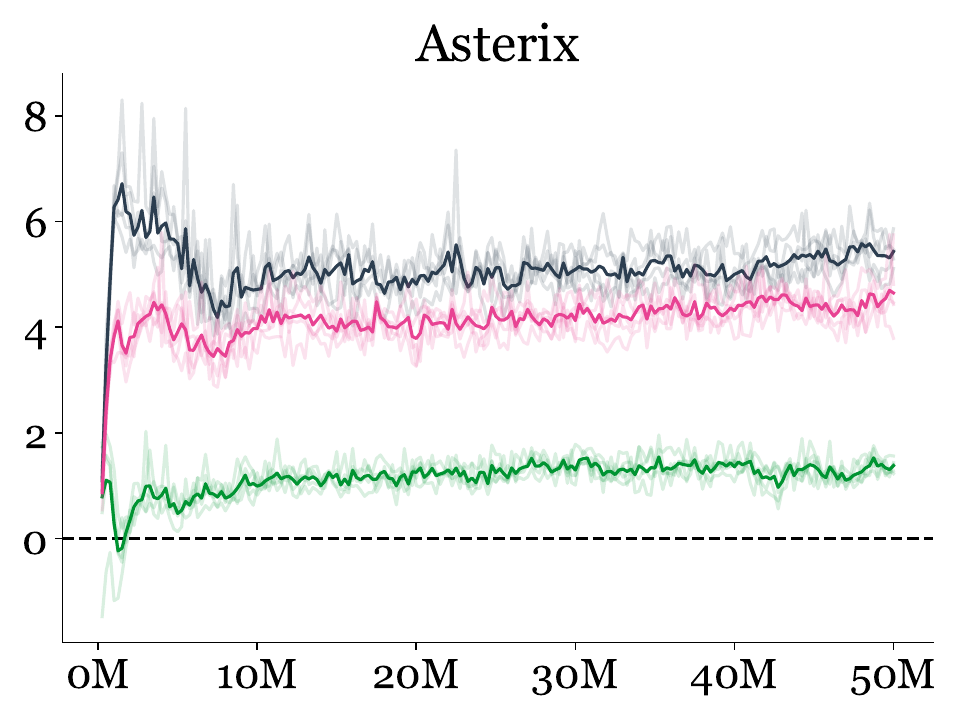} 
	\includegraphics[width=0.21\linewidth]{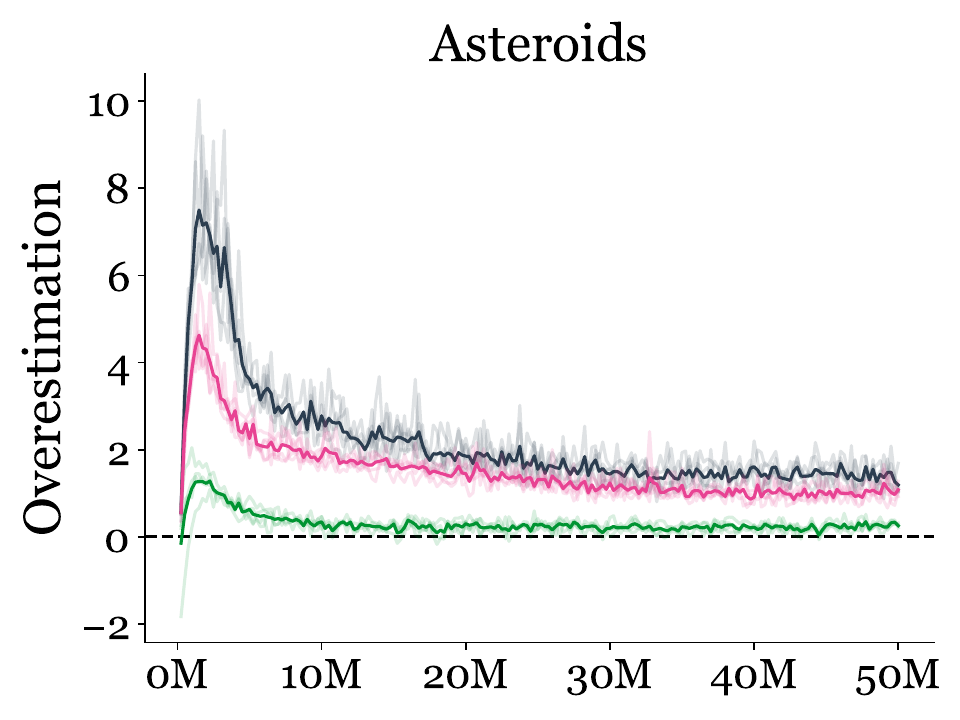} 
	\includegraphics[width=0.21\linewidth]{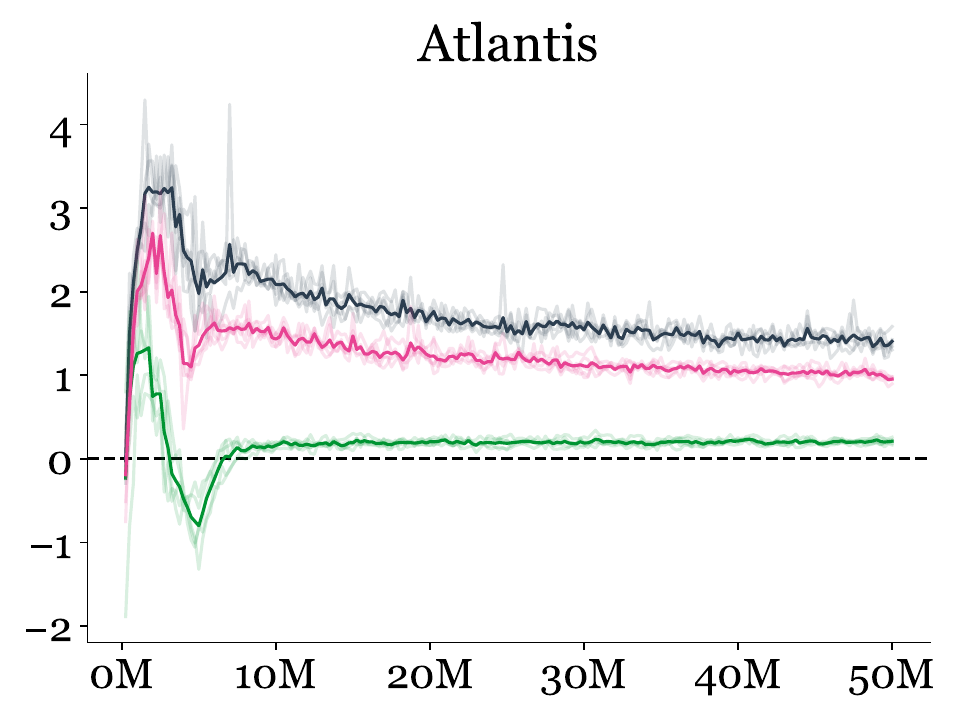} 
	\includegraphics[width=0.21\linewidth]{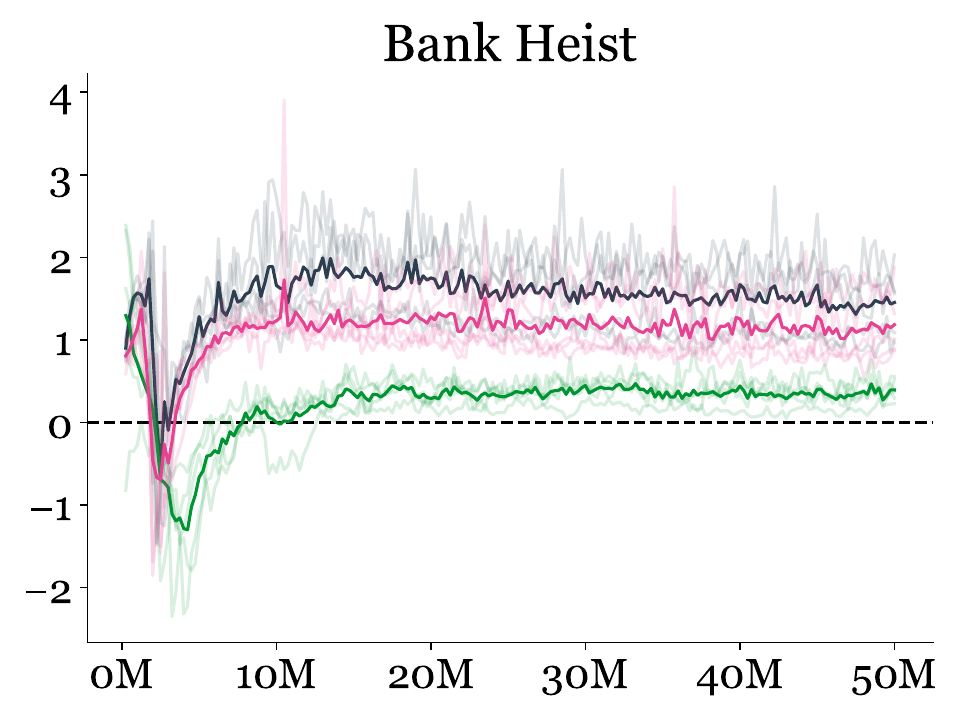} 
	\includegraphics[width=0.21\linewidth]{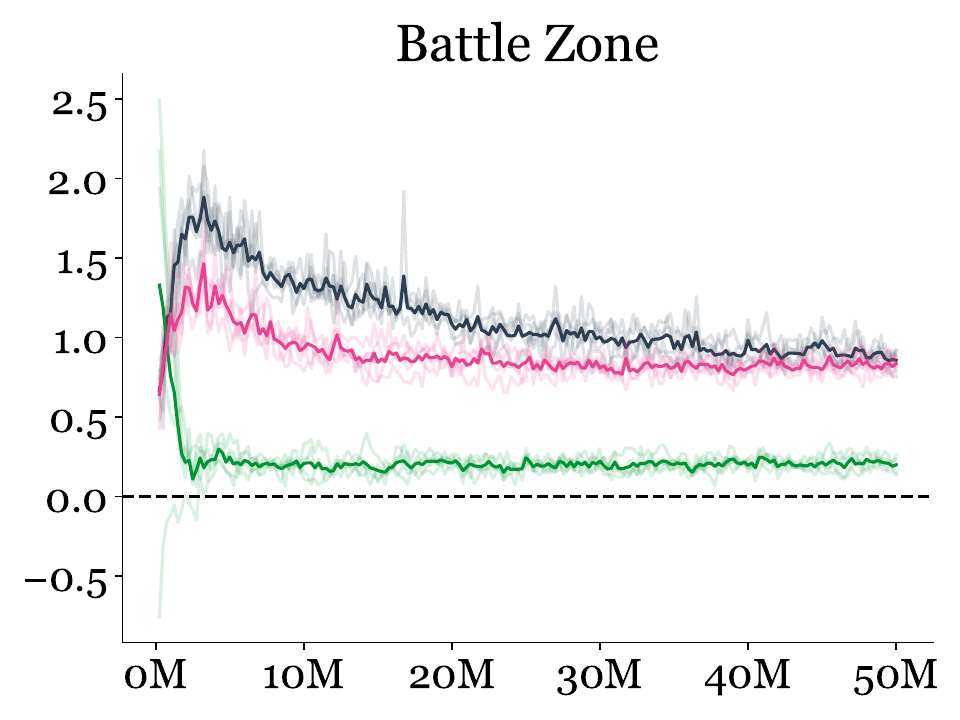} 
	\includegraphics[width=0.21\linewidth]{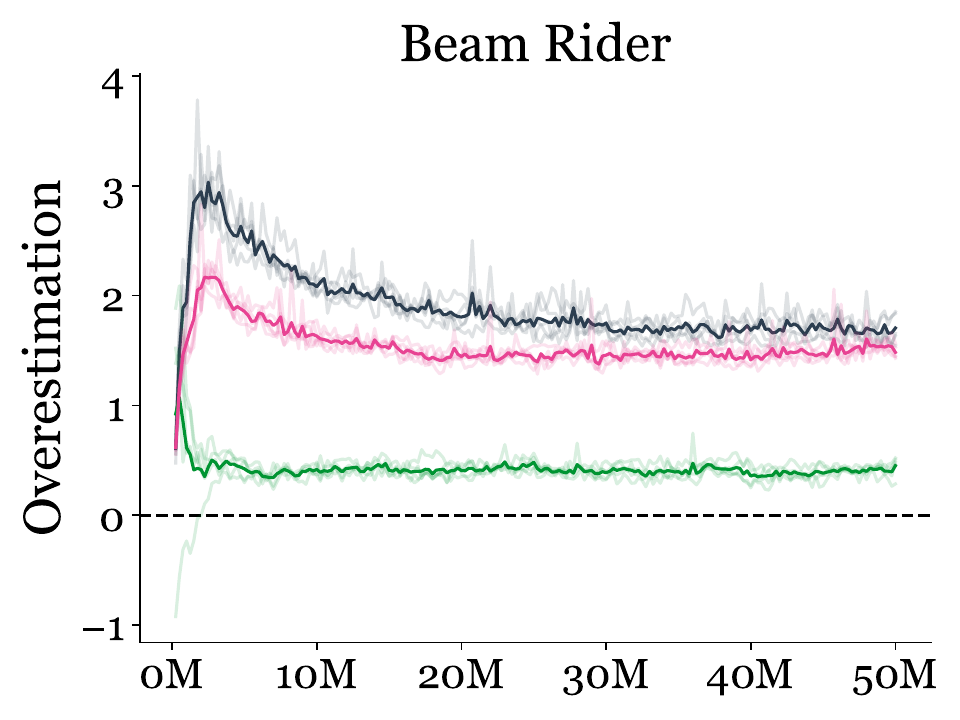} 
	\includegraphics[width=0.21\linewidth]{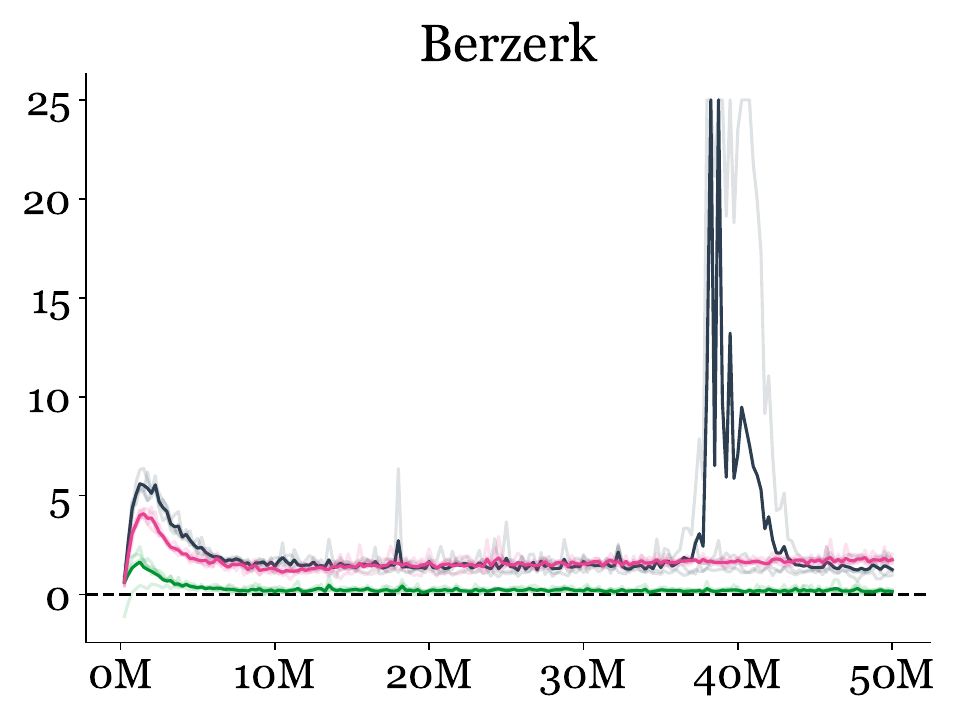} 
	\includegraphics[width=0.21\linewidth]{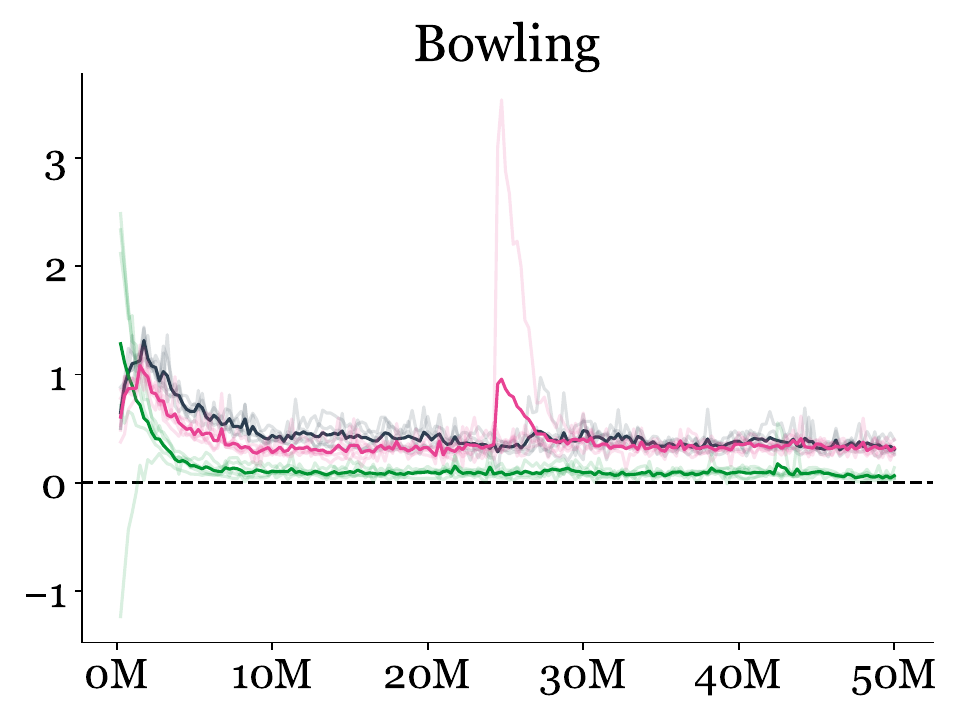} 
	\includegraphics[width=0.21\linewidth]{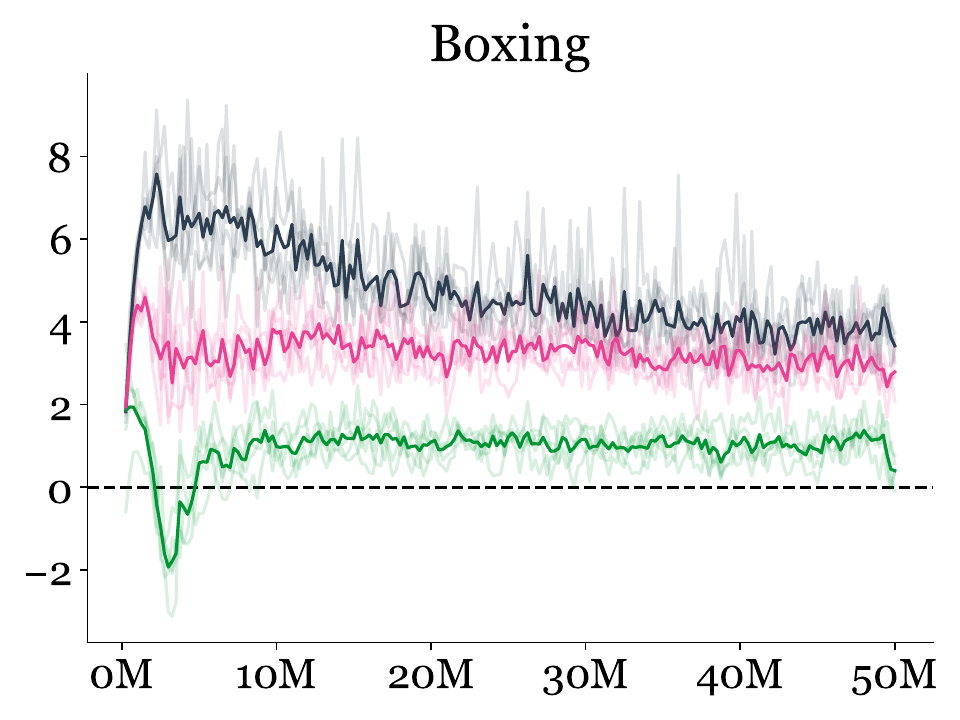} 
	\includegraphics[width=0.21\linewidth]{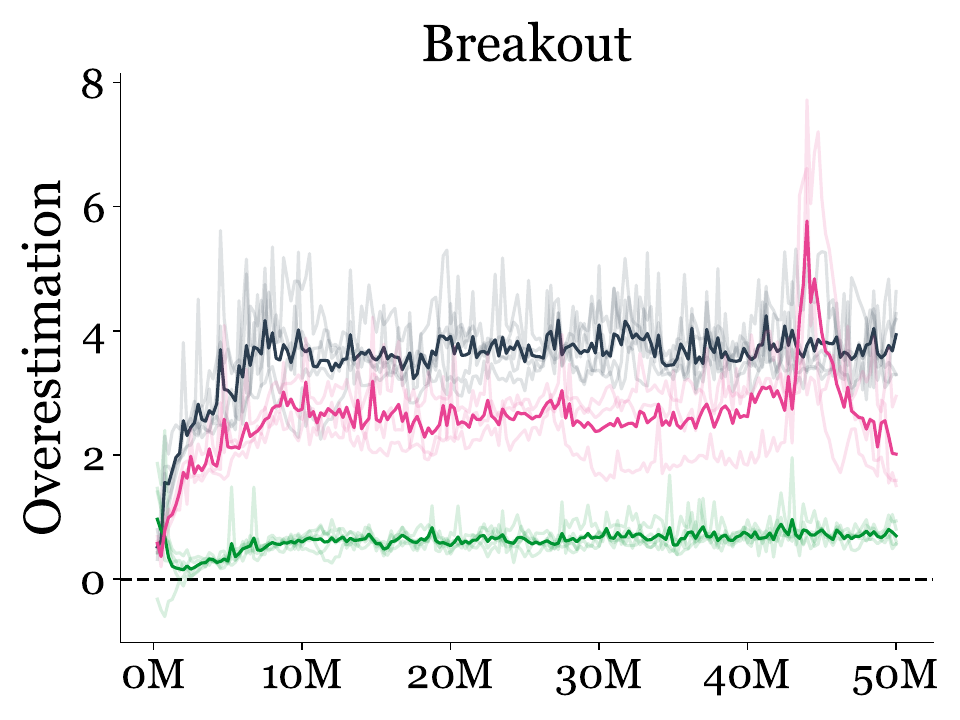} 
	\includegraphics[width=0.21\linewidth]{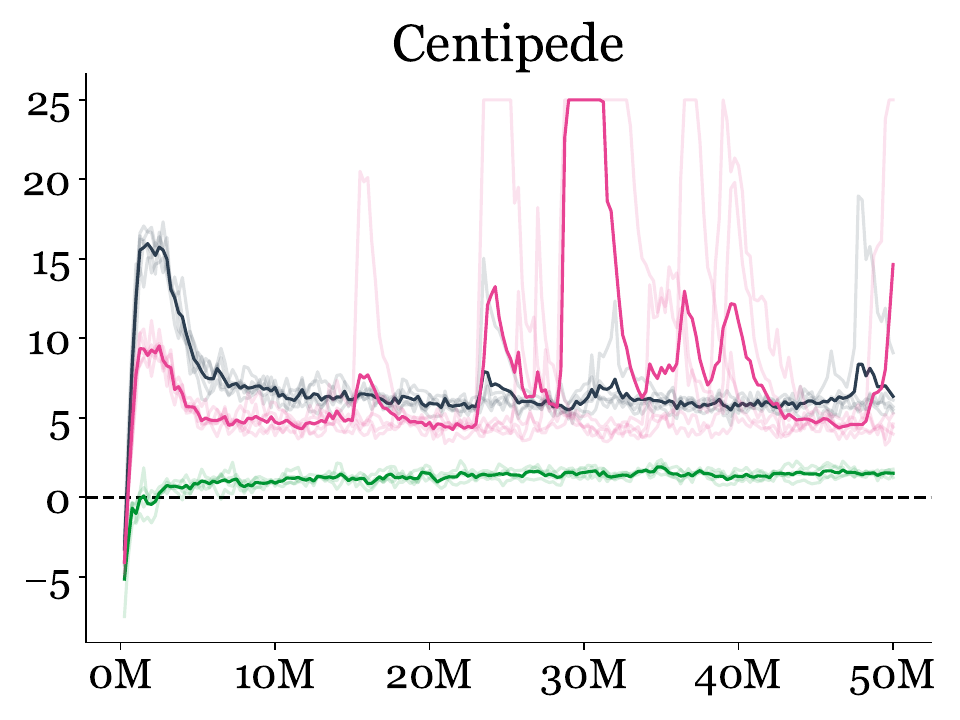} 
	\includegraphics[width=0.21\linewidth]{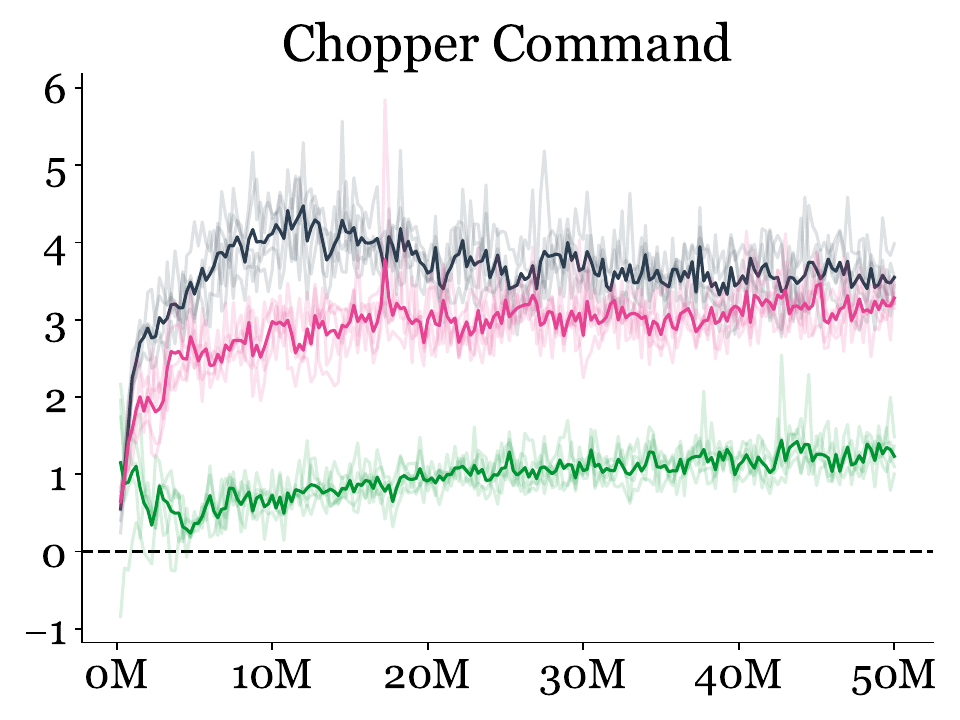} 
	\includegraphics[width=0.21\linewidth]{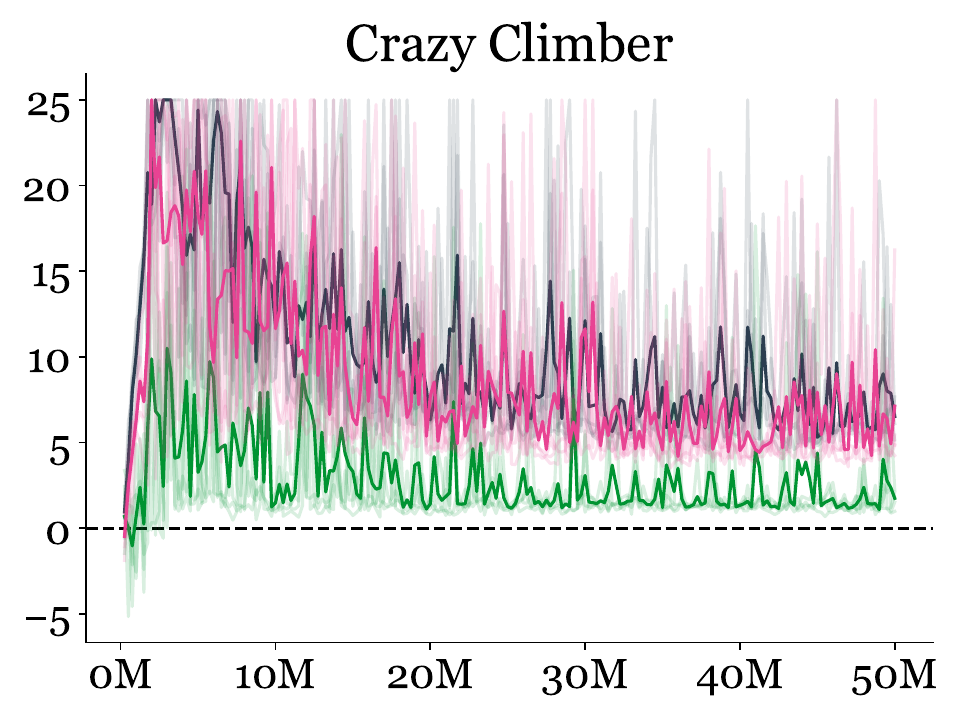} 
	\includegraphics[width=0.21\linewidth]{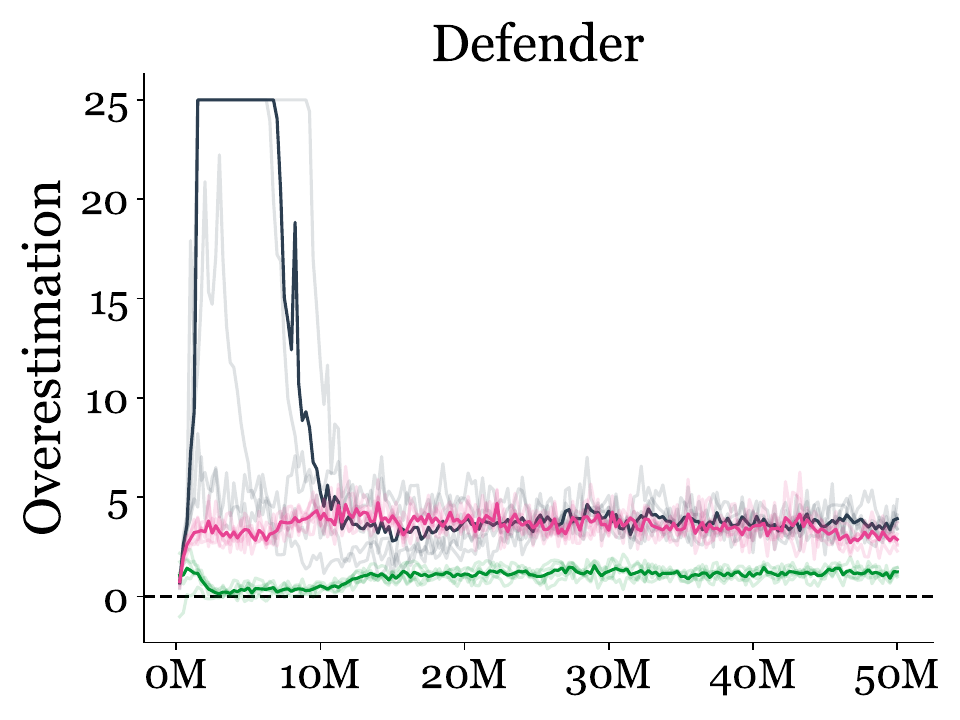} 
	\includegraphics[width=0.21\linewidth]{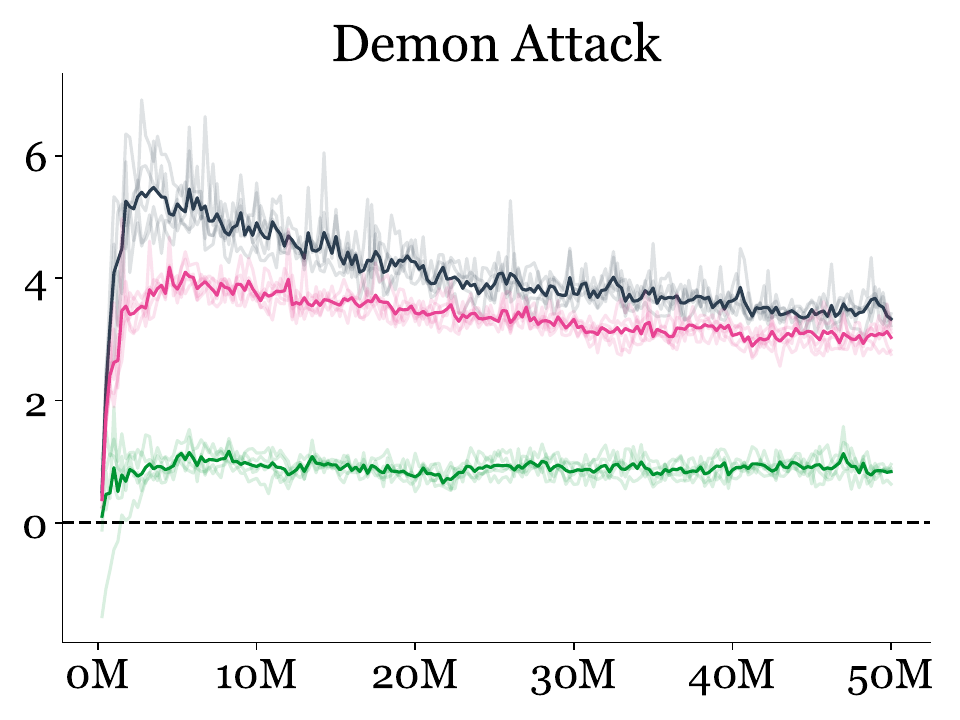} 
	\includegraphics[width=0.21\linewidth]{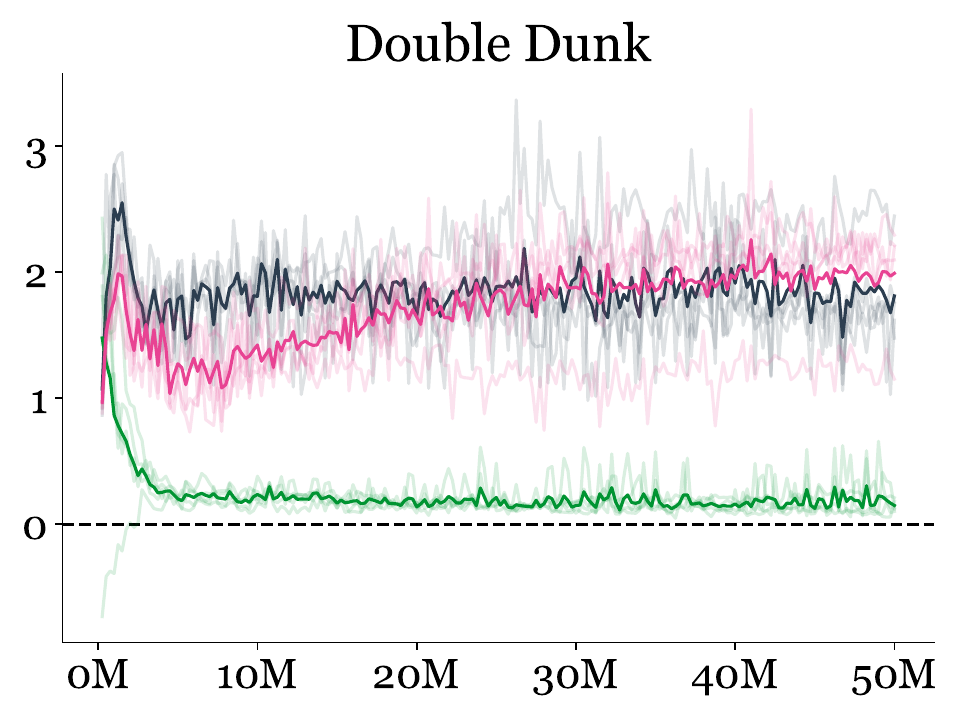} 
	\includegraphics[width=0.21\linewidth]{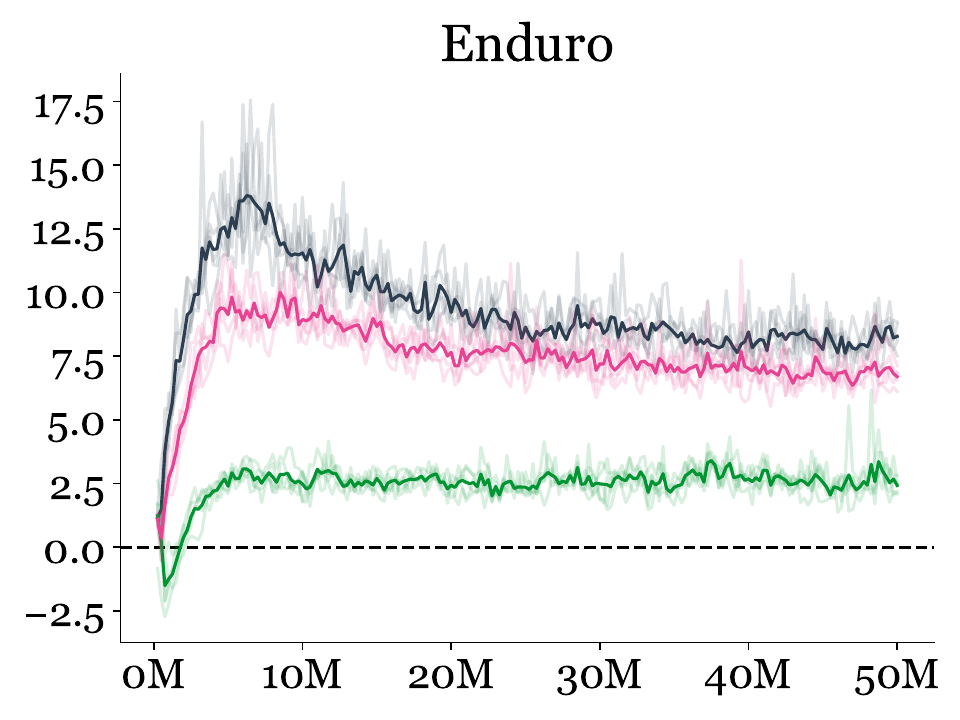} 
	\includegraphics[width=0.21\linewidth]{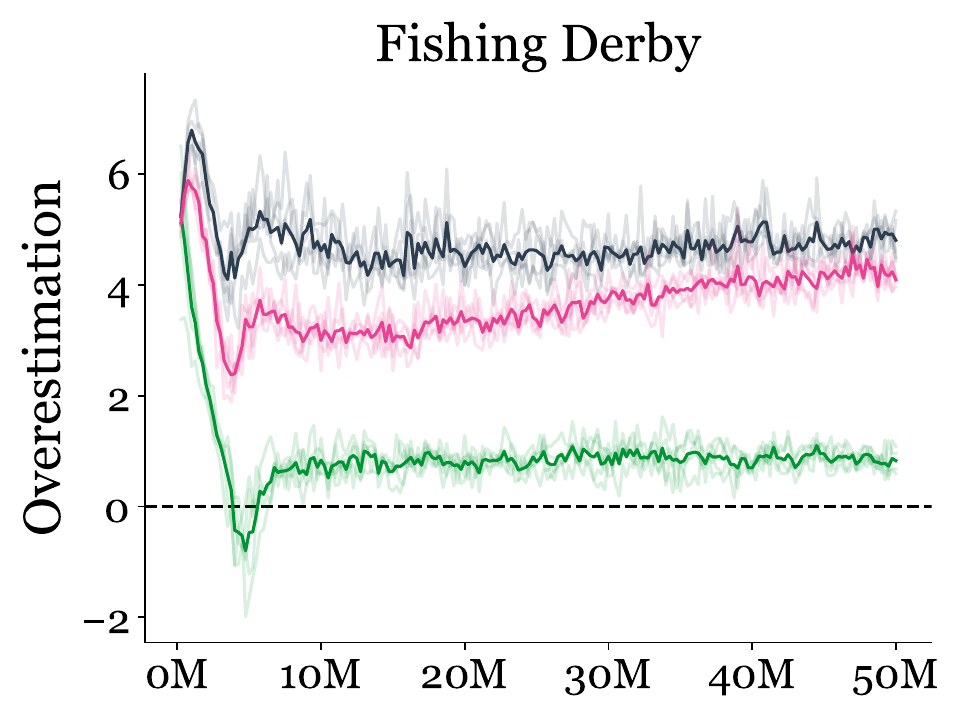} 
	\includegraphics[width=0.21\linewidth]{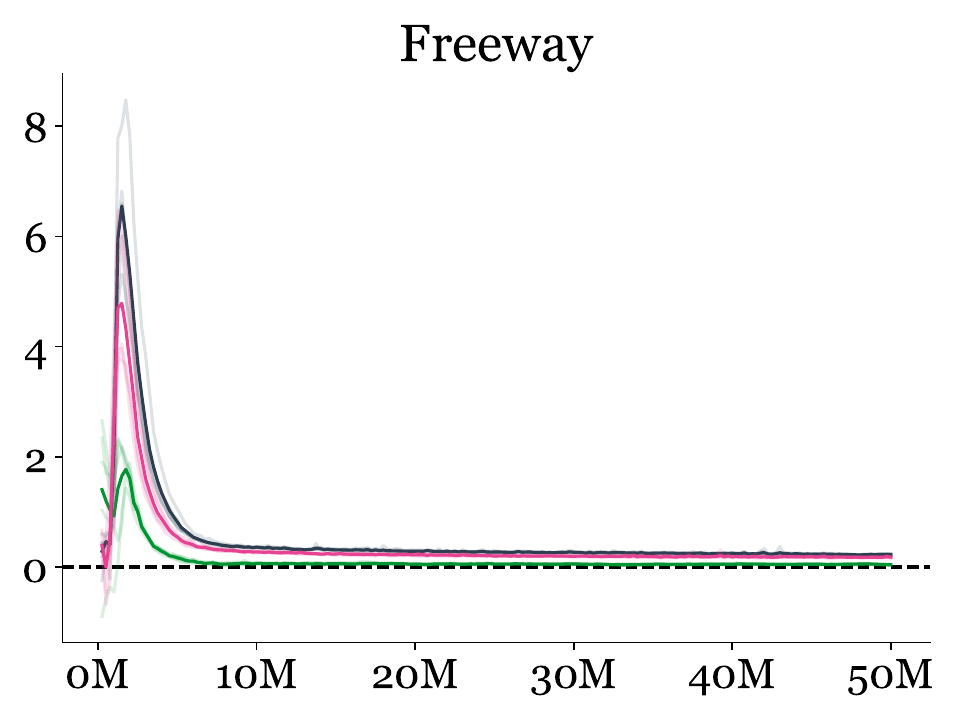} 
	\includegraphics[width=0.21\linewidth]{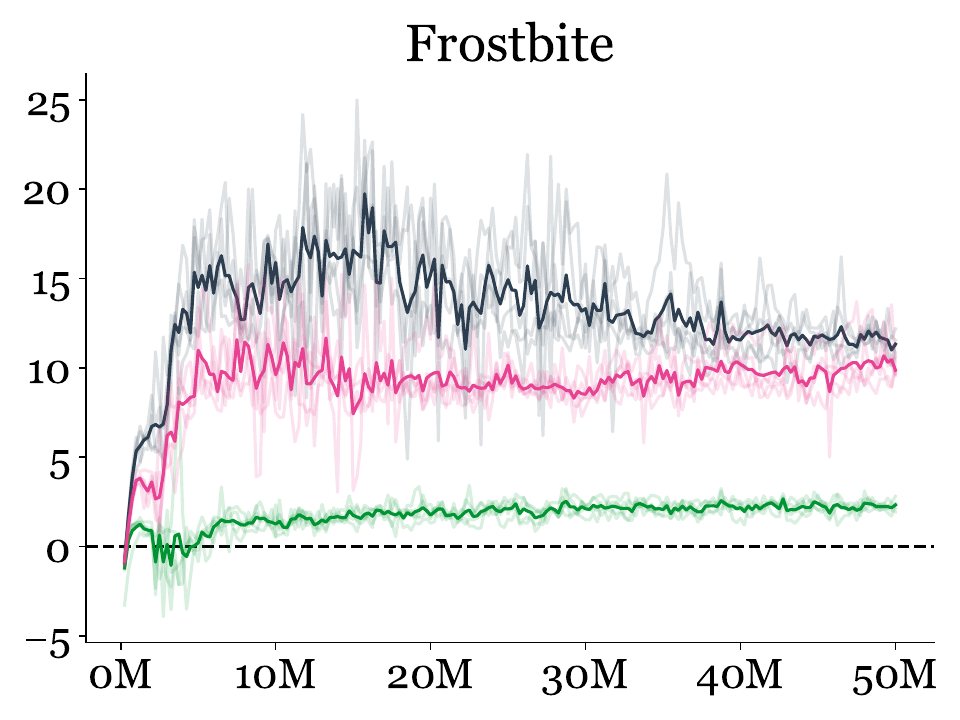} 
	\includegraphics[width=0.21\linewidth]{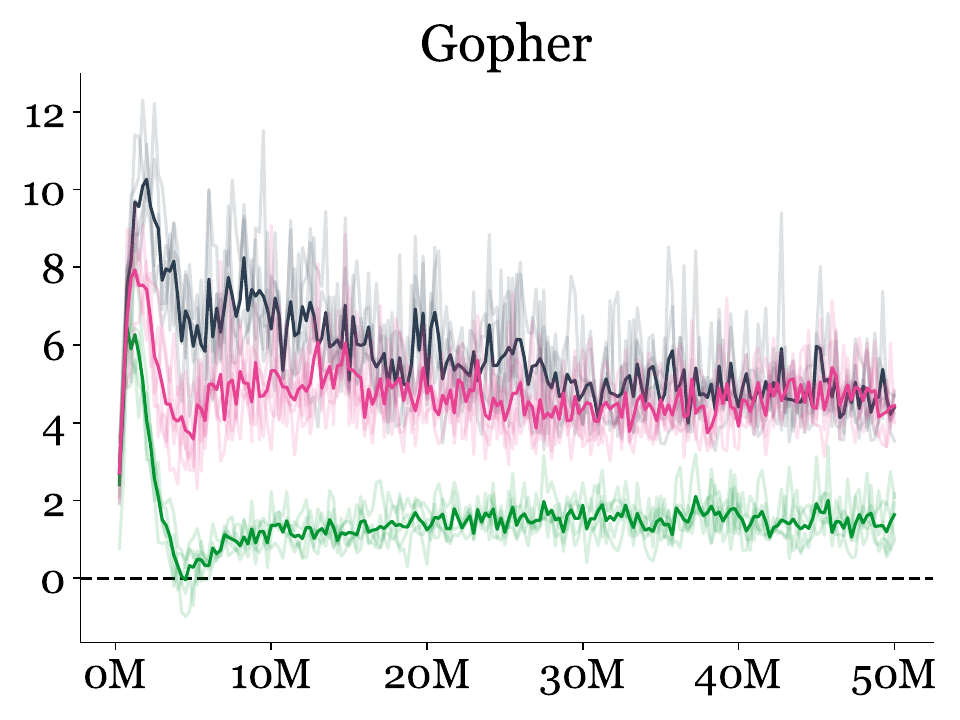} 
	\includegraphics[width=0.21\linewidth]{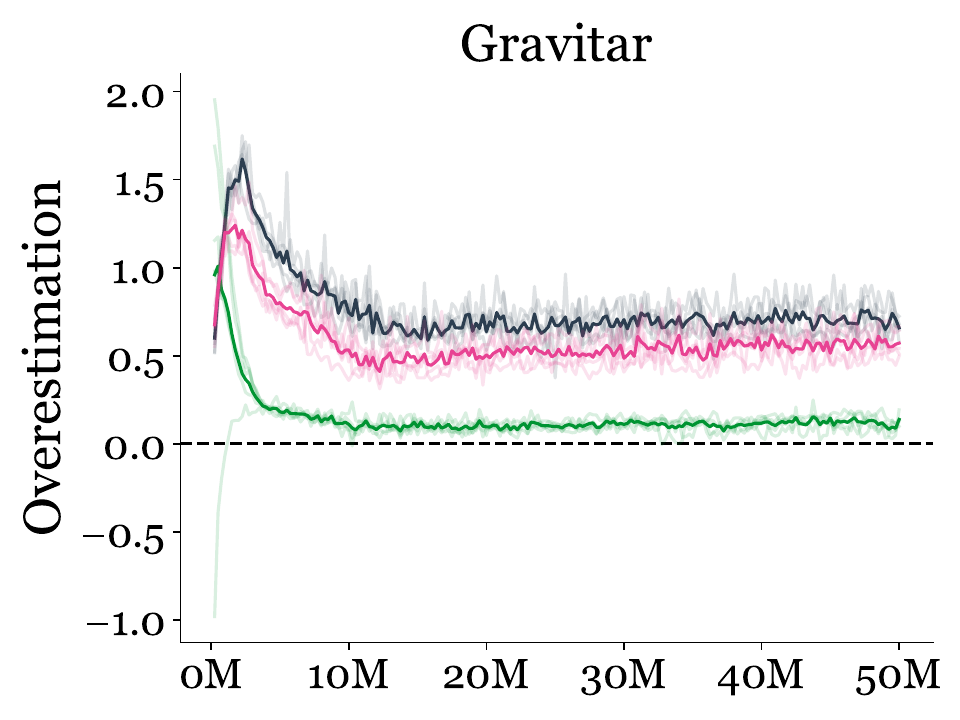} 
	\includegraphics[width=0.21\linewidth]{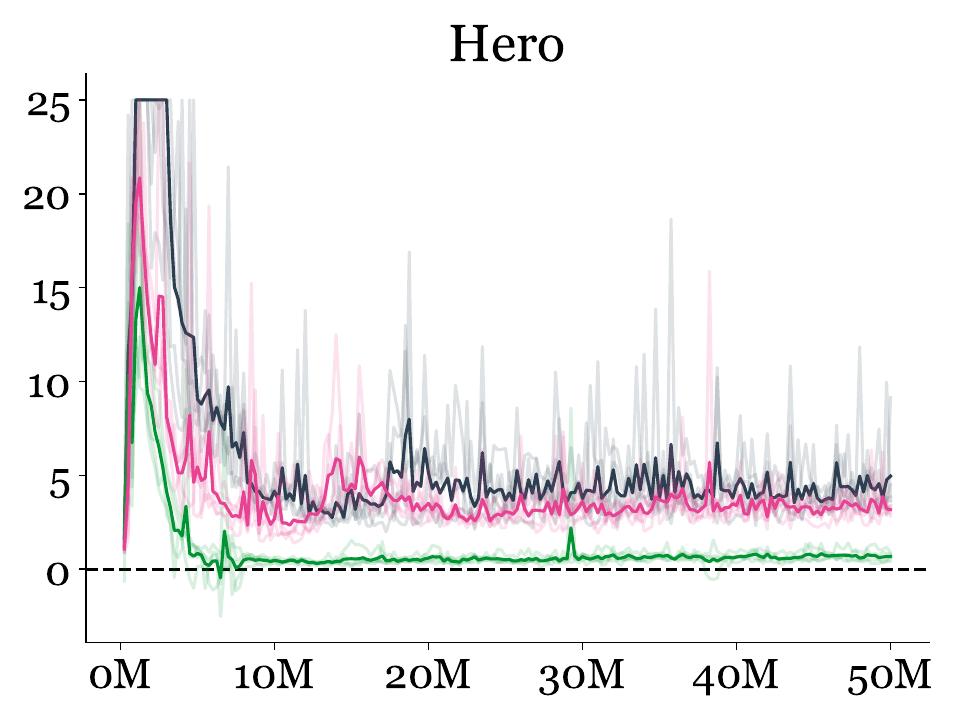} 
	\includegraphics[width=0.21\linewidth]{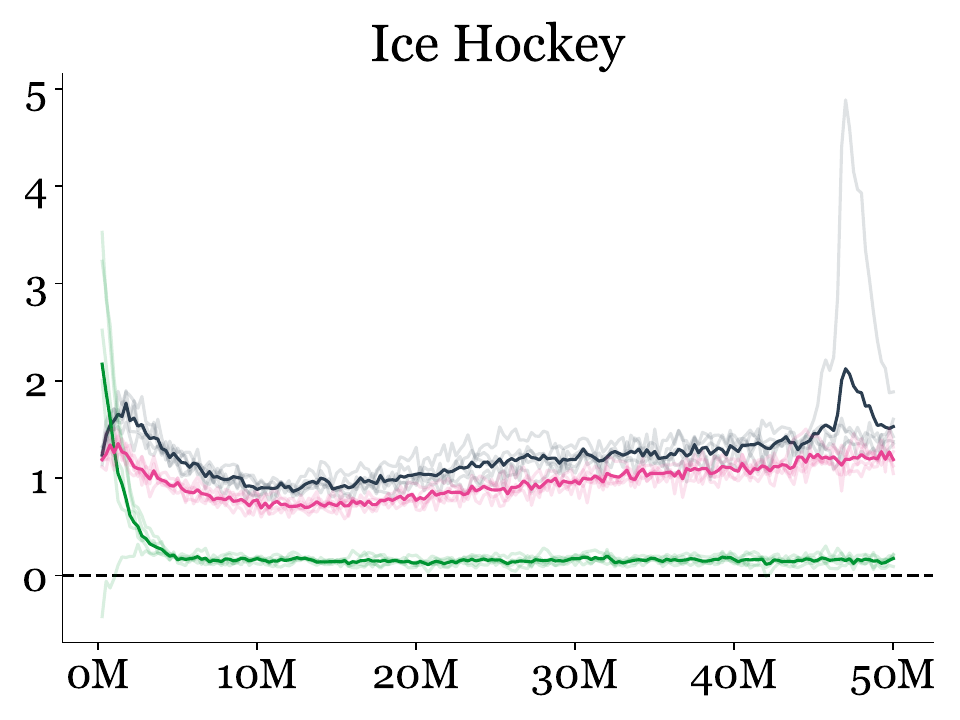} 
	\includegraphics[width=0.21\linewidth]{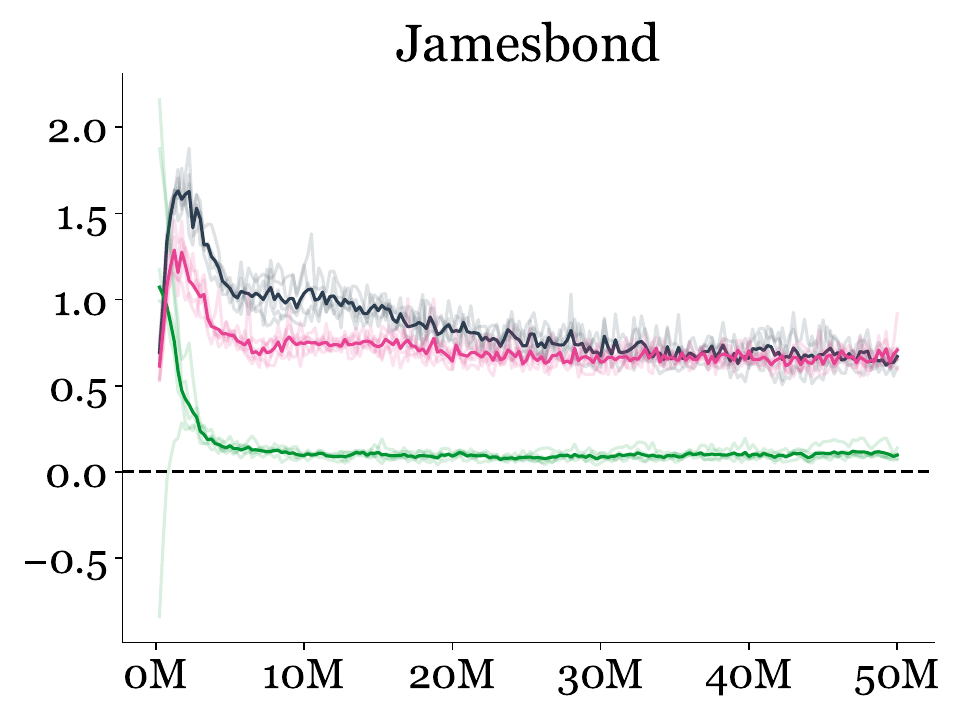} 
	\label{Atari57:Overestimation:page_1}
\end{figure}\begin{figure}[p]
        \centering
    	\includegraphics[width=0.21\linewidth]{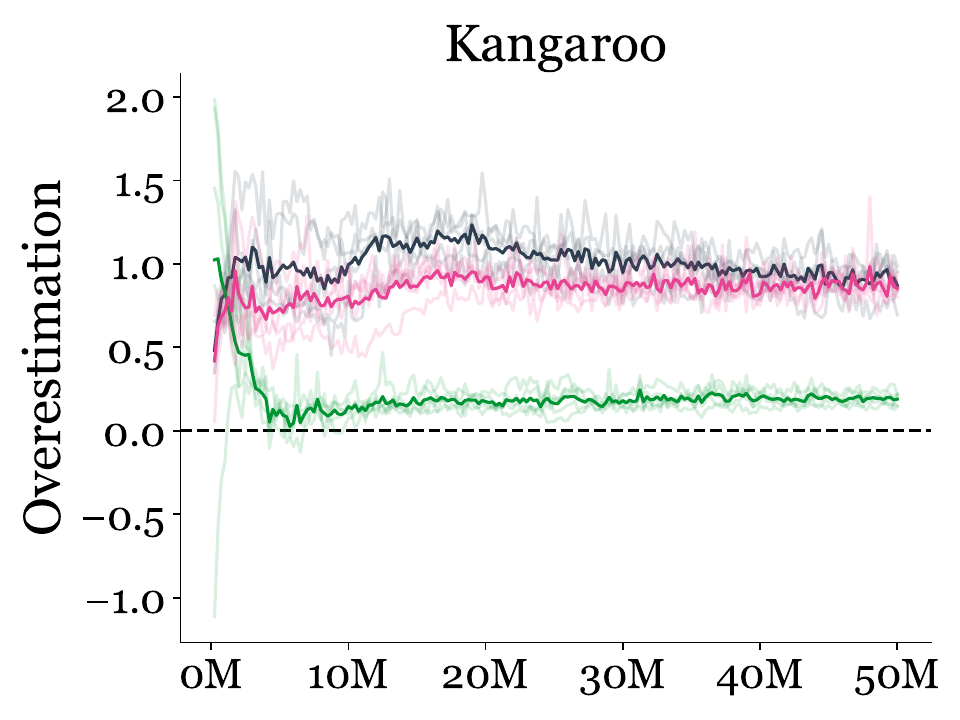} 
	\includegraphics[width=0.21\linewidth]{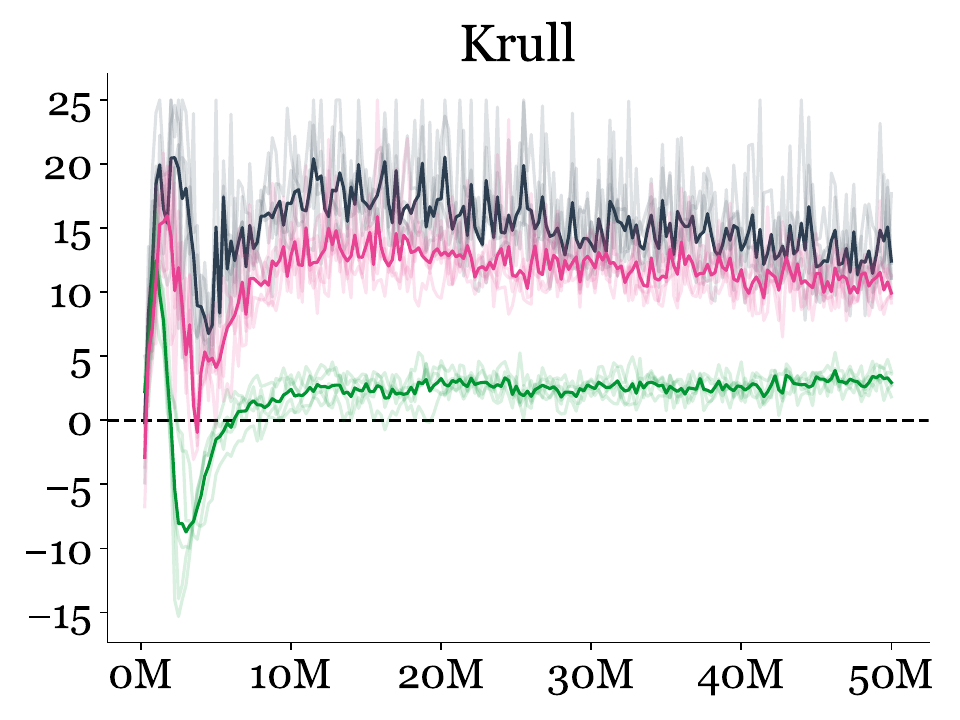} 
	\includegraphics[width=0.21\linewidth]{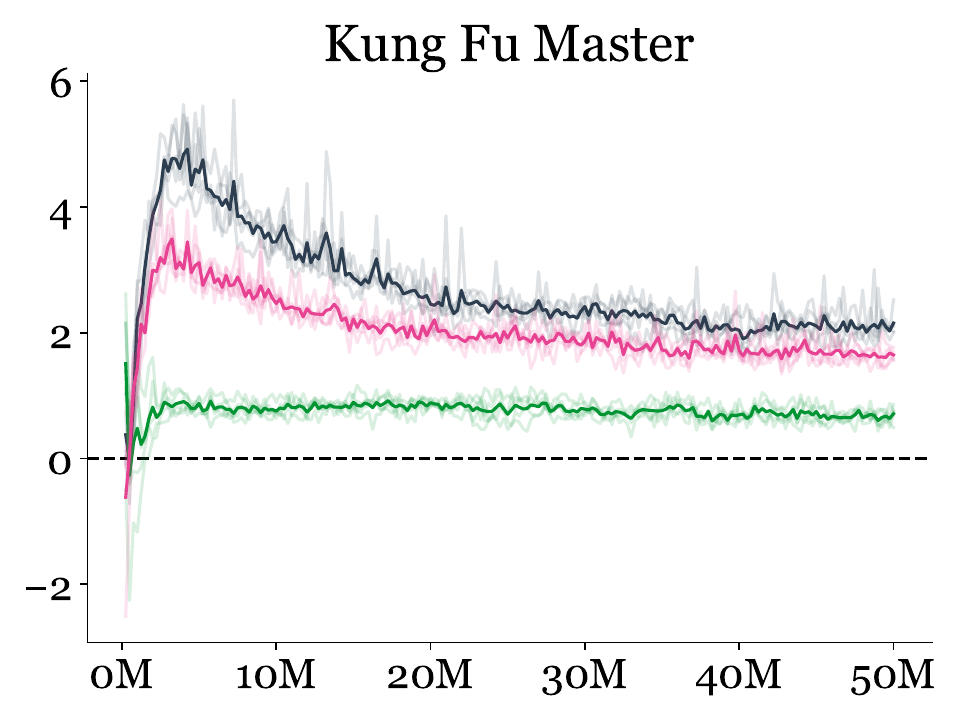} 
	\includegraphics[width=0.21\linewidth]{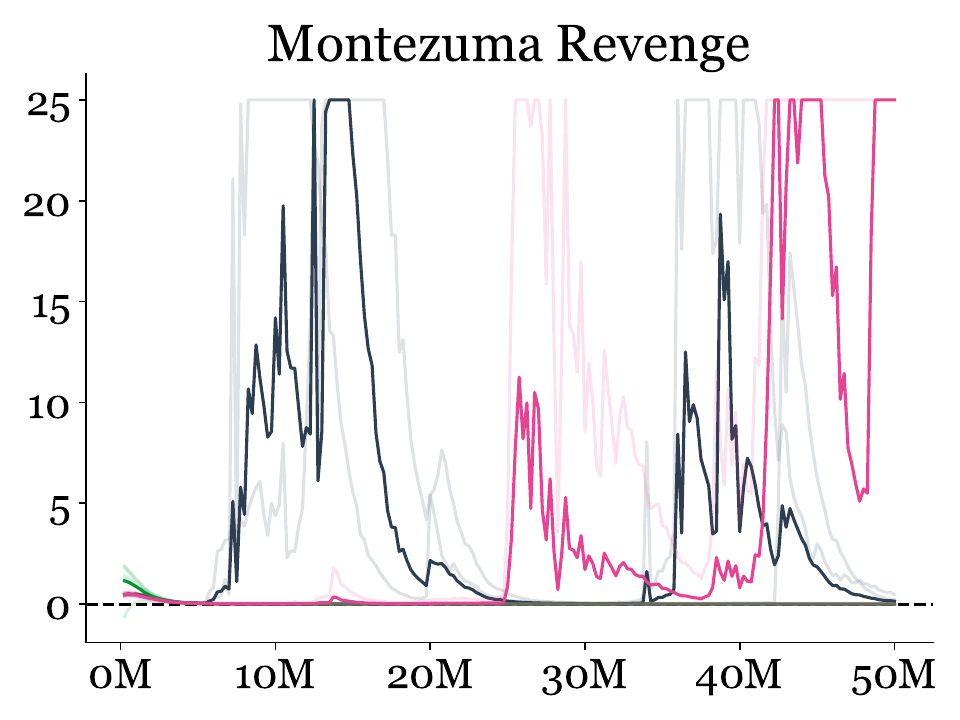} 
	\includegraphics[width=0.21\linewidth]{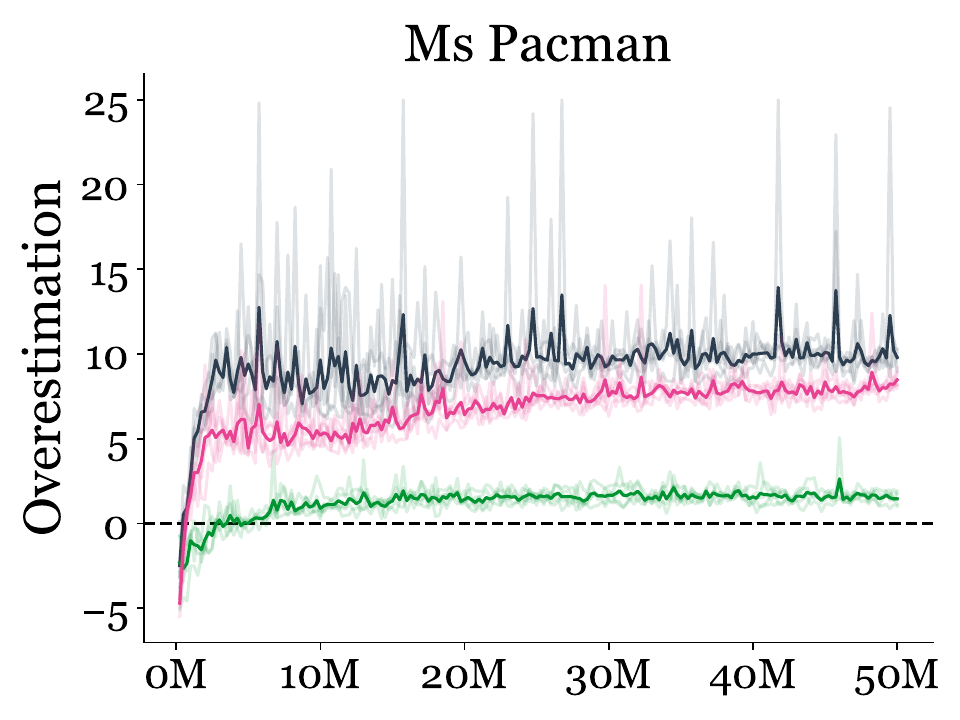} 
	\includegraphics[width=0.21\linewidth]{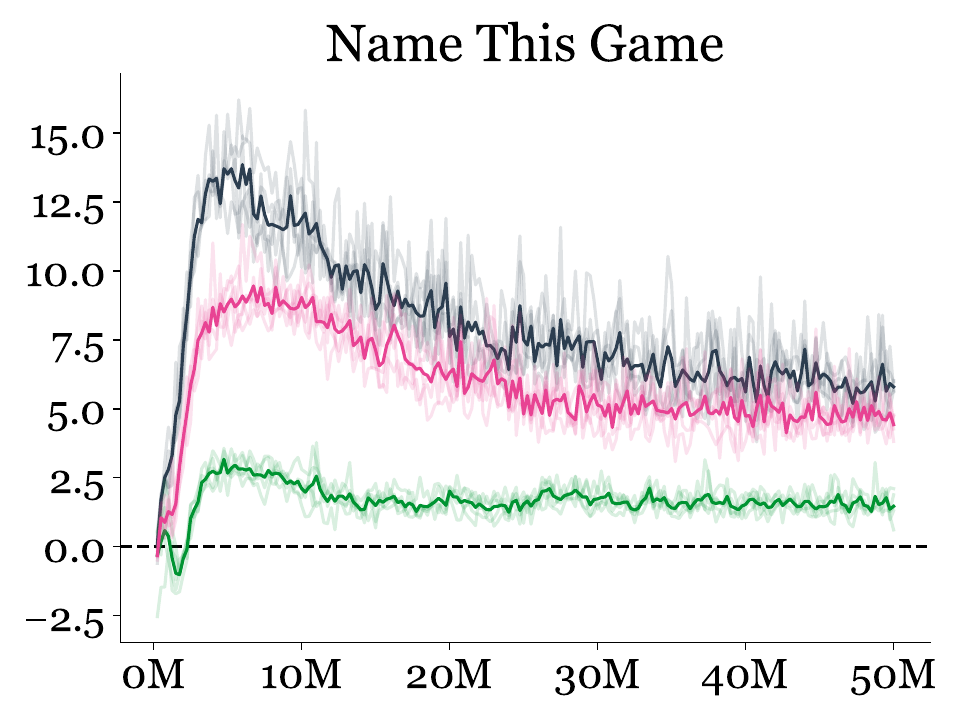} 
	\includegraphics[width=0.21\linewidth]{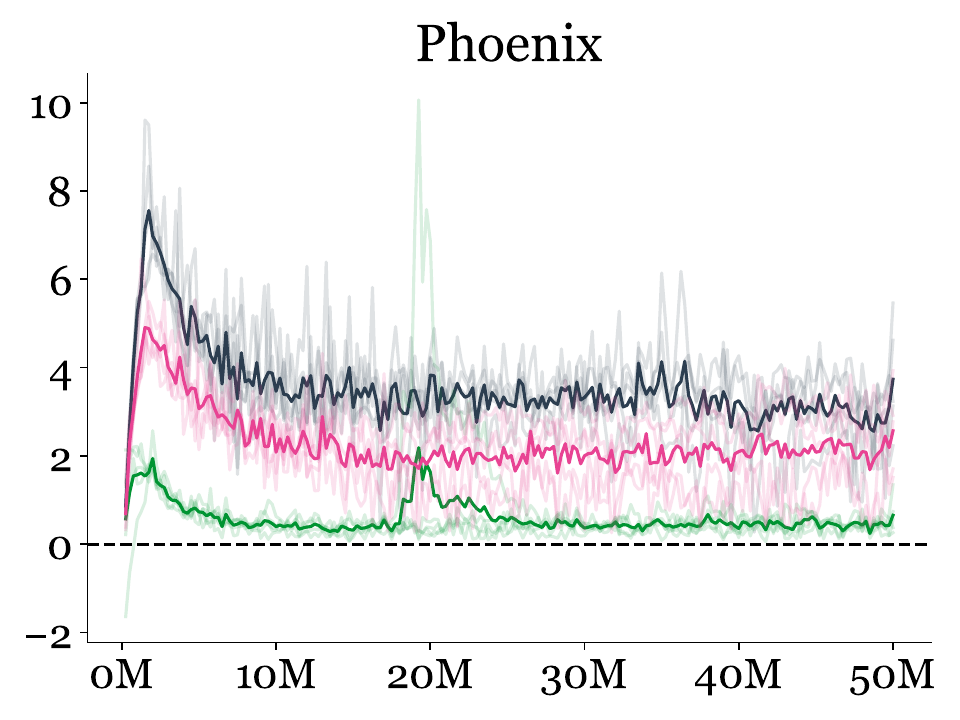} 
	\includegraphics[width=0.21\linewidth]{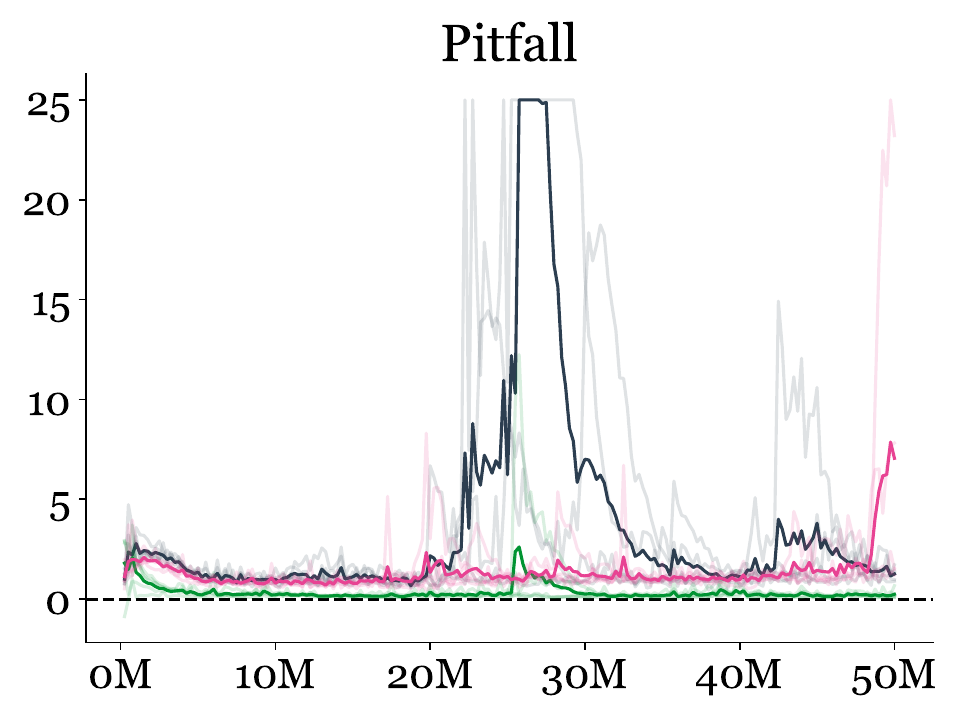} 
	\includegraphics[width=0.21\linewidth]{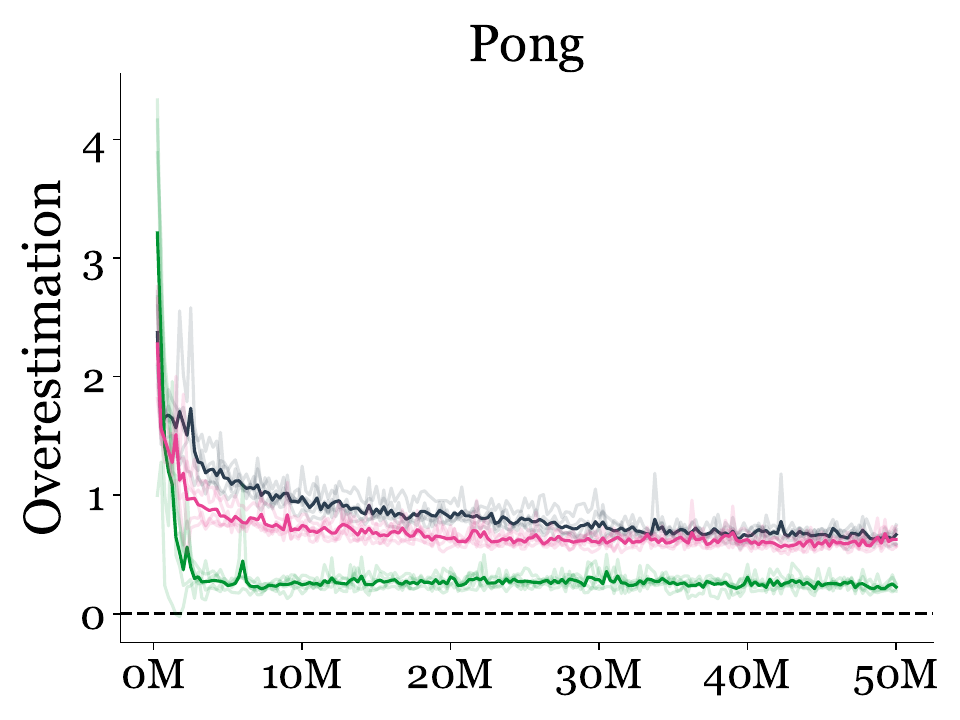} 
	\includegraphics[width=0.21\linewidth]{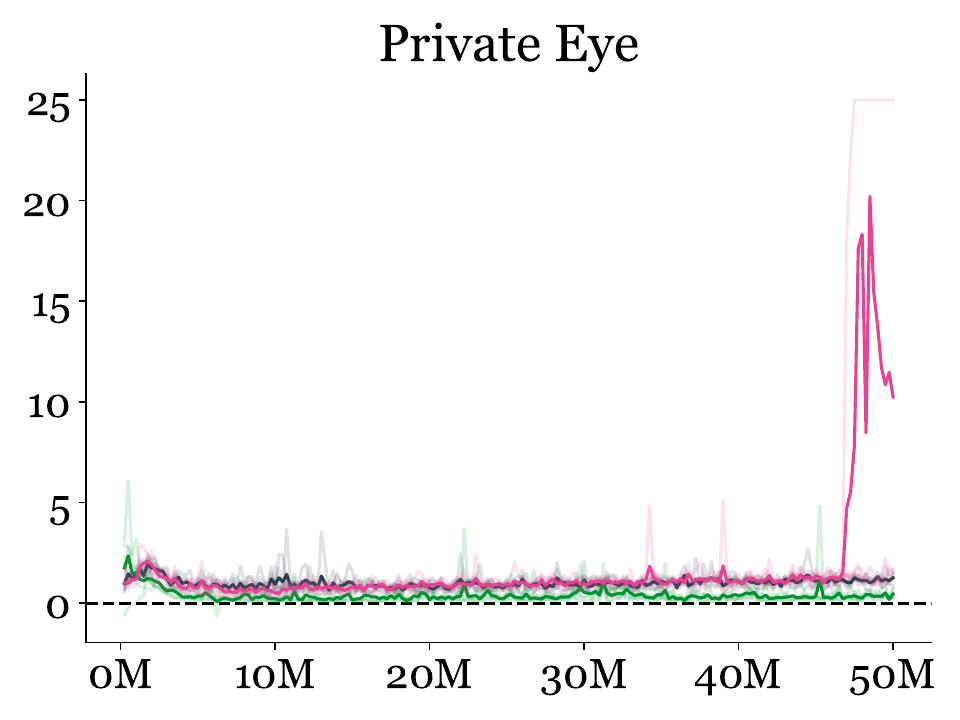} 
	\includegraphics[width=0.21\linewidth]{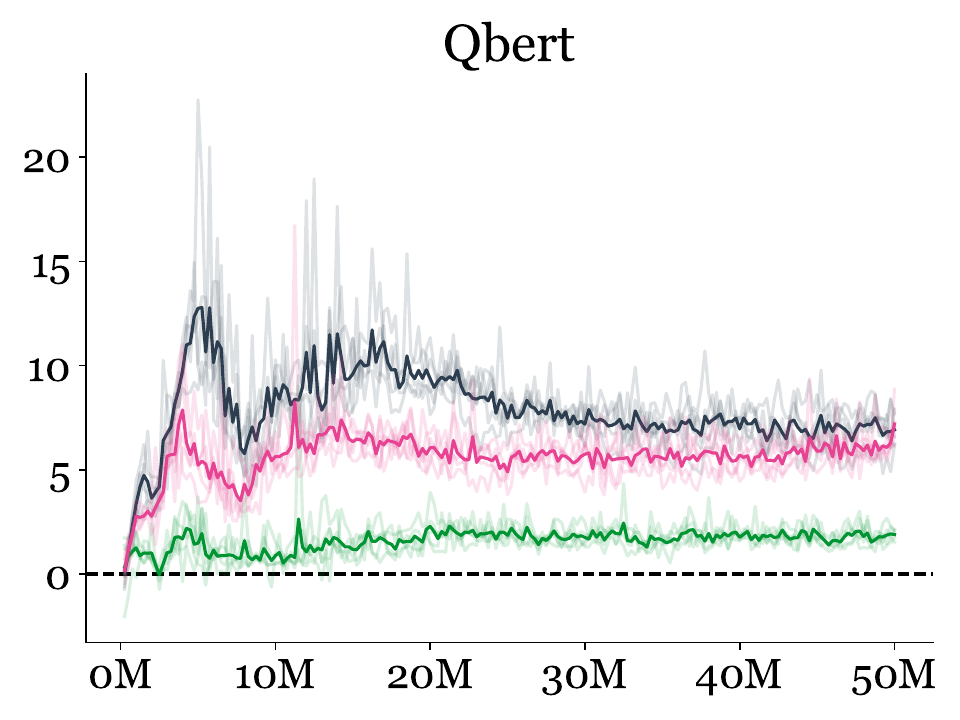} 
	\includegraphics[width=0.21\linewidth]{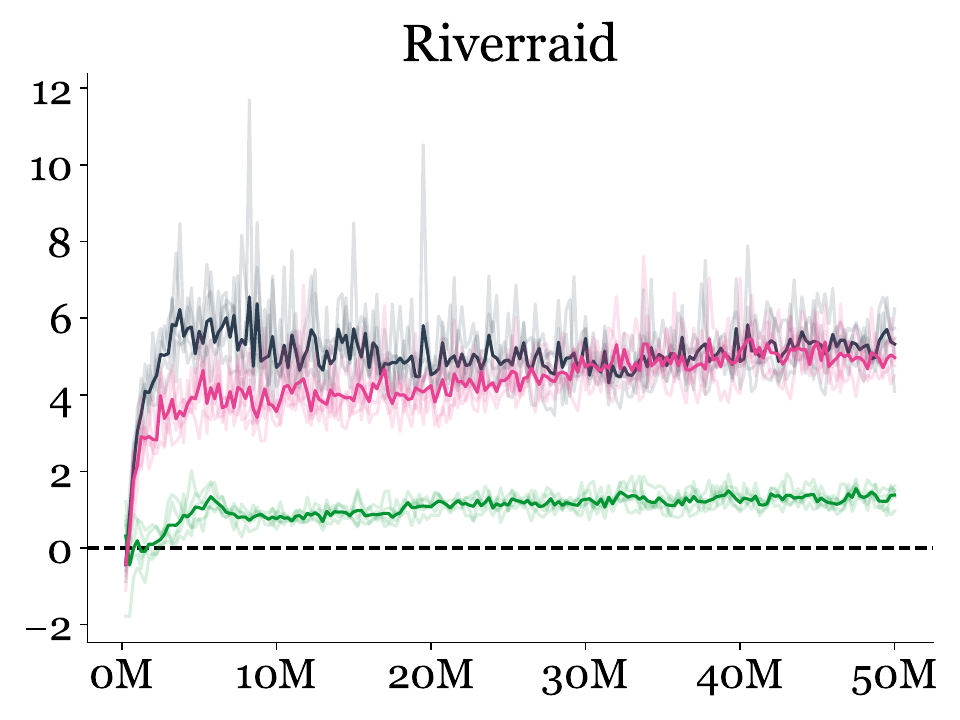} 
	\includegraphics[width=0.21\linewidth]{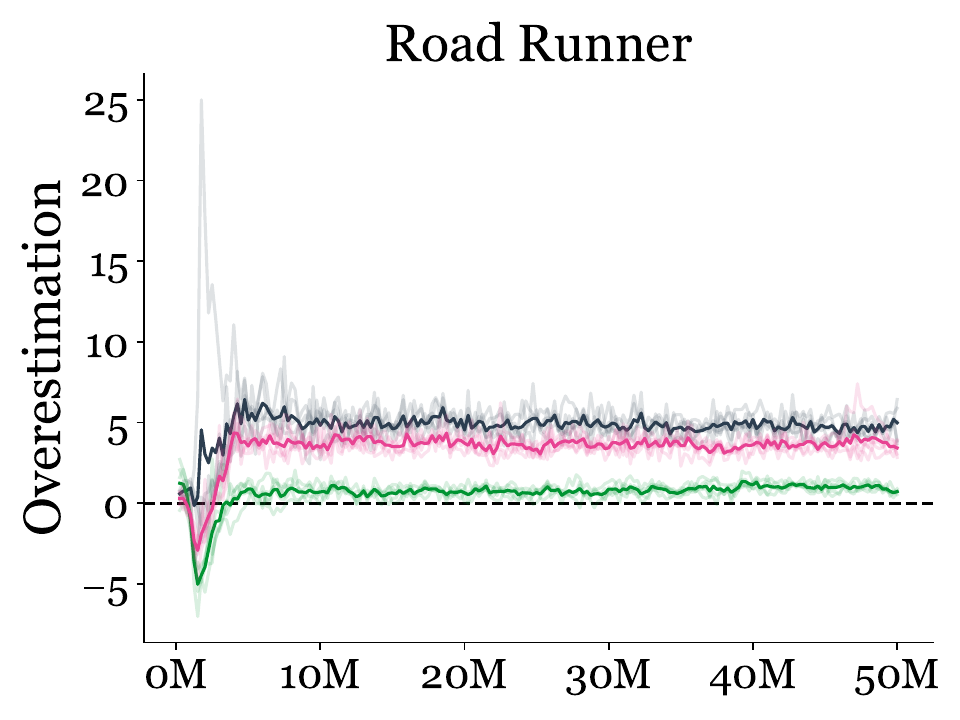} 
	\includegraphics[width=0.21\linewidth]{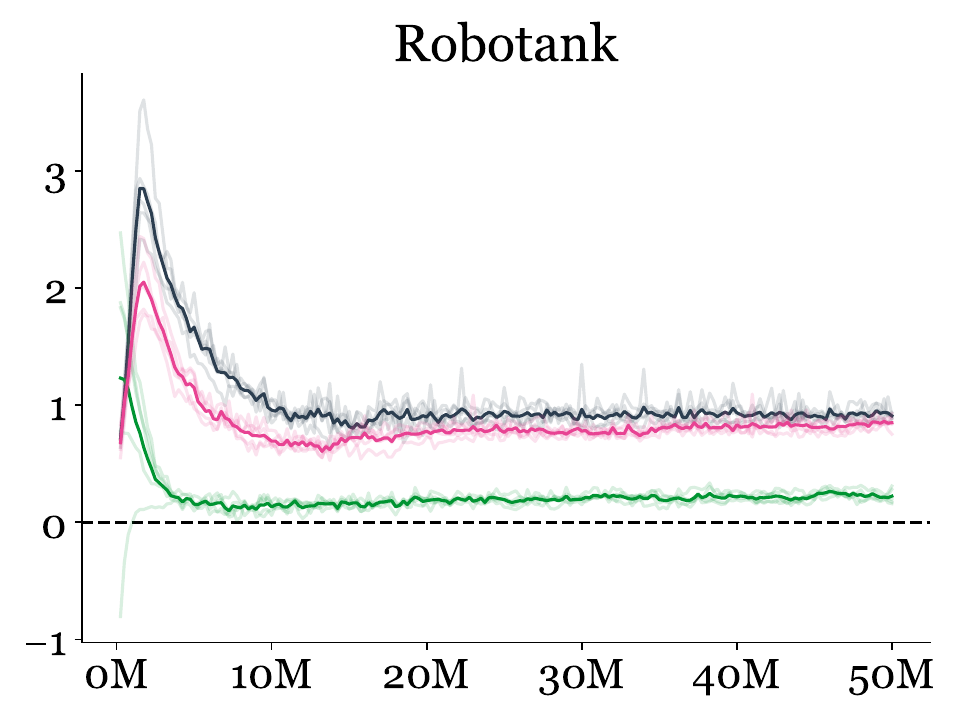} 
	\includegraphics[width=0.21\linewidth]{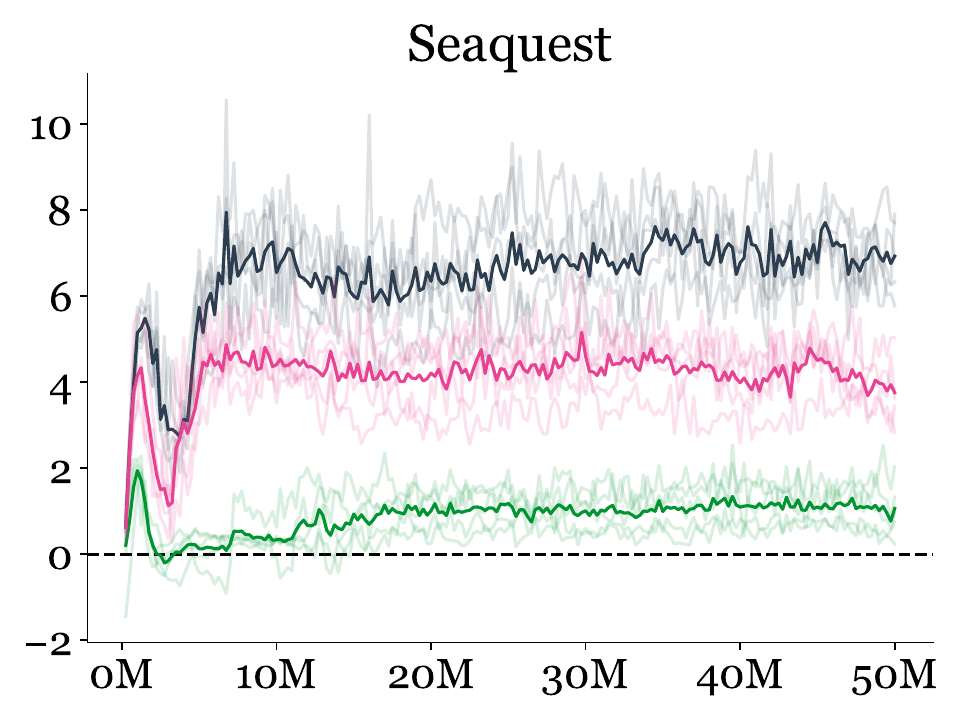} 
	\includegraphics[width=0.21\linewidth]{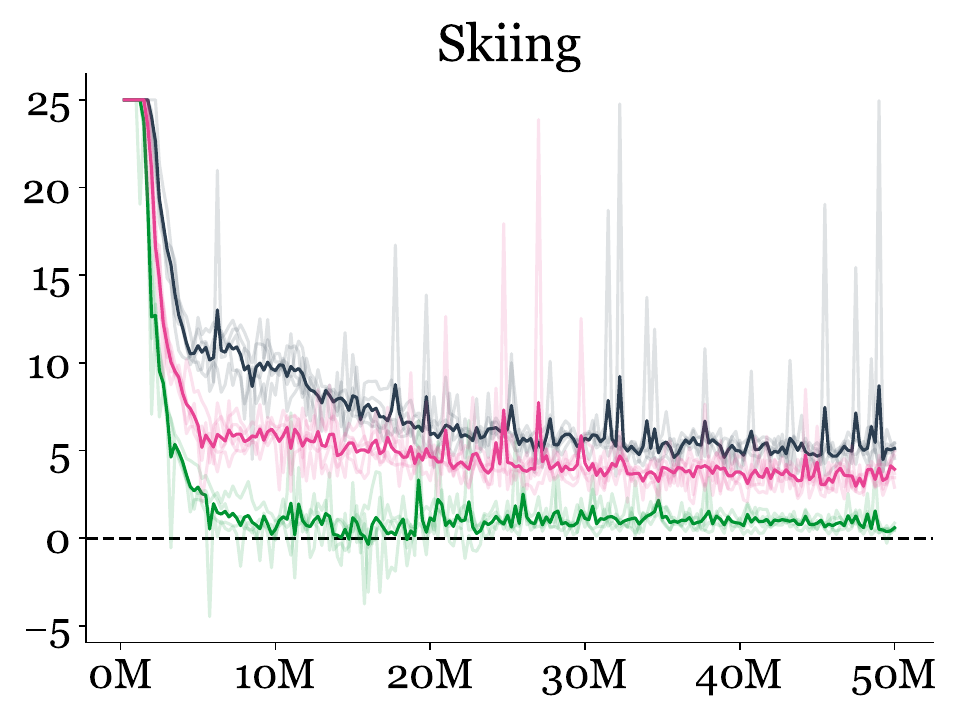} 
	\includegraphics[width=0.21\linewidth]{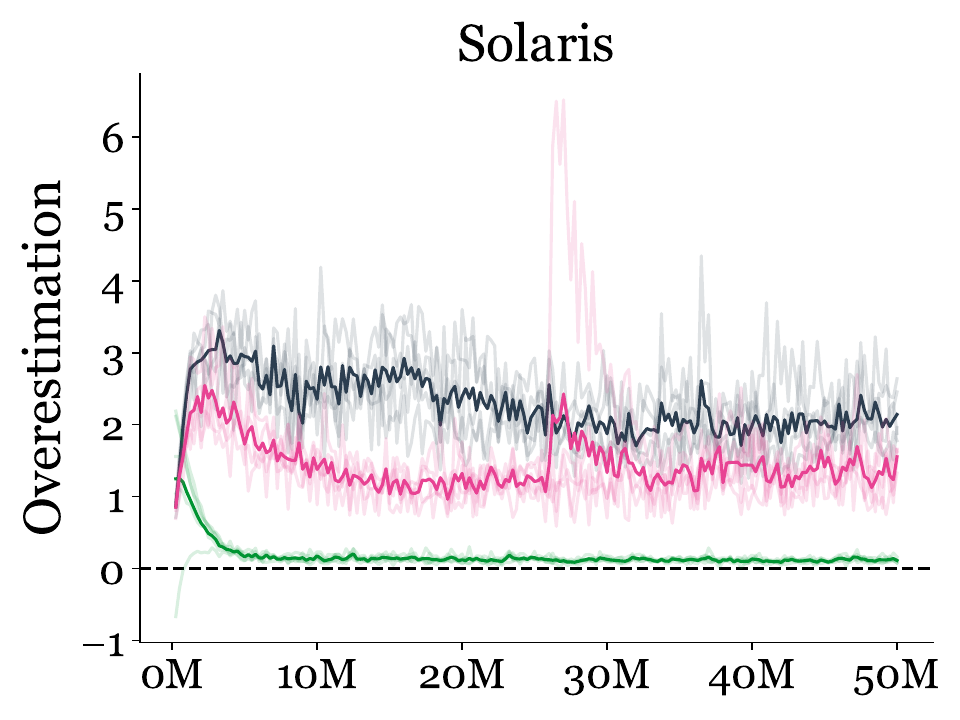} 
	\includegraphics[width=0.21\linewidth]{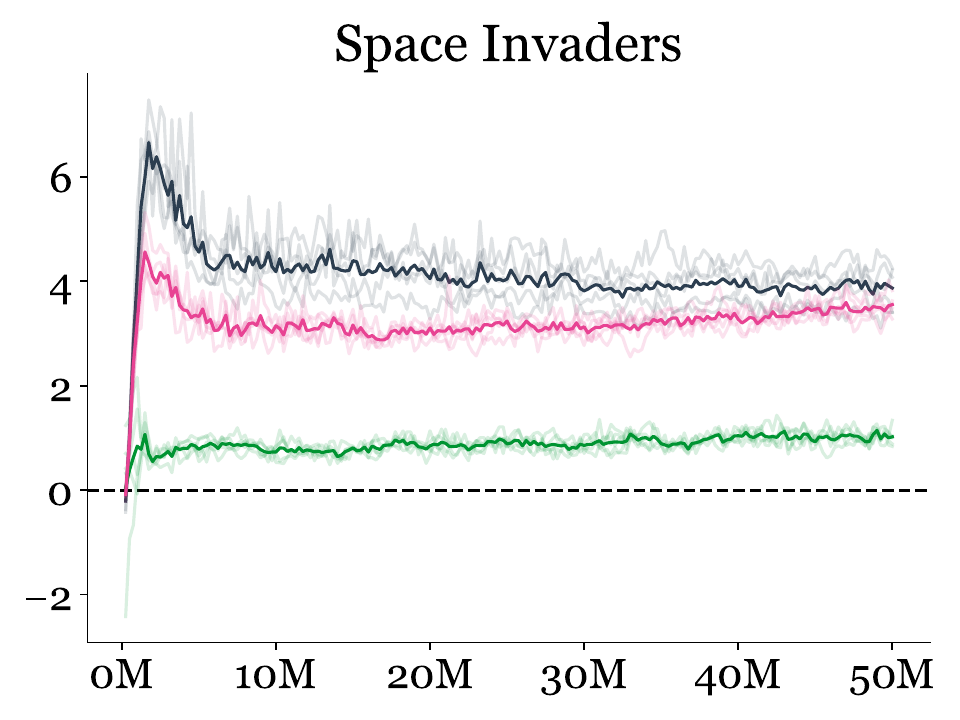} 
	\includegraphics[width=0.21\linewidth]{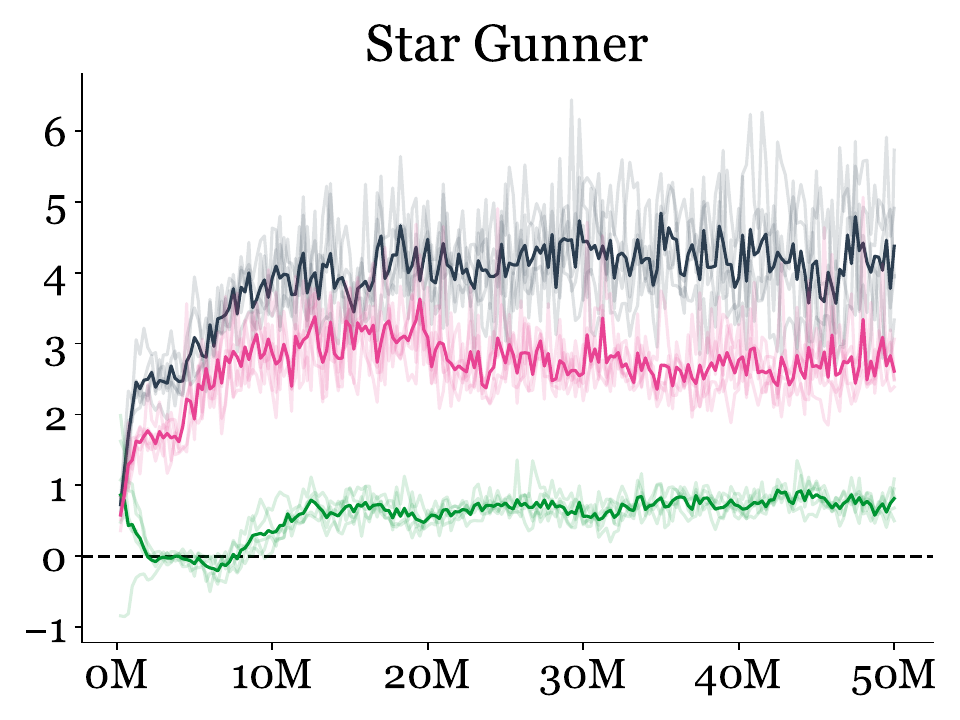} 
	\includegraphics[width=0.21\linewidth]{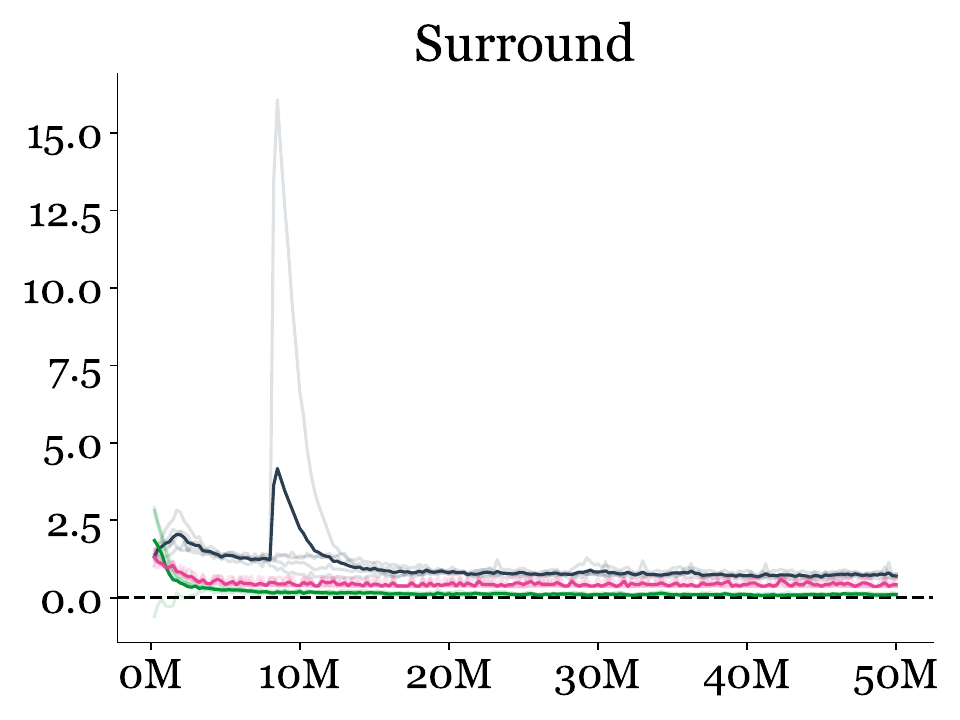} 
	\includegraphics[width=0.21\linewidth]{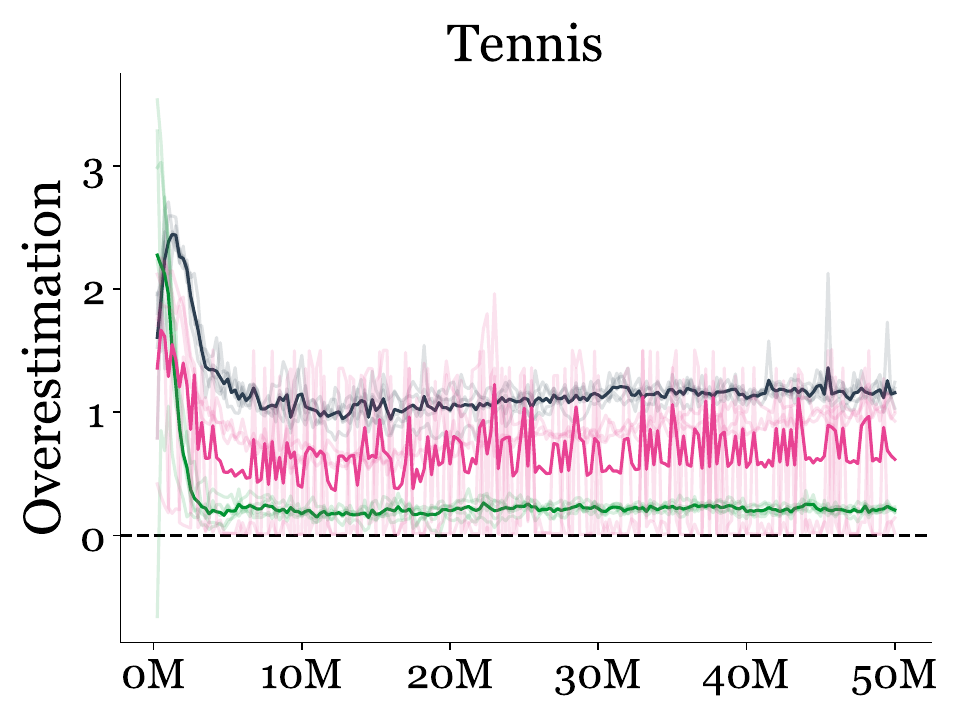} 
	\includegraphics[width=0.21\linewidth]{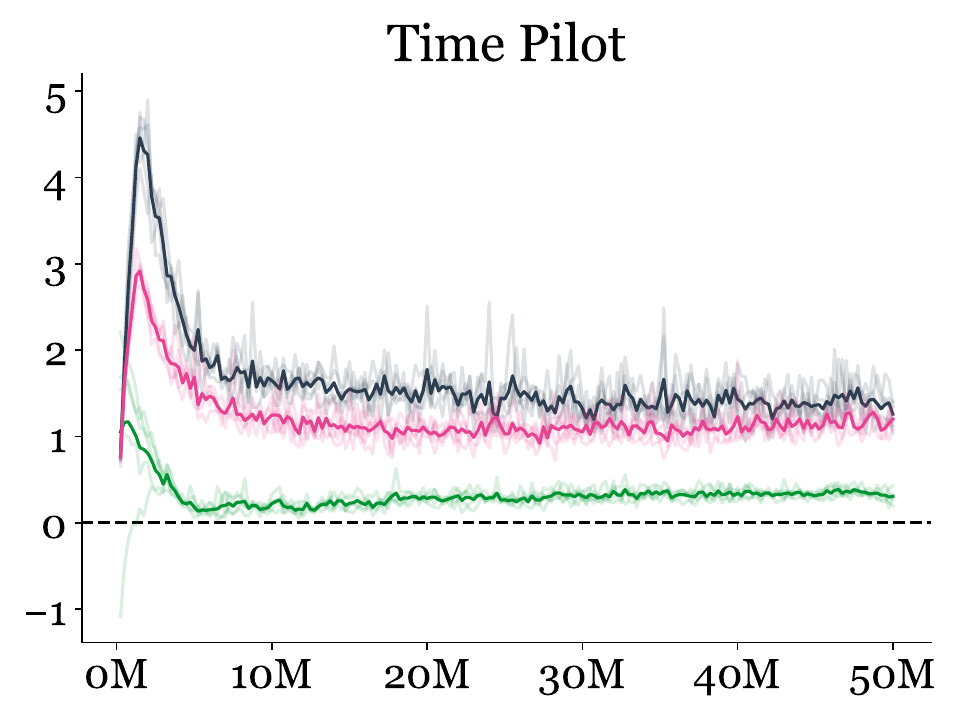} 
	\includegraphics[width=0.21\linewidth]{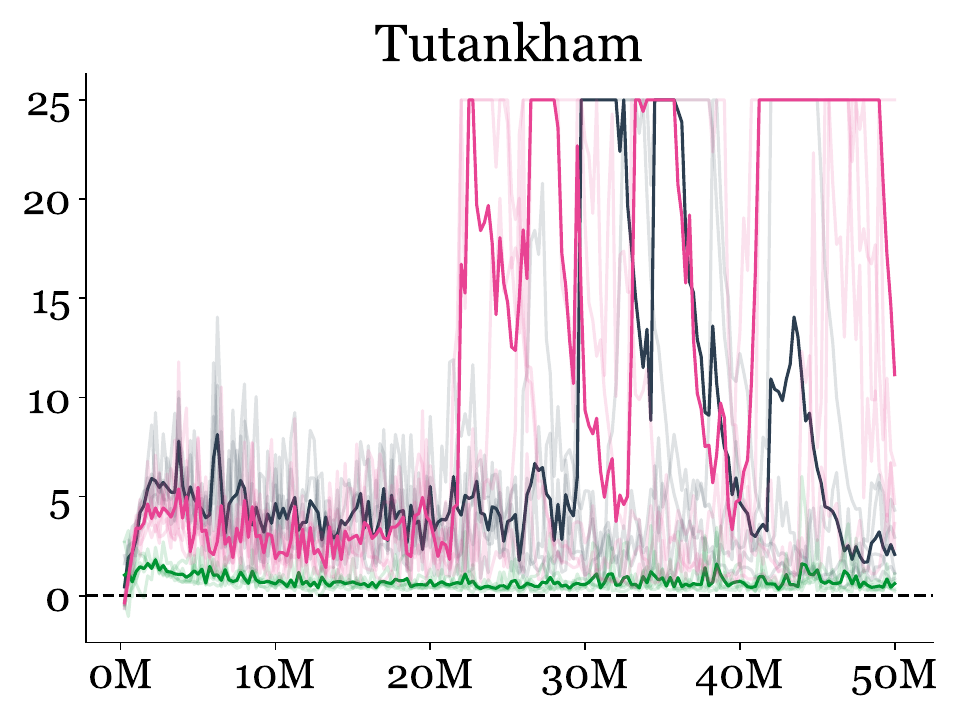} 
	\includegraphics[width=0.21\linewidth]{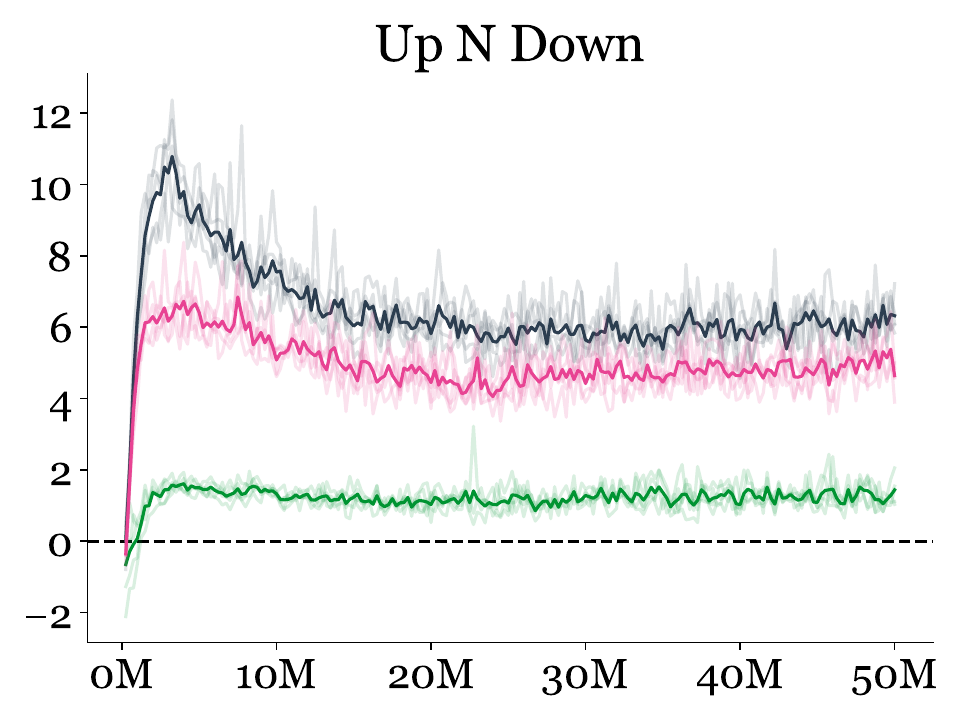} 
	\includegraphics[width=0.21\linewidth]{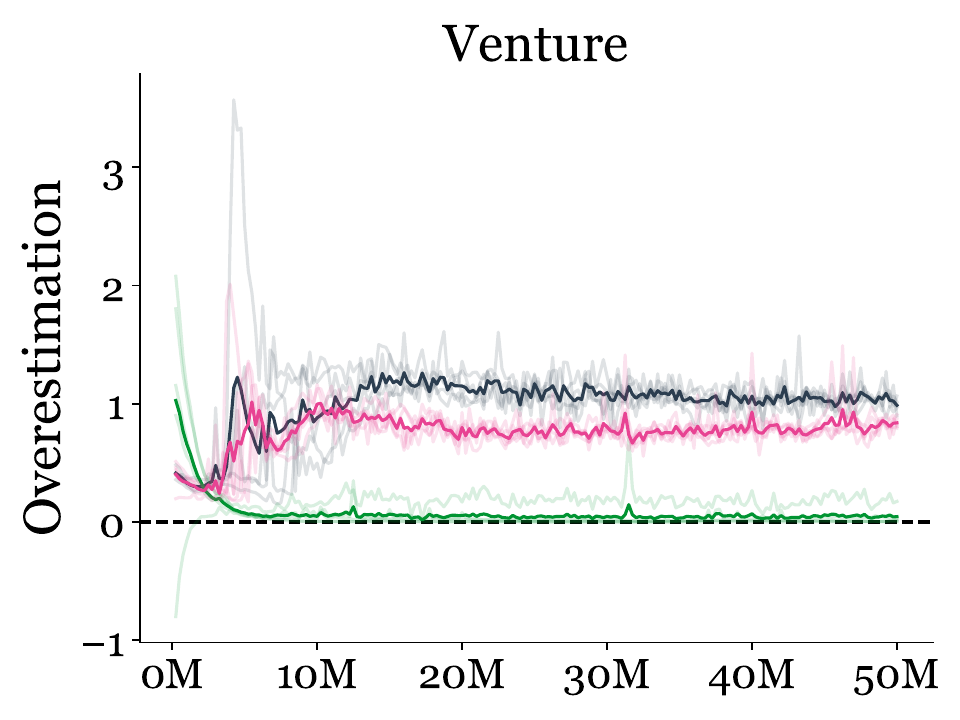} 
	\includegraphics[width=0.21\linewidth]{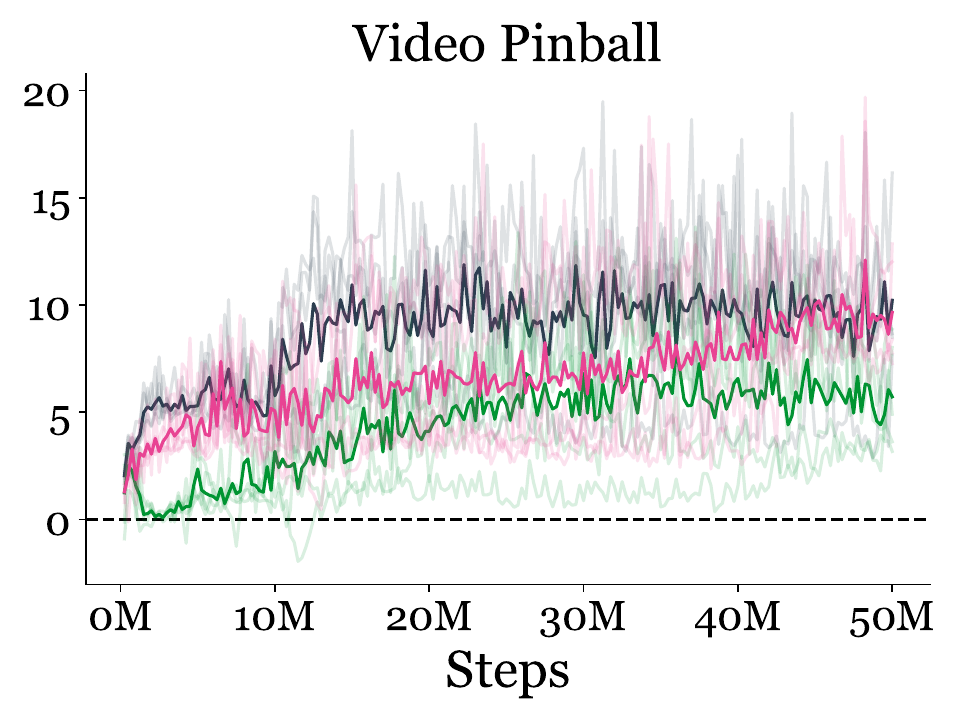} 
	\includegraphics[width=0.21\linewidth]{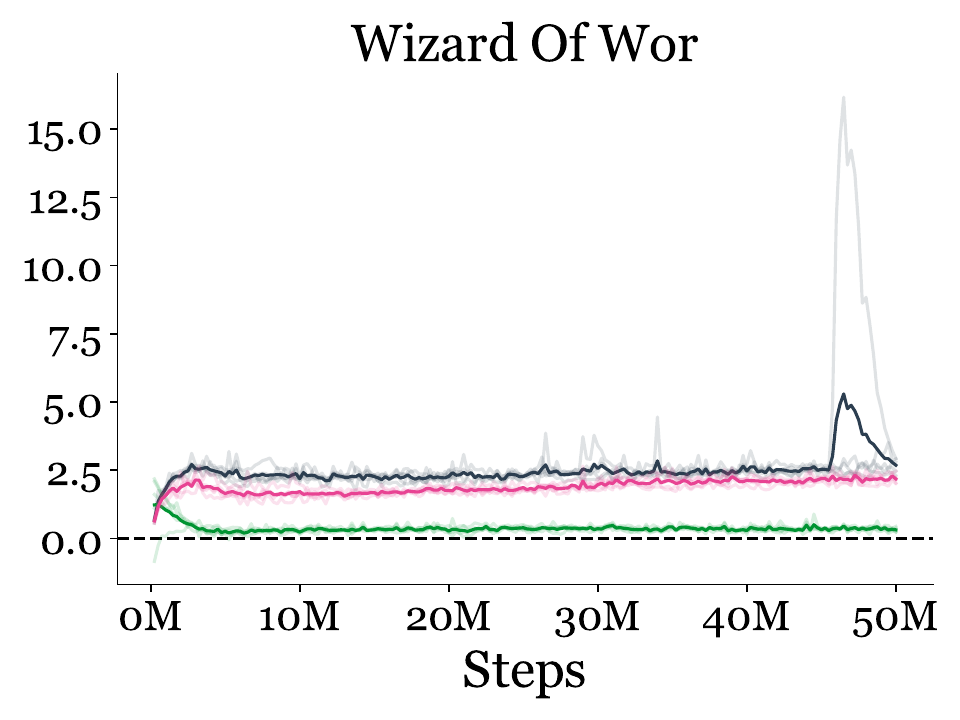} 
	\includegraphics[width=0.21\linewidth]{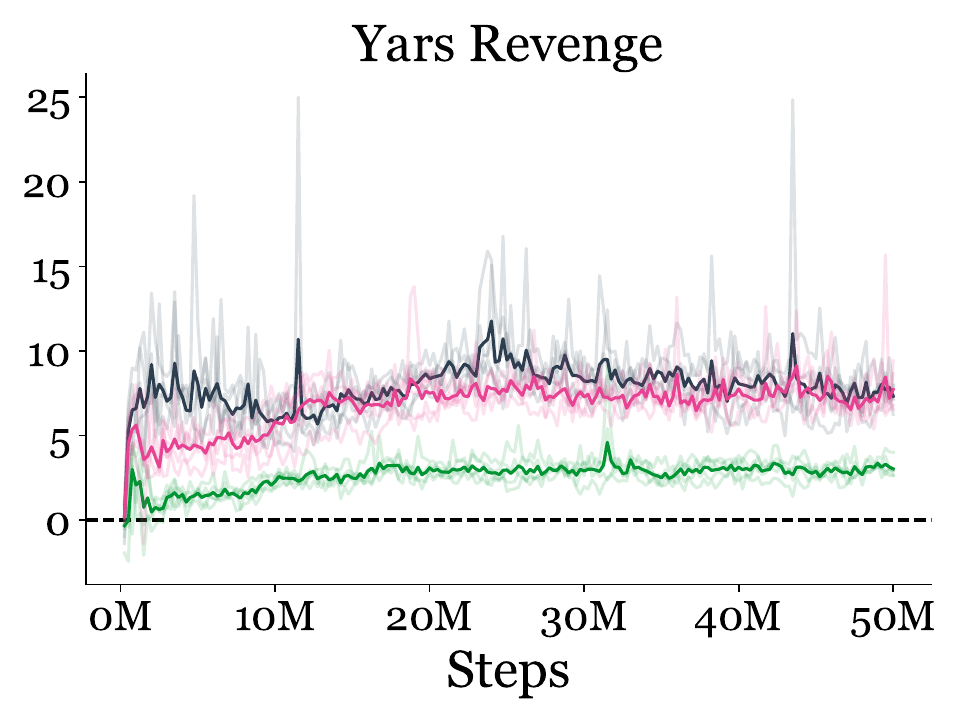} 
	\includegraphics[width=0.21\linewidth]{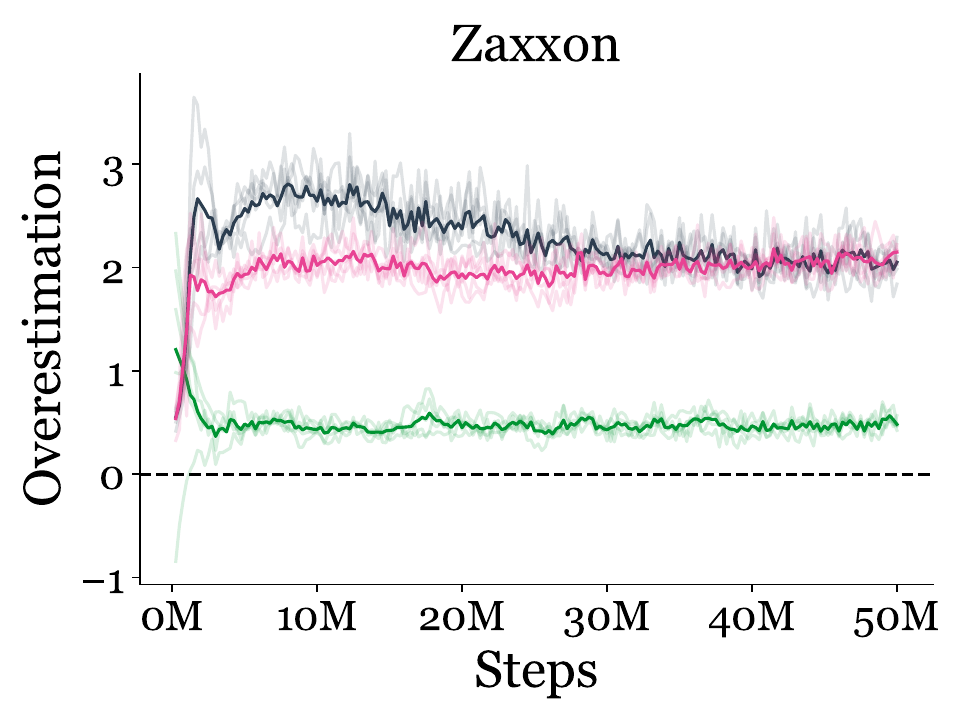}
	\hspace{0.02\linewidth}
    \raisebox{1.5em}{\includegraphics[width=0.19\linewidth]{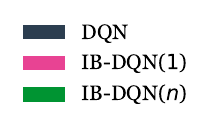}}
    \hspace{0.39\linewidth}
	\caption{Mean overestimation across 50M timesteps across five seeds. Overestimation capped at 25 for visibility. Translucent curves are the individual seeds.}
	\label{Atari57:Overestimation:page_2}
\end{figure}

\end{document}